\documentclass{article} % For LaTeX2e

\usepackage{fullpage}
\usepackage{parskip}
\usepackage{authblk}

\usepackage[utf8]{inputenc} % allow utf-8 input
\usepackage[T1]{fontenc}    % use 8-bit T1 fonts
\usepackage{hyperref}       % hyperlinks
\usepackage{url}            % simple URL typesetting
\usepackage{booktabs}       % professional-quality tables
\usepackage{amsfonts}       % blackboard math symbols
\usepackage{nicefrac}       % compact symbols for 1/2, etc.
\usepackage{microtype}      % microtypography
\usepackage{xcolor}         % colors

\usepackage{amsthm,amsmath,amssymb}
\usepackage{wrapfig}
\usepackage{graphicx}
\usepackage{caption}
\usepackage{subcaption}
\usepackage{enumitem} 

\newcommand{\cC}{\mathcal{C}}

\newcommand{\bu}{{\bf u}}
\newcommand{\bv}{{\bf v}}
\newcommand{\bw}{{\bf w}}

\newcommand{\bx}{{\bf x}}
\newcommand{\by}{{\bf y}}
\newcommand{\bz}{{\bf z}}

\newcommand{\ones}{{\bf 1}}

\newcommand{\ba}{{\bf a}}
\newcommand{\bbb}{{\bf b}}
\newcommand{\bc}{{\bf c}}

\newcommand{\be}{{\bf e}}
\newcommand{\nspl}{{n_{\rm spl}}}

\newcommand{\dist}{{\rm dist }}

\newcommand{\hor}{\text{---}}

\newcommand{\E}{\mathop{{\mathbb E}}}

\newcommand{\nat}{\mathbb{N}}
\newcommand{\real}{\mathbb{R}}

\newcommand{\voc}{\mathcal V}
\newcommand{\data}{\mathcal X}

\newcommand{\ttt}[1]{\text{{\it#1}}}

\newcommand{\dotprod}[1]{ \langle  #1 \rangle }
\newcommand{\dotprodbis}[1]{ \left\langle  #1 \right\rangle_F }
\newcommand{\dotprodbig}[1]{ \Big\langle  #1 \Big\rangle_F }
\newcommand{\tr}{{\rm Tr}} 
\newcommand{\risk}{\mathcal R}

\newcommand{\cpt}{\mathfrak f}
\newcommand{\vphi}{\varphi}

\newcommand{\hot}{\zeta}
\newcommand{\margin}{{\mathfrak M}}
\newcommand{\bbeta}{\boldsymbol{\beta} }
\begingroup
\makeatletter
\@for\theoremstyle:=definition,remark,plain\do{%
\expandafter\g@addto@macro\csname th@\theoremstyle\endcsname{%
\addtolength\thm@preskip\parskip
}%
}
\endgroup

\newtheorem{theorem}{Theorem}
\newtheorem{definition}{Definition}

\newtheorem{lemma}{Lemma}
\newtheorem{proposition}{Proposition}
\newtheorem{claim}{Claim}

\newtheorem{assumption}{Assumption}

\title{Feature Collapse}

\author[1]{Thomas Laurent}
\author[ ]{James H. von Brecht}
\author[2]{Xavier Bresson}
\affil[1]{Loyola Marymount  University,   \texttt{tlaurent@lmu.edu}}
\affil[2]{National University of Singapore,  \texttt{xaviercs@nus.edu.sg}}

\date{}
\begin{document}

\maketitle

\begin{abstract}
\noindent We formalize and study a phenomenon called \emph{feature collapse} that makes precise the intuitive idea that entities playing a similar role in a learning task receive similar representations. As feature collapse requires a notion of task, we leverage a simple but prototypical NLP task to study it. We start by showing experimentally that feature collapse goes hand in hand with generalization. We then prove that, in the large sample limit,  distinct words that play identical roles in this NLP task receive identical local feature representations in a neural network. This analysis reveals the crucial role that normalization mechanisms, such as LayerNorm, play in feature collapse and in generalization.
\end{abstract}

\section{Introduction}
Many machine learning practices implicitly rely, at least in some measure, on the belief that good generalization requires good features.
Despite this, the notion of `good features' remains  vague and carries many potential meanings.
 %Yet from a theoretical perspective the notion of  `good features' remains  vague.
% But what do we mean exactly by good feature?
One definition is  that  features/representations should only encode the information necessary to do the task at hand, and discard  any unnecessary information as noise.
  For example, two distinct patches of grass should map to essentially identical representations even if these patches differ in pixel space. Intuitively, a network that gives the same representation to many distinct patches of grass has learned  the \emph{`grass'} concept. We call this phenomenon \emph{feature collapse}, meaning that a learner gives same features to entities that play similar roles for the task at hand.
 
We aim to give some mathematical precision to this notion, so that it has a clear meaning, and to investigate the relationship between collapse, normalization, and generalization. As feature collapse cannot be studied in a vacuum, since the notion only makes sense within the context of a specific task, we use a simple but prototypical NLP task to formalize and study the concept.
 %In order to make intuitive the theoretical results which constitute the core of our work, we start, in section 2, with a series of experiments that tell in pictures what the theoretical section (section 3) tells with equations.
 %In our experiments, we consider a simple neural network, and observe the following:
 In order to make intuitive the theoretical results which constitute the core of our work, we begin our presentation in section \ref{section:datamodel} with a set of visual experiments that illustrate the key ideas that are later mathematically formalized in section \ref{section:theory}.
In these experiments, we explore the behavior of a simple neural network on our NLP task, and make the following observations:
 \begin{enumerate}[leftmargin=22pt]
%  \begin{enumerate}[label=\roman*]
 \item[(i)] If words in the corpus have \emph{identical} frequencies then feature collapse occurs. Words that play the same role for the task at hand receive identical embeddings, and the learner generalizes well.
 \item[(ii)] If words in the corpus have \emph{distinct} frequencies (e.g. their frequencies follow the Zipf’s law) then feature collapse does not take place in the absence of a normalizing mechanism. Frequent words receive larger embeddings than rare words, and this failure to collapse features leads to poorer generalization.
 \item[(iii)] Including a  normalization mechanism (e.g. LayerNorm) restores feature collapse and, as a consequence, restores the generalization performance of the learner.
 \end{enumerate}
 %In this way we directly observe an instantiation of the prior intuition. A learner that encodes irrelevant information, namely the overall frequency of a word token, performs worse than a learner that does not. Additionally, the coding of this irrelevant information manifests as a lack of feature collapse.
 
 These observations motivate our theoretical investigation into the feature collapse phenomenon. We show that the optimization problems associated to our empirical investigation have explicit analytical solutions under certain symmetry assumptions placed on the NLP task. These analytical solutions allow us to get a  precise theoretical  grasp on the phenomena (i), (ii) and (iii) observed empirically. Concretely, we provide a rigorous proof of \emph{feature collapse} in the context of our NLP data model. Distinct  words that play the same role for the task at hand receive the same features. Additionally, we show that normalization is the key to obtain collapse and good generalization, when words occur with distinct frequencies. These contributions provide a theoretical framework for understanding two empirical phenomena  that occur in actual machine learning practice: entities that play similar roles receive similar representations, and normalization is key in order to obtain good representations.

\subsection{Related works}  In pioneering work \cite{papyan2020prevalence}, a series of experiments on popular network architectures and popular  image classification tasks revealed a striking phenomenon --- a well-trained network  gives identical representations, in its last layer, to training points that belongs to the same class.  In a $K$-class classification task we therefore see the emergence, in the last layer,  of $K$ vectors coding for the $K$ classes. Additionally, these $K$ vectors `point' in `maximally opposed' directions. This phenomenon, coined \emph{neural collapse}, has been studied extensively since its discovery. A recent line of theoretical works 
(e.g. \cite{mixon2020neural, lu2020neural, wojtowytsch2020emergence, fang2021exploring, zhu2021geometric, ji2021unconstrained, tirer2022extended, zhou2022optimization}) investigate neural collapse in the context of the  so-called \emph{unconstrained feature model}, which treats the representations of the training points in the last layer as free optimization variables that are not constrained by the previous layers. Under various assumptions and with various losses, and under the unconstrained feature model, these works prove that the $K$ vectors coding for the $K$ classes indeed have maximally opposed directions.

The phenomenon we study in this work, \emph{feature collapse},  has superficial similarity to neural collapse but in detail is quite different. Feature collapse describes the emergence of `good' local features in a neural network, and in particular, it has no meaningful instantiation within the unconstrained feature model framework. To give an illustrative example of the difference, \emph{neural collapse} refers to a phenomenon where all images from the same class receive the same representation at the end of the network. By contrast, \emph{feature collapse} refers to the phenomenon where all image patches that play the same role for the task at hand, such as two distinct patches of grass, receive the same representation. This distinction makes feature collapse harder to define and analyze because it is fundamentally  a task dependent  phenomenon. It demands both a well-defined `input notion' (such as patch) as well as a well-defined notion of `task', so that the statement ``patches that play the same role in the task'' has content. In contrast,  \emph{neural collapse} is mostly a task agnostic phenomenon.

\subsection{Limitations}
Our theoretical results are applicable only in the large sample limit and under certain symmetry assumptions placed on the NLP task.
In this idealized scenario, the problems we examine possess perfect symmetry and homogeneity.
This allows us to derive  symmetric and homogeneous  analytical solutions  that describe the weights of a trained network. 
Importantly,  these analytical solutions still provide valuable predictions  when the large sample limit and symmetry  assumptions are relaxed. Indeed,
our experiments demonstrate that the weights of networks trained with a small number of samples closely approximate the idealized analytical solution. 
Quantifying the robustness of these analytical solutions to under-sampling is technically challenging, but of great interest.

\subsection{Reproducibility}
The codes for all our experiments are available at \url{https://github.com/xbresson/feature_collapse}, where we provide notebooks that reproduces all the figures shown in this work.

\section{A tale of feature collapse} \label{section:datamodel}

We begin by more fully telling the empirical tale that motivates our theoretical investigation of feature collapse. It starts with a simple model in which sentences of length $L$ are generated  from some underlying set  of  latent variables that encode the $K$ classes of a classification task.  Figure \ref{figure:datamodel} illustrates the basic idea. \\  The left side  of the figure  depicts a vocabulary of $n_w=12$ word tokens and $n_c = 3$ concept tokens
$$
\voc = \{ \text{potato, cheese, carrots, chicken}, \ldots \} \qquad \text{and} \qquad \mathcal C = \{\text{vegetable, dairy, meat}\}
$$ 
with the $12$ words partitioned into the  $3$ equally sized concepts. A sentence  $\bx \in \mathcal V^L$ is a sequence of $L$ words  ($L=5$ on the figure), and a latent variable $\bz \in \mathcal C^L$  is a sequence of $L$  concepts.
%A latent variable $\bz \in \mathcal C^L$  is a sequence of $L$  concepts ($L=5$ on the figure), and a sentence $\bx \in \mathcal V^L$  is a sequence of $L$ words. 
 The latent variables generate sentences. For example
\begin{equation*} %\label{generate}
\bz =  [ \,\text{dairy}, \, \text{veggie},\, \text{meat} ,\, \text{veggie} , \,  \text{dairy}\,] \;\;\;\;\;  \overset{\text{generates}}{\longrightarrow} \;\; \;\;\;
 \bx =  [\,\text{cheese}, \, \text{carrot},\, \text{pork}, \,\text{potato},\, \text{butter} \, ]
 \end{equation*}
 with the sentence on the right obtained by sampling each word  at random from the corresponding concept. The first word represents a random sample from the dairy concept ({\it butter, cheese, cream, yogurt}) according to the dairy distribution (square box at left), the second word represents a random sample from the vegetable concept ({\it potato, carrot,  leek, lettuce}) according to the vegetable distribution, and so forth. At right, figure   \ref{figure:datamodel} depicts a classification task with $K=3$ categories prescribed by  the three latent variables $\bz_1, \bz_2, \bz_3 \in \mathcal C^L$. Sentences generated by the latent variable $\bz_k$  share the same label $k$, yielding a  classification problem that requires a learner to classify sentences among $K$ categories. This task provides a clear way of studying the extent to which feature collapse occurs, as all words in a concept clearly play the same role. An intuitively `correct' solution should therefore map all words in a concept to the same representation.

% \vspace{-0.3cm}

   \begin{figure}[t] 
         \centering
          \includegraphics[width=1\linewidth]{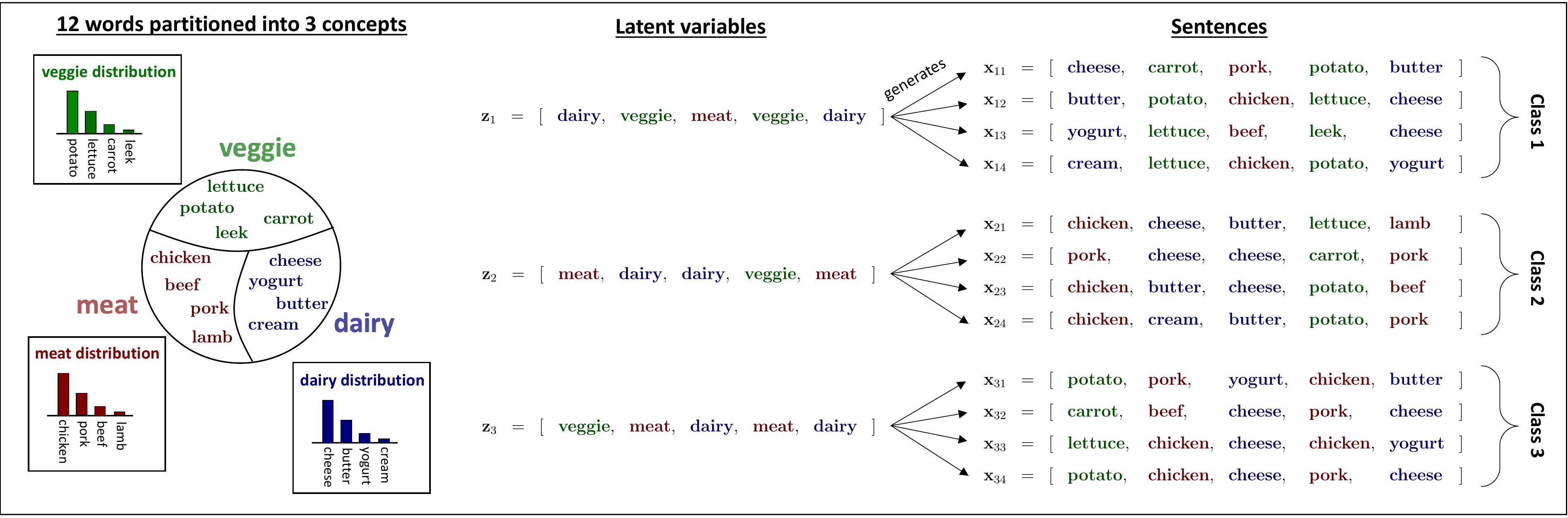}
             \caption{
             Data model with parameters set to
             $n_c=3$, $n_w=12$,   $L=5$,  and $K=3$.}
             \label{figure:datamodel}
\end{figure}

 \begin{wrapfigure}[15]{r}{0.24\textwidth}
 % \begin{center}
 \vspace{-0.4cm}
 \centering
    \includegraphics[totalheight=0.12\textheight]{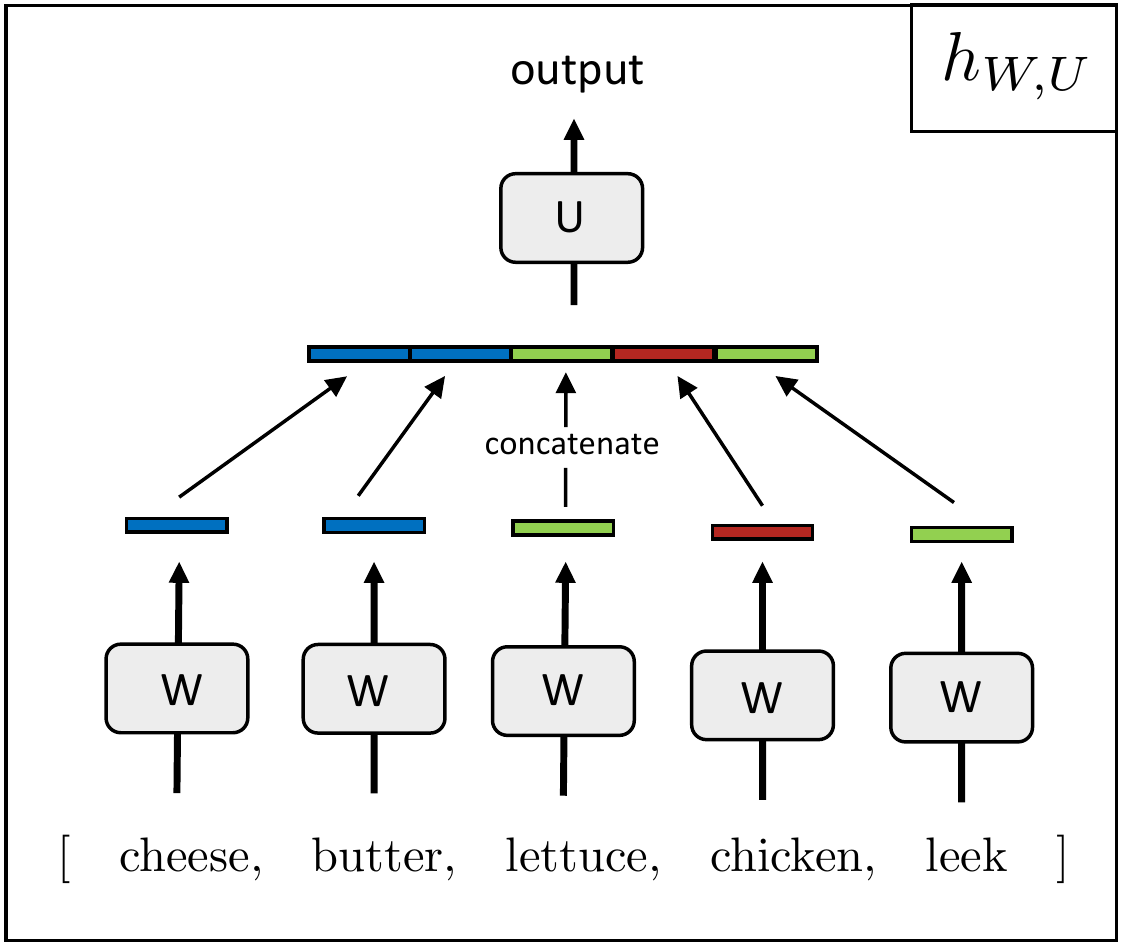} \\ \vspace{0.12cm}
       \includegraphics[totalheight=0.133\textheight]{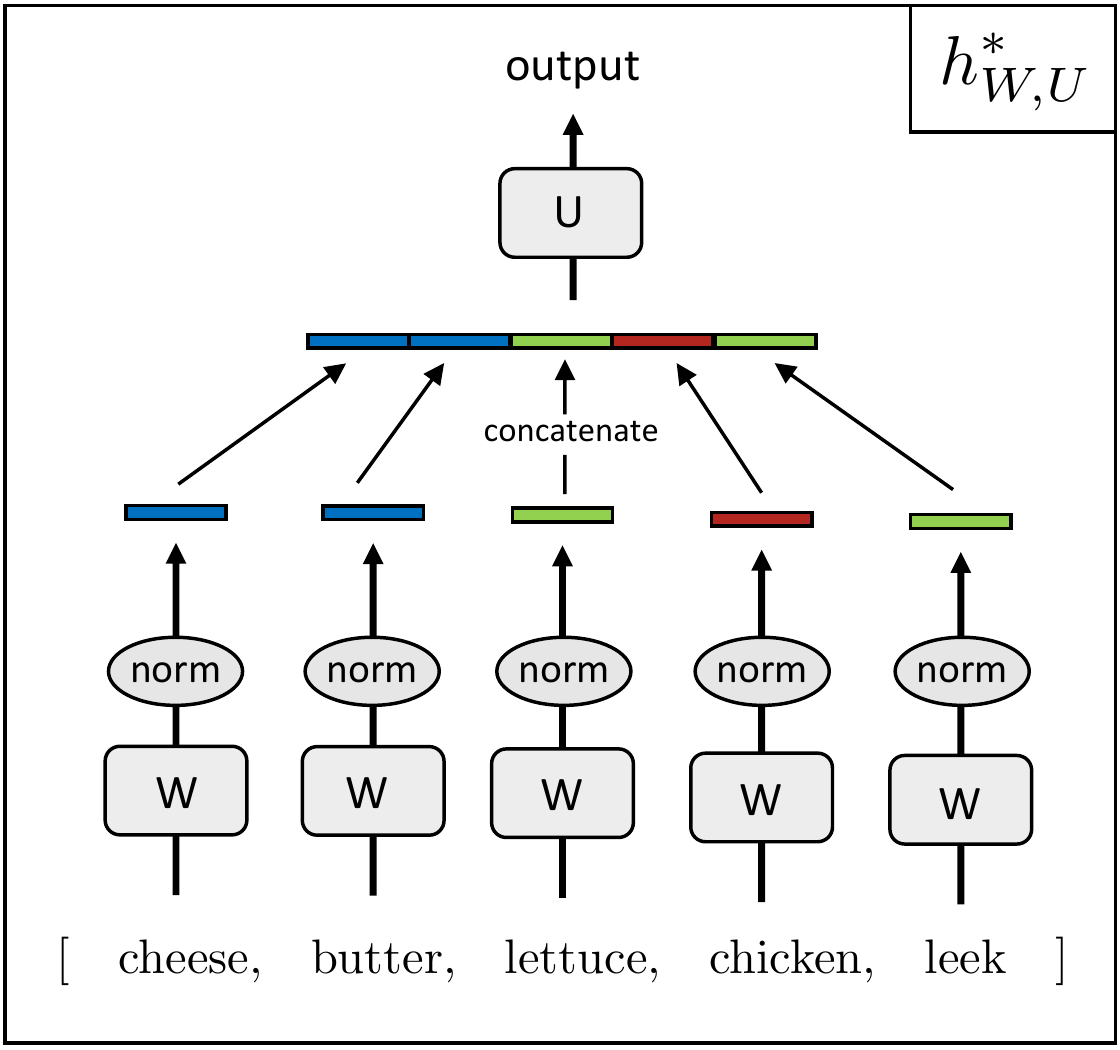}
  %\end{center}
    \caption{networks}
    \label{figure:net}
\end{wrapfigure}

 We use two similar networks to empirically study if and when this phenomenon occurs.
 The first network $\bx \mapsto h_{W,\,U}(\bx)$, depicted on the top panel of figure   \ref{figure:net}, 
 starts by embedding each word in a sentence by    applying a $d \times n_w$ matrix $W$ to the one-hot representation of each word. 
  It then concatenates these $d$-dimensional
  embeddings of each word into a single vector. Finally, it applies a linear transformation $U$ 
  to produce
   a $K$-dimensional score vector $\by = h_{W,\,U}(\bx)$ with one entry for each of the $K$ classes.  
%  {\color{red}The first network $\bx \mapsto h_{W,\,U}(\bx)$ is depicted on the top panel of figure   \ref{figure:net}.
%  Each word of the input sentence, after being encoded into a one-hot-vector, is multiplied by a matrix $W \in \real^{d \times n_w}$. This results into $L$  \emph{word embeddings}  that lives in a $d$-dimensional space. 
%  The word embeddings are then concatenated into a single vector, before being multiplied by a matrix $U$ 
%  to produce
%   a $K$-dimensional score vector $\by = h_{W,\,U}(\bx)$ with one entry for each of the $K$ classes. }
The $d \times n_w$ embedding matrix $W$ and the $K \times Ld$ matrix $U$ of linear weights are the only learnable parameters, and the network has no nonlinearities. The second network $\bx \mapsto h^{*}_{W,\,U}(\bx)$, depicted at bottom, differs only by the application of a LayerNorm module (c.f. \cite{ba2016layer}) to the word embeddings prior to the concatenation. For simplicity we use a LayerNorm module which does not contain any learnable parameters; the module simply removes the mean and divides by the standard deviation of its input vector. As for the first network, the only learnable weights are  the matrices $W \in \real^{d \times n_w}$ and $U \in \real^{K \times Ld}$.

If feature collapse occurs then these networks will give identical representations to words that play the same role. For example, the four words \emph{butter}, \emph{cheese}, \emph{cream} and \emph{yogurt} all belong to the \emph{dairy} concept, and we should see this clearly reflected in the weights. At the level of the the word embeddings $W$ this has a transparent meaning; all words belonging to the \emph{dairy} concept (indeed any concept) should receive similar embeddings, and these embeddings should allow for distinguishing between concepts. Now partition the linear transformation
\begin{align}\label{eq:upart00}
U = \begin{bmatrix}
 \hor\bu_{1,1} \hor &  \hor\bu_{1,2} \hor & \cdots  &\hor\bu_{1,L} \hor \\
 \hor\bu_{2,1} \hor &  \hor\bu_{2,2} \hor & \cdots  &\hor\bu_{2,L} \hor \\
 \vdots  &\vdots  & & \vdots \\
  \hor\bu_{K,1} \hor &  \hor\bu_{K,2} \hor & \cdots  &\hor\bu_{K,L} \hor 
\end{bmatrix}
\end{align}
into its components $\bu_{k,\,\ell} \in \real^d$ that `see' the embeddings of the $\ell$${{\rm th}}$ concept $z_{k,\,\ell}$ from the $k$${{\rm th}}$ class. For example, if $z_{k,\,\ell} =$ {\it veggie} then the latent variable $\bz_k$ contains the {\it veggie} concept in the $\ell$${{\rm th}}$ position. If $W$ properly encodes concepts then we expect  the vector $\bu_{k,\,\ell}$ to give  a strong response when presented with the embedding of a word that belongs to the {\it veggie} concept. So we would expect $\bu_{k,\,\ell}$ to align with the embeddings of the words that belong to the \ttt{veggie} concept, and so feature collapse would occur in this manner as well.
  
If feature collapse does, in fact, play an important role then we  should observe it empirically in well-trained networks that exhibit good generalization performance. To test this hypothesis we use the standard cross entropy loss
\[
\ell(\by, k) = -   \log \left(  \frac{\exp\left( y_k\right)}{\sum_{k'=1}^K \exp\left( y_{k'}\right)}\right) \qquad \text{ for } \by \in \real^K
\]
and then minimize the corresponding regularized empirical risks
       \begin{align} \label{empirical_risk}
&\mathcal R_{\rm emp}(W,U) = \frac{1}{K} \frac{1}{\nspl}  \sum_{k=1}^K   \sum_{i=1}^{\nspl}    \ell\big( \; h_{W,U}\left(\bx_{k,i} \right) \; , \; k  \; \big)  + \frac{\lambda}{2}\|U\|_F^2  +  \frac{\lambda}{2} \|W\|_F^2   \\
&\mathcal R^*_{\rm emp}(W,U) = \frac{1}{K} \frac{1}{\nspl}  \sum_{k=1}^K   \sum_{i=1}^{\nspl}    \ell\big( \; h^*_{W,U}\left(\bx_{k,i} \right) \; , \; k  \; \big)  + \frac{\lambda}{2}\|U\|_F^2      \label{empirical_risk_LN}
\end{align}
%and
%     \begin{align} \label{empirical_risk2}
%&\mathcal R_{\rm emp}(U,V) = \frac{1}{K} \frac{1}{\nspl}  \sum_{k=1}^K   \sum_{i=1}^{\nspl}    \ell\big( \; h_{W,U}\left(\bx_{k,i} \right) \; , \; k  \; \big)  + \frac{\lambda}{2}\|U\|_F^2  +  \frac{\lambda}{2} \|W\|_F^2\\   
%&\mathcal R^*_{\rm emp}(U,V) = \frac{1}{K} \frac{1}{\nspl}  \sum_{k=1}^K   \sum_{i=1}^{\nspl}    \ell\big( \; h^*_{W,U}\left(\bx_{k,i} \right) \; , \; k  \; \big)  + \frac{\lambda}{2}\|U\|_F^2     
%\end{align}
of each network via stochastic gradient descent.  The $\bx_{k,\,i}$ denote the $i$-${{\rm th}}$ sentence of the $k$-${{\rm th}}$ category in the training set, and so each of the $K$ categories has $n_{{\rm spl}}$ representatives. 

For the parameters of the architecture, loss, and training procedure,  we use an embedding dimension of $d=100$, a weight decay of $\lambda=0.001$,  a mini-batch size of $100$ and a constant learning rate $0.1$, respectively,  for all experiments. The regularization terms play no essential role apart from making proofs easier; the empirical picture remains the same without weight decay. We do not penalize $\|W\|_F^2$  in equation \eqref{empirical_risk_LN} since the LayerNorm module implicitly regularizes the matrix $W$. For the parameters of the data model, we use $n_c=3$ so that we may think of the concepts as being  {\it vegetable}, {\it dairy}, and  {\it meat}. But any $n_c$ would work, as the theoretical section will make clear. Finally, we will work in the regime where the number of classes is large (e.g. $K=1000$) but the number of sample per class is small (e.g. $\nspl=5$). In this regime a learner is forced to discover features that are meaningful for many categories, therefore promoting generalization.

%The codes for all our experiments are available at \url{https://anonymous.4open.science/r/feature_collapse-9192} and run on a single GPU in a few minutes.

\subsection{The uniform case}
We start with an instance of the task from figure \ref{figure:datamodel} with parameters
   \begin{equation*} %\label{param}
n_c =3, \quad n_w = 1200, \quad L=15,  \quad  K=1000
   \end{equation*}
and with uniform word distributions. So each of the $n_c = 3$ concepts ({\it vegetable}, {\it dairy}, and  {\it meat}) contain $400$ words and the corresponding distributions (the veggie distribution, the dairy distribution, and the meat distribution) are uniform. We form $K=1000$ latent variables $\bz_1, \ldots, \bz_{1000}$ by selecting them uniformly at random from the set $\mathcal C^L$, which simply means that any concept sequence $\bz = [z_1,\,\ldots,\,z_L]$ has an equal probability of occurrence. We then construct a training set by generating  $\nspl=5$ data points from each latent variable. We then train both networks $h,\,h^*$ and evaluate their generalization performance; both achieve $100\%$ accuracy on test points.

     We therefore expect that both networks exhibit feature collapse. To illustrate this collapse, we start by visualizing in  figure \ref{figure:first_experiment}  the
      learnable parameters $W, U$ of the network $h_{W,\,U}$ after training. The
       embedding 
     \begin{wrapfigure}[18]{r}{0.21\textwidth}
  \begin{center}
 \vspace{-0.22cm}
 \centering
    \includegraphics[totalheight=0.145\textheight]{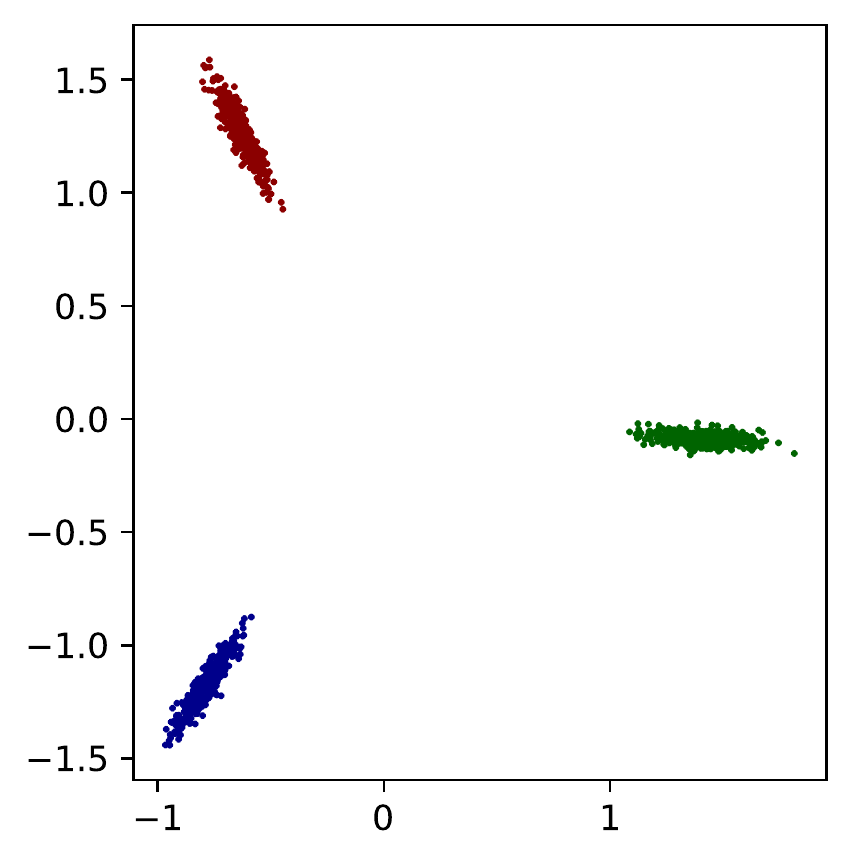} \\ \vspace{0.2cm}
       \includegraphics[totalheight=0.135\textheight]{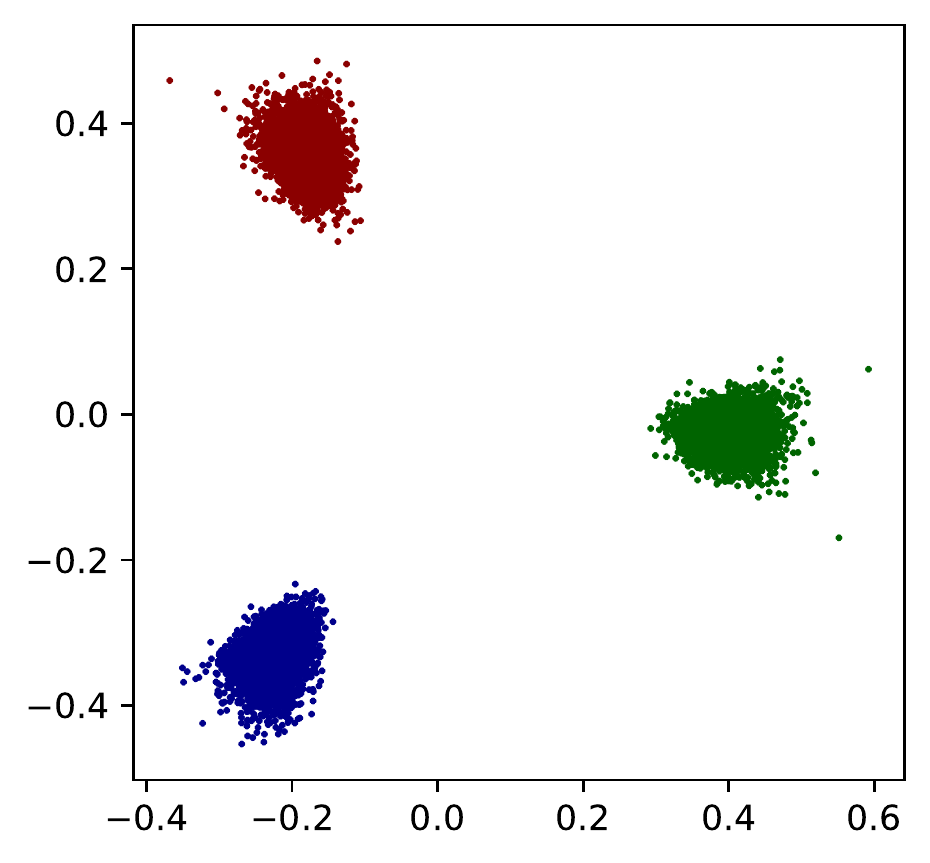}
  \end{center}
    \caption{$W$ and $U$}
    \label{figure:first_experiment}
\end{wrapfigure}   
       matrix $W$ contains $n_w = 1200$ columns.
 Each column is a vector in $\real^{100}$  and corresponds to a word embedding.  
 The top panel of figure \ref{figure:first_experiment}
depicts these $1200$    
 word embeddings after dimensionality reduction via PCA. 
 The top singular values $\sigma_1=34.9$, $\sigma_2=34.7$ and $\sigma_3 = 0.001$ associated with the PCA indicate that the word embeddings essentially live in a 2 dimensional subspace of $\real^{100}$, and so the PCA paints an accurate picture of the distribution of word embeddings. We then color code each word embedding accorded to its concept, so that all embeddings  of words within a concept receive the same color (say all \ttt{veggie} words in green, all \ttt{dairy} words in blue, and so forth). As the figure illustrates, words from the same concept receive nearly identical embeddings, and these embeddings form an equilateral triangle or two-dimensional simplex. We therefore observe collapse of features into a set of $n_c=3$ \emph{equi-angular vectors} at the level of word embeddings. The bottom panel of figure \ref{figure:first_experiment} illustrates collapse for the parameters $U$ of the linear layer. We partition the matrix $U$ into vectors $\bu_{k,\,\ell} \in \real^{100}$ via \eqref{eq:upart00} and visualize them once again with PCA. As for the word embeddings, the singular values of the PCA ($\sigma_1=34.9$, $\sigma_2=34.6$ and $\sigma_3 = 0.0003$) reveal that the vectors $\bu_{k,\,\ell}$ essentially live in a two dimensional subspace of $\real^{100}$. % and the PCA paints an accurate picture. 
 We color code each $\bu_{k,\,\ell}$ according to the concepts contained in the corresponding latent variable (say $\bu_{k, \ell}$ is green if  $z_{k, \ell}=$ {\it veggie}, and so forth). The figure indicates that vectors $\bu_{k, \ell}$ that correspond to a same concept collapse around a single vector. A similar analysis applied to the weights of the network $h^*_{W,U}$ tells the same story, provided we examine the actual word features (i.e. the embeddings \emph{after} the LayerNorm) rather than the weights $W$ themselves.
 
In theorem \ref{theorem:1} and \ref{theorem:3} (see section \ref{section:theory}) we prove the correctness of this empirical picture. We show that the weights of $h$ and $h^*$ collapse into the configurations illustrated  on figure \ref{figure:first_experiment} in the large sample limit. Moreover, this limit captures the empirical solution very well. For example, the word embeddings in figure \ref{figure:first_experiment} have a norm equal to $1.41 \pm 0.13$, while we predict a norm of $1.42214$ theoretically. Within the framework of our data model, these theorems provide justification of the fact that entities that play a similar role for a task receive similar representations.

\subsection{The long-tailed case}
At a superficial glance it appears as if the nonlinearity (LayerNorm) plays no essential role, as both networks $h,\,h^*$, in the previous experiment, exhibit feature collapse and generalize perfectly. To probe this issue further, we continue our investigation by conducting a similar experiment (keeping $n_c=3$, $n_w=1200$, $L=15$,  and $K=1000$) but with non-uniform, long-tailed word distributions within each of the $n_c = 3$ concepts. For concreteness, say the \ttt{veggie} concept contains the $400$ words
  \begin{equation*}
   \ttt{potato}, \;\;  \ttt{lettuce}, \;\; \ldots\ldots , \;\;  \ttt{arugula}, \;\; \ttt{parsnip},  \ldots\ldots ,   \;\; \ttt{achojcha}
   \end{equation*}
where \ttt{achojcha} is a rare vegetable that grows in the Andes mountains. We form the \emph{veggie} distribution by sampling \ttt{potato} with probability $C/1$, sampling \ttt{lettuce} with probability $C/2$, and so forth down to \ttt{achojcha} that has probability $C/400$ of being sampled  ($C$ is chosen so that all the probabilities sum to $1$). This ``$1 / i$'' power law distribution has a long-tail, meaning that relatively infrequent words such as  \ttt{arugula} or \ttt{parsnip} collectively capture a significant portion of the mass. 
%The fact that natural data tend to follow long-tailed distributions   \citep{salakhutdinov2011learning, zhu2014capturing, liu2019large, feldman2020does, feldman2020neural} motivates this choice. In practice, datasets of interest almost always display some type of  long-tailed behavior, and these longed-tail behaviors often pose challenges to the learner. 
  Natural data in the form of text or images typically exhibit long-tailed distributions   \cite{salakhutdinov2011learning, zhu2014capturing, liu2019large, feldman2020does, feldman2020neural}. 
  For instance, the frequencies of words in natural  text approximately conform to the  ``$1 / i$'' power law distribution (also known as Zipf's law  \cite{zipf1935}) which motivates the specific choice made in this experiment. Many datasets of interest display some form of long-tail behavior, whether at the level of object occurrences in computer vision or the frequency of words or topics in NLP, and effectively addressing these long-tail behaviors is frequently a challenge for the learner.
     % As we will shortly discuss, these long-tail behaviors often present challenges for the learner.
  
   \begin{figure}[t]
     
     \begin{subfigure}[b]{0.32\linewidth}
                  \includegraphics[width=0.4415\linewidth]{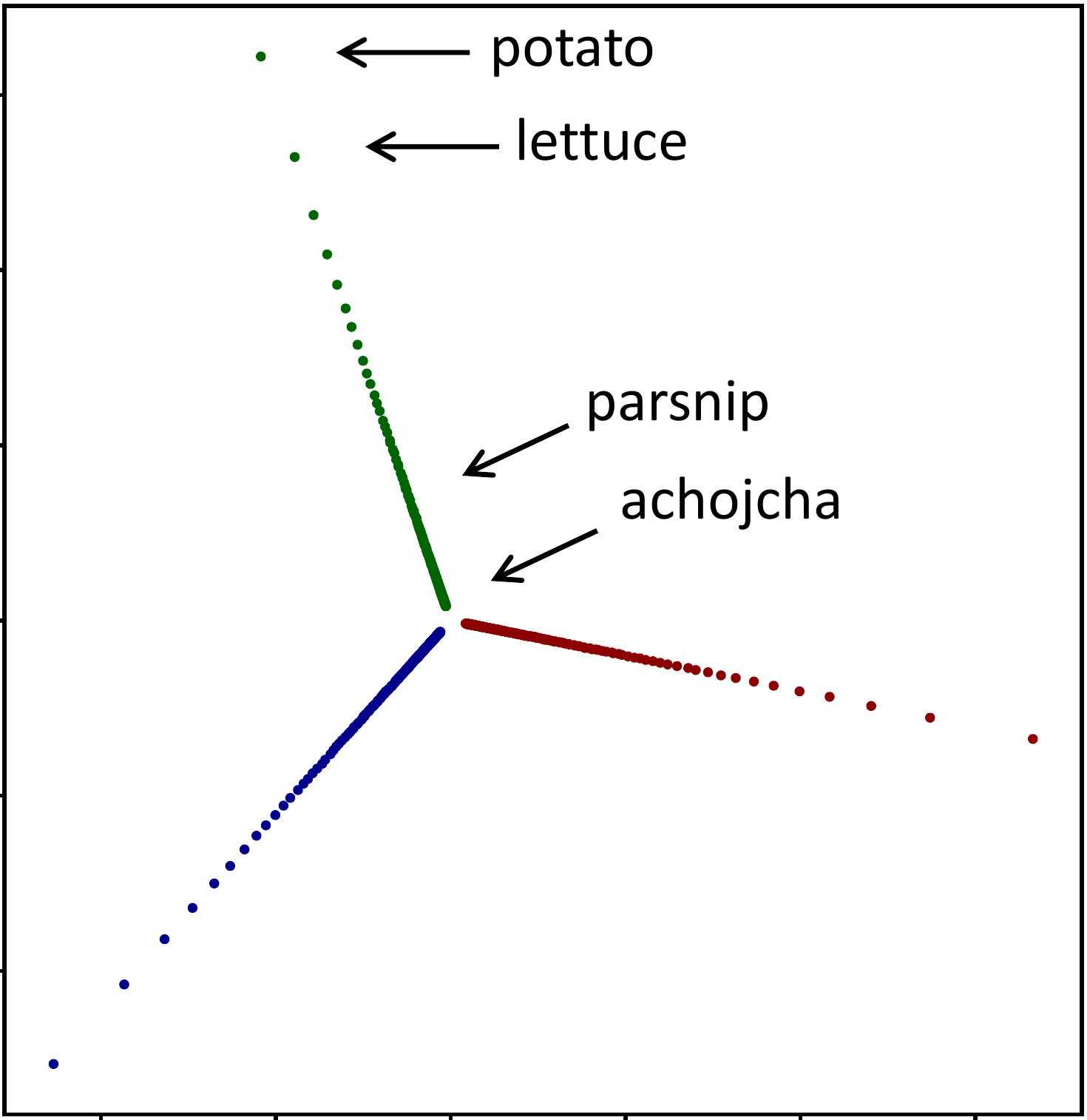}
          \includegraphics[width=0.438\linewidth]{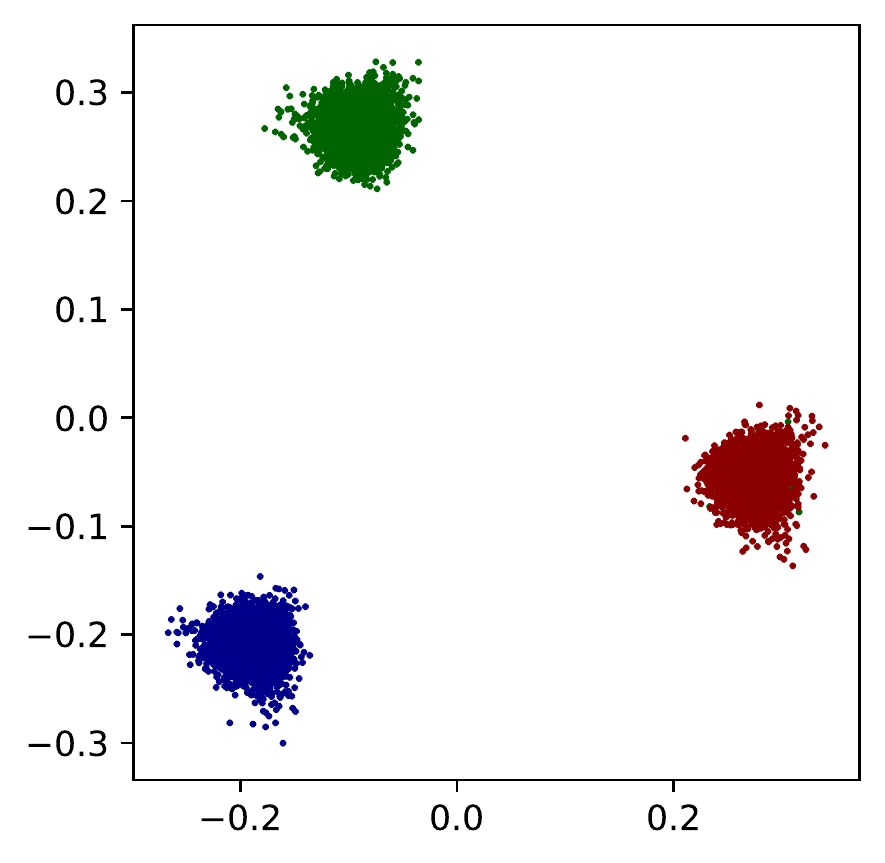} 
         \caption{$h$ trained on the {\bf large} \\ training set. 
         Test acc. $= 100\%$}
     \end{subfigure}
          \begin{subfigure}[b]{0.32\linewidth}
         \includegraphics[width=0.49\linewidth]{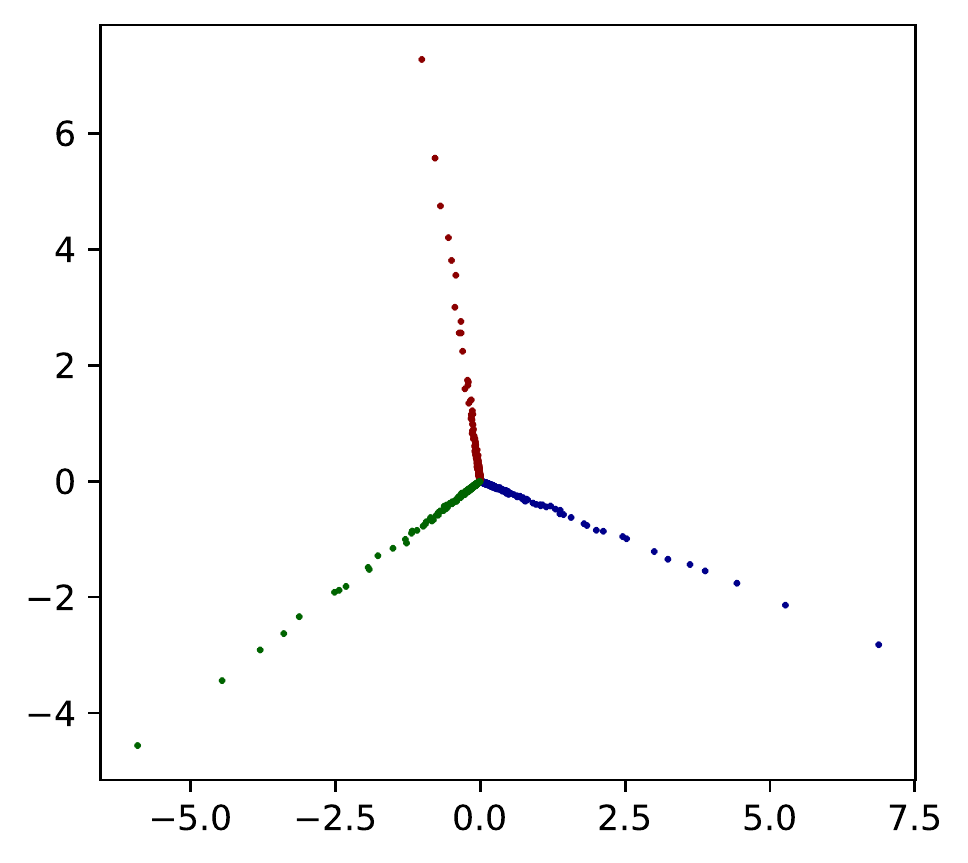} 
          \includegraphics[width=0.485\linewidth]{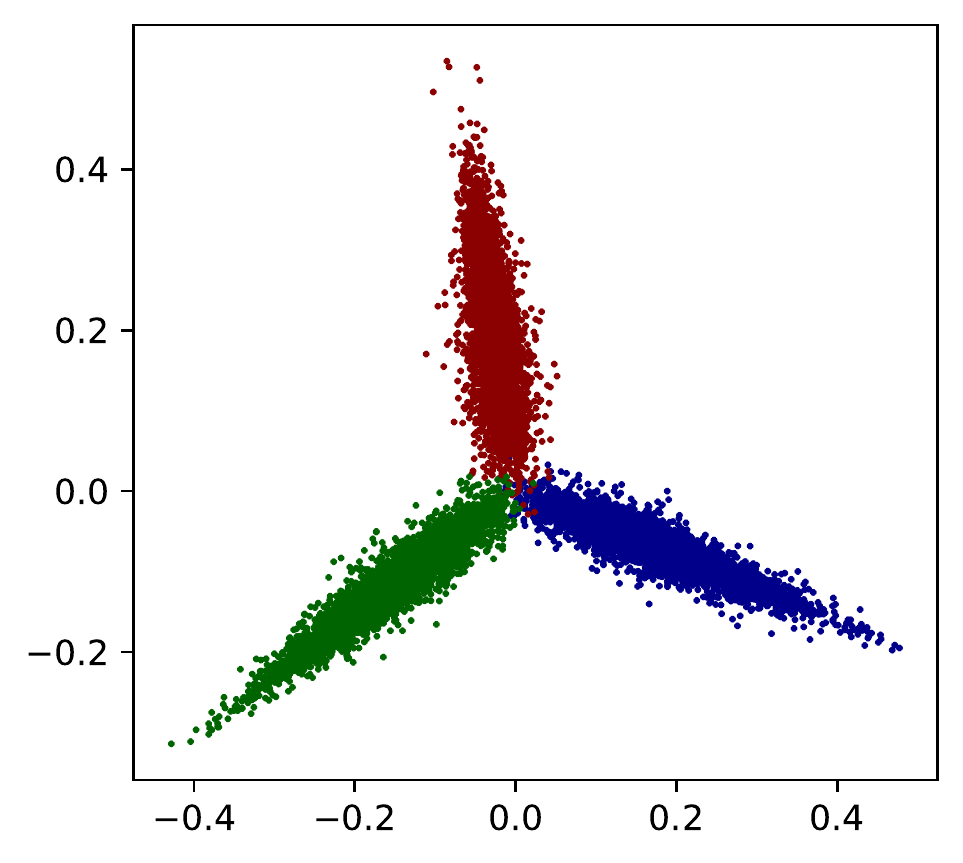} 
                  \caption{$h$ trained on the {\bf small}\\ training set. 
         Test acc. $= 45\%$}
     \end{subfigure}
          \hspace{0.25cm}
          \begin{subfigure}[b]{0.32\linewidth}
         \includegraphics[width=0.438\linewidth]{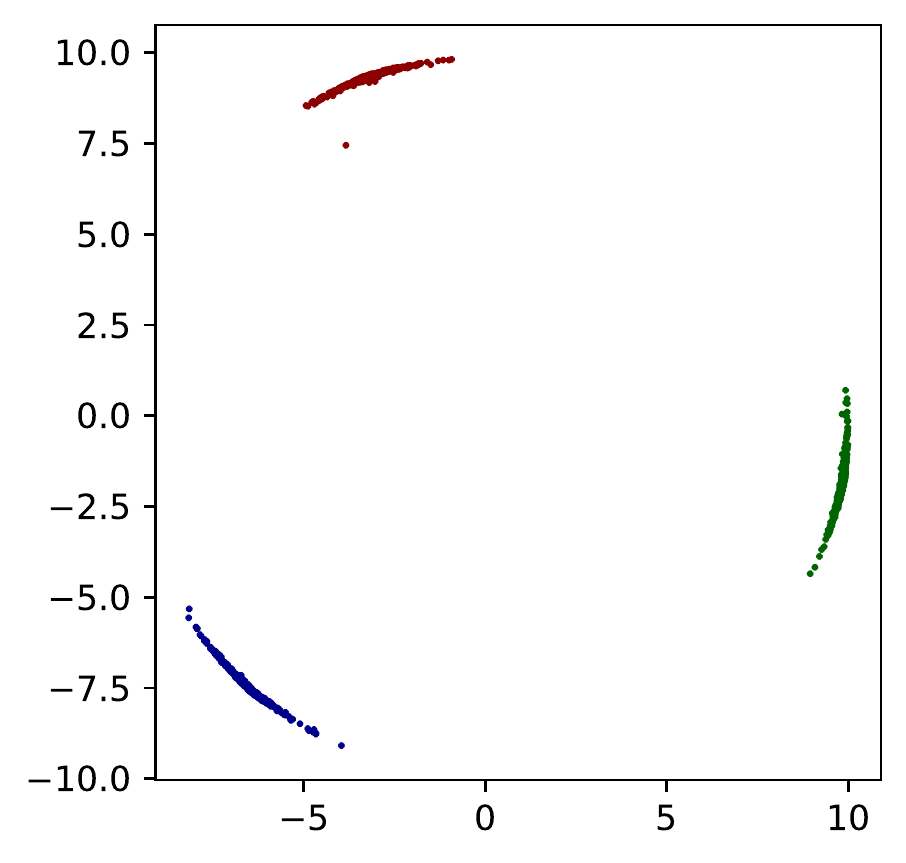} 
          \includegraphics[width=0.455\linewidth]{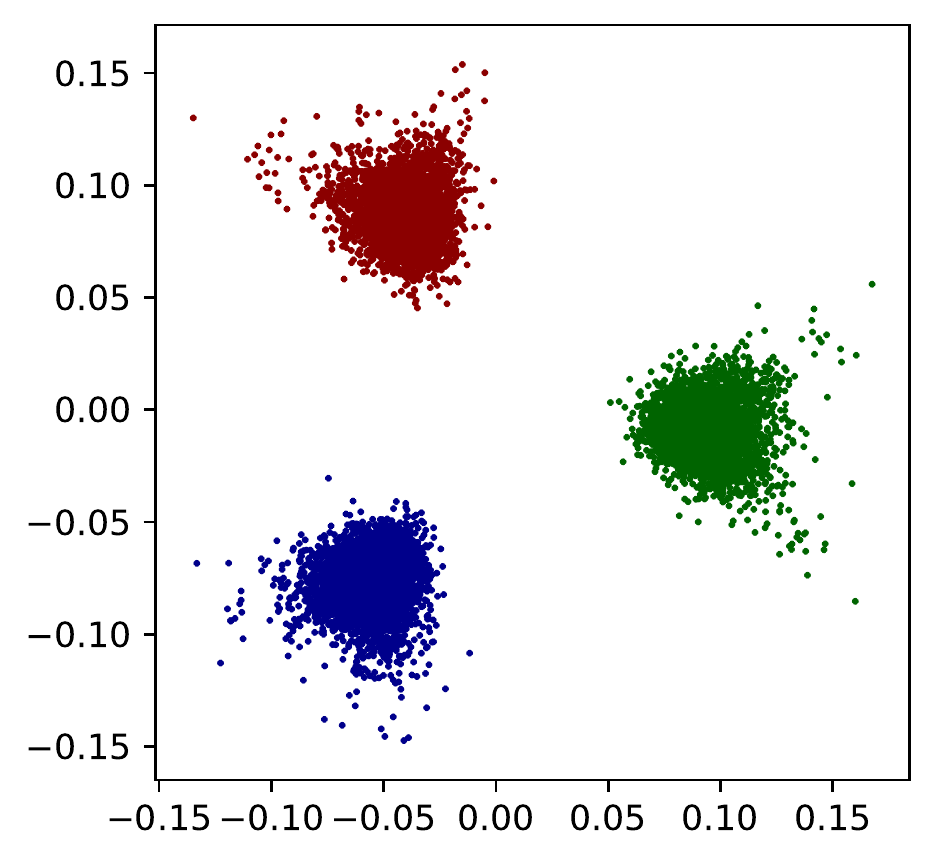} 
            \caption{$h^*$ trained on the {\bf small}\\ training set. 
         Test acc. $= 100\%$}
     \end{subfigure}
            \caption{Visualization of matrices $W$ (left in each subfigure) and $U$ (right in each subfigure)  }
        \label{figure:zipf}
\end{figure}

To investigate the impact of a long-tailed word distributions, we first randomly select the latent variables  $\bz_1, \ldots, \bz_{1000}$ uniformly at random as before. We then use them to build two distinct training sets. We build a large training set by generating $\nspl=500$ training points per latent variable and a small training set by generating $\nspl=5$ training points per latent variable. We use the ``$1 / i$'' power law distribution when sampling words from concepts in both cases. We then train  $h$ and $h^*$ on both training sets and evaluate their generalization performance. When trained on the large training set, both are $100\%$ accurate at test time  (as they should be --- the large training set has $500,000$ total samples). A significant difference emerges between the two networks when trained on the small training set. The network $h$ achieves a test accuracy of $45\%$ while $h^*$ remains $100\%$ accurate.

 %{\bf Network  $\bf h$ on the large training set:}
  We once again visualize the weights of each network to study the relationship between generalization and collapse. Figure \ref{figure:zipf}(a)  depicts the weights of $h_{W,U}$ (via dimensionality reduction and color coding) after training   on the  large  training set. The word embeddings are on the left sub-panel and  the linear weights $\bu_{k,\,\ell}$ on the right sub-panel. Words that belong to the same concept still receive embeddings that are aligned, however, the magnitude of these embeddings depends upon word frequency. The most frequent words in a concept  (e.g. \ttt{potato})  have the largest embeddings while the least frequent words (e.g. \ttt{achojcha}) have the smallest embeddings. In other words, we observe `directional collapse' of the embeddings, but the magnitudes do not collapse. In contrast, the linear weights $\bu_{k, \ell}$ mostly concentrate around three well-defined, equi-angular locations; they collapse in both direction and magnitude. 

 A major contribution of our work (c.f. theorem \ref{theorem:2} in the next section) is a theoretical insight that explains the configurations observed in figure  \ref{figure:zipf}(a), and in particular, explains why the magnitudes of word embeddings depend on their frequencies.

%In theorem \ref{theorem:2} we theoretically justify the configuration observed in figure  \ref{figure:zipf}(a), and in particular, explain why the magnitude of word embeddings must depend on their frequencies.

%{\bf Network  $\bf h$ on the small training set:} 
 Figure \ref{figure:zipf}(b) illustrates the weights of $h_{W,U}$ after training on the small training set. While the word embeddings exhibit a similar pattern as in figure \ref{figure:zipf}(a), the linear weights $\bu_{k, \ell}$ remain dispersed and fail to collapse. This leads to poor generalization performance ($45\%$ accuracy at test time).

To summarize, when the training set is large, the linear weights $\bu_{k, \ell}$ collapse correctly and the network $h_{W,U}$ generalizes well. When the training set is small the linear weights fail to collapse, and the network fails to generalize. This phenomenon can be attributed to the long-tailed nature of the word distribution. To see this, say that
 $$
%\bz_k =  [ \text{dairy},  \text{dairy}, \text{veggie},  \ldots, \text{meat} , \text{veggie} ] 
\bz_k =  [\, \text{veggie}, \, \text{dairy}, \, \text{veggie}, \, \ldots, \,\text{meat} ,\, \text{dairy}\, ] 
 $$
represents the $k^{{\rm th}}$ latent variable for the sake of concreteness. With only $\nspl = 5$ samples for this latent variable, we might end up in a situation where the 5 words selected to represent the first occurrence of the \ttt{veggie} concept have very different frequencies than  the five words selected  to represent the third occurrence of the \ttt{veggie} concept.  Since word embeddings have magnitudes that depend on their frequencies, this will result in a serious imbalance between the vectors $\bu_{k,\,1}$ and  $\bu_{k,\,3}$ that code for the first and third occurrence of the \ttt{veggie} concept. This leads to two vectors $\bu_{k,\,1},\,\bu_{k,\,3}$ that code for the same concept but have different magnitudes (as seen on  figure  \ref{figure:zipf}(b)), so features do not properly collapse. This imbalance results from the `noise' introduced by sampling only $5$ training points per latent variable. Indeed, if $\nspl=500$ then each occurrence of the veggie concept will exhibit a similar mix of frequent and rare words, $\bu_{k,\,1}$ and $\bu_{k,\,3}$ will have roughly same magnitude, and full collapse will take place (c.f. figure  \ref{figure:zipf}(a)).  
Finally, the poor generalization ability of $h_{W,U}$ when the training set is small really stems from the long-tailed nature of the word distribution. The failure mechanism occurs due to the relatively balanced mix of rare and frequent words that occurs with long-tailed data. If the data were dominated by a few very frequent words, then all rare words combined would just contribute small perturbations and would not adversely affect performance.

 %{\bf Network  $\bf h^*$ on the small training set:}  
  We conclude this section by examining the weights of the network $h^*_{W,U}$ after training on the small training set. The left panel of figure  \ref{figure:zipf}(c) provides a visualization  of  the word  embeddings \emph{after} the LayerNorm module. These word representations collapse both in \emph{direction} and \emph{magnitude}; they do not depend on word frequency since the LayerNorm forces vectors to have identical magnitude.  The right panel of  figure \ref{figure:zipf}(c)  depicts the linear weights $\bu_{k, \,\ell}$ and shows  that they properly collapse. As a consequence, $h^*_{W,U}$ generalizes perfectly (100\% accurate) even with only $\nspl=5$ sample per class. Normalization plays a crucial role by ensuring that word representations do not depend upon word frequency. In turn, this prevents the undesired mechanism that  causes $h_{W,U}$ to have uncollapsed linear weights $\bu_{k, \ell}$ when trained on the small training set.
%In essence, normalization ensures that word representations are frequency independent. As a result, the problematic mechanism that led $h_{W,U}$ to have uncollapsed, noisy linear weights $\bu_{k,l}$  during training on a small dataset is avoided. Overall, layerNorm proves to be crucial in achieving good feature collapse and good generalization.
% In summary, the normalization   prevents the word representations from being frequency dependent. This in turn prevent the undesirable mechanism that caused $h_{W,U}$ to have un-collapsed, noisy linear weights $\bu_{k,l}$, when trained on the small training set. 
%In this way, layerNorm plays a crucial role in achieving effective feature collapse and optimal generalization.
%The layerNorm, overall,  reveal itself to play a key role toward obtaining  good feature collapse and good generalization.
Theorem \ref{theorem:3} in the next section proves the correctness of this picture. The weights of the network $h^*$ collapse to the `frequency independent' configuration of figure  \ref{figure:zipf}(c) in the large sample limit.
      
\section{Theory} \label{section:theory}

While the empirical results paint a clear picture, a handful of compelling experiments alone do not constitute strong evidence. Nevertheless, our main contributions show that these experiments properly illustrate the truth of the matter. We start by proving that the weights of the network $h_{W,U}$ collapse into the configurations in figure \ref{figure:first_experiment}  when words have identical  frequencies  (c.f. theorems \ref{theorem:1}). In theorem \ref{theorem:2} we provide theoretical justification of the fact that, when words have distinct frequencies, the word embeddings of  $h_{W,U}$ must depend on frequency in the manner that figure \ref{figure:zipf}(a) illustrates. Finally, in theorem \ref{theorem:3} we show that the weights of the network $h^*_{W,U}$ exhibit full collapse  even when words have distinct frequencies.
Each of these theorems hold in the large $n_{{\rm spl}}$ limit and under some symmetry assumptions on the latent variables (see the appendix for all proofs).    When taken together, these theorems provide a solid theoretical understanding, at least within the context of our data model, of the empirically well-known facts 
that entities that play similar roles for a task receive similar representations, and that normalization is key in order to obtain good representations.

%that entities that plays a similar role for a task are given similar representations by the network.

%\paragraph{Notation} 
\paragraph{Notation.} 
The set of concepts, which up to now was $\mathcal C = \{\text{veggie, dairy, meat}\}$, will  be represented in this section by the more   abstract  $\mathcal C = \{1, \ldots, n_c\}$.
We let $s_c := n_w/n_c$ denote the number of words per concept, and represent the vocabulary by
%$$
%\voc = \voc_1 \cup  \voc_2 \cup \ldots \cup \voc_{n_c} \qquad \text{ where } \qquad  \voc_\alpha =  \{ (\alpha,1), (\alpha,2), \ldots, (\alpha ,s_c) \}.
%$$
\[
\voc \; = \;  \big\{ \; (\alpha,\beta) \in \nat^2 \;  : \;   1 \le \alpha \le n_c  \;\text{ and } \; 1 \le \beta \le s_c \; \big\}
\]
So elements of $\voc$ are tuples of the form $(\alpha,\beta)$ with $1 \le \alpha \le n_c$ and $1 \le \beta \le s_c$, and we think of the tuple  $(\alpha,\beta)$ as representing the $\beta^{{\rm th}}$ word of the $\alpha^{{\rm th}}$ concept.
Each concept $\alpha \in \cC$ comes equipped with a probability distribution $p_\alpha : \{1, \ldots, s_c  \} \to [0,1]$ over the words within it, so that 
$
p_\alpha(\beta)
$
is the probability of selecting the $\beta^{{\rm th}}$ word when sampling out of the $\alpha^{{\rm th}}$ concept. For simplicity we assume that the word distributions within each concept follow identical laws, so that
$$
p_\alpha(\beta) = \mu_{\beta}  \qquad \text{ for all $(\alpha, \beta) \in \voc$} 
$$
for some positive scalars $\mu_\beta>0$ that sum to $1$. We think of  $\mu_\beta$ as being the `frequency' of word $(\alpha,\beta)$ in the vocabulary. For example, choosing $\mu_\beta = 1 / s_c$ gives uniform word distributions while $\mu_\beta \propto 1 / \beta$ corresponds to Zipf's law. We use the definitions

 \vspace{-0.4cm}

\[
\data = \mathcal V^L \qquad \text{ and } \qquad \mathcal Z = \mathcal C^L,
\]

 \vspace{-0.2cm}

for the data space and latent space, respectively. The elements of  data space $\mathcal X$ correspond to sequences $\bx = [ (\alpha_1,\beta_1),\ldots, (\alpha_L,\beta_L)]$ of $L$ words, while elements of the latent space $\mathcal Z$ correspond to sequences $\bz = [\alpha_1,\, \ldots,\,\alpha_L]$ of $L$ concepts. For a given latent variable $\bz$ we write
$ \bx \sim \mathcal D_\bz$
to indicate that the data point $\bx$ was generated by that latent variable.  Formally, $\mathcal D_\bz: \mathcal X \to [0,1]$ is a distribution, whose formula can be found in the appendix.

%\paragraph{Word Embeddings and Word Representations}
\paragraph{Word embeddings, LayerNorm, and word representations.}  We use $\bw_{(\alpha,\beta)} \in \real^d$ to denote the \emph{embedding} of word $(\alpha,\beta) \in \voc$. 
%Recall that the embedding matrix $W \in \real^{d\times n_w}$ contains $n_w$ columns 
The collection of all $\bw_{(\alpha,\beta)}$ determines the columns of the matrix $W \in \real^{d\times n_w}$.
 %which determines a column of $W$ by identify word $(\alpha,\beta)$ with its one-hot encoding.
 %and identify 
 %$\bw_{(\alpha,\beta)}$  as a column of  $W$ when identifying words with their one-hot version.
% We note that  $\bw_{(\alpha,\beta)}$  can also be thought of as a column of  $W \in \real^{d\times n_w}$.
 These embeddings feed into a LayerNorm module \emph{without learnable parameters}
\begin{equation*}
 \vphi(\bv)= \frac{\bv - {\rm mean}(\bv)  \ones_d }{ \sigma(\bv)}    \quad \text{where } \quad
{\rm mean}(\bv) = \frac{1}{d} \sum_{i=1}^d v_i
\;\;\; \text{and} \;\;\;
\sigma^2(\bv) =  \frac{1}{d} \sum_{i=1}^d  \big(v_i - {\rm mean}(\bv) \big)^2,
\end{equation*}
producing outputs in the form of word features. So the LayerNorm module converts a word embedding $\bw_{(\alpha,\,\beta)}$ into a word feature $\varphi(\bw_{(\alpha,\,\beta)})$, and we call this feature a \emph{word representation}. 

 \paragraph{Equiangular vectors.}  
 We call a collection of $n_c$ vectors $\cpt_1, \ldots, \cpt_{n_c} \in \real^d$ \emph{equiangular} if the relations
  \begin{equation}
  \sum_{\alpha=1}^{n_c} \cpt_\alpha =0   \qquad \text{ and }  \qquad  
\dotprod{\cpt_\alpha , \cpt_{\alpha'} } = \begin{cases}
1 & \text{if } \alpha = \alpha'\\
-1/(n_c-1)   & \text{otherwise} 
\end{cases}
 \label{equi} 
\end{equation}
hold for all possible pairs $\alpha, \alpha' \in [n_c]$ of concepts. For example, three vectors $\cpt_1,\cpt_2, \cpt_3 \in \real^{100}$ are equiangular exactly when they have unit norms, live in a two dimensional subspace of $\real^{100}$, and form the vertices of an equilateral triangle in this subspace. This example exactly corresponds to the configurations in figure \ref{figure:first_experiment} and \ref{figure:zipf} (up to a scaling factor). Similarly, four vectors $\cpt_1,\cpt_2, \cpt_3, \cpt_4 \in \real^{100}$ are equiangular when they have unit norms, live in a three dimensional subspace of $\real^{100}$ and form the vertices of a regular tetrahedron in this subspace. The neural collapse literature refers to satisfying \eqref{equi} as the vertices of the  `Simplex Equiangular Tight Frame,' but we use  \emph{equiangular} for the sake of conciseness. We will sometimes require $\cpt_1, \ldots, \cpt_{n_c} \in \real^d$ to also satisfy
\begin{equation*}
\dotprod{\cpt_\alpha , \ones_{d}} = 0 \label{modified2}  \qquad \text{ for all } \alpha \in [n_c],
  \end{equation*}
in which case we say $\cpt_1, \ldots, \cpt_{n_c} \in \real^d$ form a collection of \emph{mean-zero equiangular vectors}.

%\paragraph{Parameter Configurations} 
\paragraph{Collapse configurations.} 
Our empirical investigations reveal two distinct candidate solutions for the features $(W,\,U)$ of the network $h_{W,\,U}$  and one candidate solution for the features $(\varphi(W),\,U)$ of the network $h^{*}_{W,\,U}$. We therefore isolate each of these possible candidates as a definition before turning to the statements of our main theorems. We begin by defining the type of collapse observed when training the network $h_{W,\,U}$ with uniform word distributions (c.f. figure \ref{figure:first_experiment}).
\begin{definition}[Type-I Collapse] \label{def:1} The weights $(W,U)$ of the network $h_{W,U}$ form a type-I collapse configuration if and only if the conditions
\begin{enumerate}%[topsep=0pt]
\item[\rm i)] There exists $c \ge 0$ so that $\bw_{(\alpha,\beta)} = c    \,\cpt_\alpha$ \;  for all $(\alpha,\beta) \in \voc$.
\item[\rm ii)] There exists $c' \ge 0$ so that $\bu_{k, \ell} =  c'\,  \cpt_\alpha$ \; for all  $(k,\ell)$ satisfying $z_{k, \ell}=\alpha$ and all $\alpha \in \mathcal C$.
\end{enumerate}
hold for some collection $\cpt_1, \ldots, \cpt_{n_c} \in \real^d$ of equiangular vectors.
\end{definition}
\noindent Recall that the network $h^{*}_{W,\,U}$ exhibits collapse as well, up to the fact that the word representations $\varphi(\bw_{\alpha,\,\beta})$ collapse rather than the word embeddings themselves. Additionally, the LayerNorm also fixes the magnitude of the word representations. We isolate these differences in the next definition.
\begin{definition}[Type-II Collapse]   \label{def:2} The weights $(W,U)$ of the network $h^*_{W,U}$ form a type-II collapse configuration if and only if the conditions 
\begin{enumerate}%[topsep=0pt]
\item[\rm i)] $\vphi(\bw_{(\alpha,\beta)}) = \sqrt{d}    \,\cpt_\alpha$ \;  for all $(\alpha,\beta) \in \voc$.
\item[\rm ii)] There exists $c \ge 0$ so that $\bu_{k, \ell} =  c\,  \cpt_\alpha$ \; for all  $(k,\ell)$ satisfying $z_{k,\ell}=\alpha$ and all $\alpha \in \mathcal C$.
\end{enumerate}
hold for some collection $\cpt_1, \ldots, \cpt_{n_c} \in \real^d$ of mean-zero equiangular vectors.
\end{definition}
\noindent Finally, when training the network $h_{W,\,U}$ with non-uniform word distributions (c.f. figure  \ref{figure:zipf}(a)) we observe collapse in the direction of the word embeddings $\mathbf{w}_{(\alpha,\,\beta)}$ but their magnitudes depend upon word frequency. We therefore isolate this final observation as
\begin{definition}[Type-III Collapse]  \label{def:3}  The weights $(W,U)$ of the network $h_{W,U}$ form a type-III collapse configuration if and only if
\begin{enumerate}
\item[\rm i)] There exists positive scalars $r_{\beta} \ge 0$ so that $\bw_{(\alpha,\,\beta)} =  r_\beta  \,\cpt_\alpha$ \;  for all $(\alpha,\beta) \in \voc$.
\item[\rm ii)] There exists $c \ge 0$ so that $\bu_{k, \ell} =  c\,  \cpt_\alpha$ \; for all  $(k,\ell)$ satisfying $z_{k, \ell}=\alpha$ and all $\alpha \in \mathcal C$.
\end{enumerate}
hold for some collection $\cpt_1, \ldots, \cpt_{n_c} \in \real^d$ of equiangular vectors.
\end{definition}
\noindent In a type-III collapse we allow the word embedding $\bw_{(\alpha,\,\beta)}$ to have a frequency-dependent magnitude $r_\beta $ while in type-I collapse we force all embeddings to have the same magnitude; this makes type-I collapse a special case of type-III collapse, but not vice-versa.

\subsection{Proving collapse}
Our first result proves that the words embeddings $\bw_{(\alpha,\,\beta)}$ and linear weights $\bu_{k,\,\ell}$ exhibit type-I collapse in an appropriate large-sample limit. When turning from experiment (c.f. figure \ref{figure:first_experiment}) to theory we study the true risk
\begin{align} 
&\mathcal R(W,U) = \frac{1}{K}  \sum_{k=1}^{K}   \E_{  \;\;\bx \sim \mathcal D_{\bz_k}} \Big[  \ell(h_{W,U}(\bx) , k ) \Big]  + \frac{\lambda}{2} \left( \|W\|_F^2 +  \|U\|_F^2  \right)  \label{true_risk}  
\end{align}
rather than the empirical risk $\mathcal R_{\rm emp}(W,U)$ and place a symmetry assumption on the latent variables.
\begin{assumption}[Latent Symmetry] \label{assumption:symmetry}
 For every $k \in [K]$, $r \in [L]$, $\ell \in[L]$, and $\alpha \in [n_c]$ the identities
\begin{equation} \label{sym00}
 \left| \Big\{ k' \in[K]   : \dist(\bz_k,\bz_{k'}) = r \text{ and }  z_{k'\!,\ell} = \alpha \Big\} \right| =  \begin{cases}
\frac{K}{|\mathcal Z|}  { L-1 \choose r} (n_c-1)^r  & \text{ if } z_{k,\ell} = \alpha \\ \\
\frac{K}{|\mathcal Z|}  { L-1 \choose r-1} (n_c-1)^{r-1} & \text{ if } z_{k,\ell} \neq \alpha 
  \end{cases}
\end{equation}
hold, with $\dist(\bz_k, \bz_{k'})$ denoting the Hamming distance between a pair ($\bz_k, \bz_{k'}$) of latent variables.
\end{assumption}
% A simple combinatorial calculation shows that this symmetry assumption is satisfied, for example, in the extreme case where
% $K = |\mathcal Z|$ and $\{ \bz_1, \ldots, \bz_K\} = \mathcal Z$ (meaning that all possible latent variables are present).
% As a consequence, if we select $K$  latent variables uniformly at random in $\mathcal Z$, as we did during our experiments,  we  expect this symmetry assumption to approximately hold. We now state
%In particular, if $K = |\mathcal Z|$ and $\{ \bz_1, \ldots, \bz_K\} = \mathcal Z$ then the symmetry assumption holds by simple combinatorics. As a consequence, if we select $K$  latent variables uniformly at random then we expect the symmetry assumption to approximately hold.
With this assumption in hand we may state our first main result
\begin{theorem}[Full Collapse of $h$]\label{theorem:1}
Assume uniform sampling $\mu_\beta = 1 / s_c\,$ for each word distribution. Let $\tau \ge 0$ denote the unique minimizer of the strictly convex function
$$
H(t) :=  \log \left( 1 -  \frac{K}{n_c^L}  +  \frac{K}{n_c^L} \Big( 1+ (n_c-1) e^{-\eta t } \Big)^L \right) + \lambda t \qquad \text{where }\quad  \eta =    \frac{n_c}{n_c-1} \; \frac{1}{\sqrt{n_w KL}}
$$
and assume that the latent variables  $\bz_1,\,\ldots,\,\bz_K$ are mutually distinct and satisfy the symmetry assumption \ref{assumption:symmetry}. Then any $(W,U)$ in a type-I collapse configuration 
with constants  $c = \sqrt{\tau / n_w}$ and $c' = \sqrt{\tau /(KL)}$ 
is a global minimizer of \eqref{true_risk}.
\end{theorem}

We also prove two strengthenings of this theorem in the appendix. First, under an additional technical assumption on the latent variables $\bz_1, \ldots, \bz_K$ we prove its converse; any $(W,U)$ that minimizes \eqref{true_risk} must be in a  type-I collapse configuration (with the same constants $c, c'$). This additional assumption is mild but technical, so we state it in  section  \ref{section:sym} of the appendix. We also prove that if $d>n_w$ then $\mathcal R(W,U)$ does not have spurious local minimizers; all local minimizers are global (see appendix \ref{section:spurious}).

The symmetry assumption, while odd at a first glance, is both needed and natural.  Indeed, a type-I collapse configuration is highly symmetric and perfectly homogeneous. We therefore expect that such configurations could only solve an analogously `symmetric' and `homogeneous' optimization problem. In our case this means using the true risk  \eqref{true_risk} rather than the empirical risk \eqref{empirical_risk}, and imposing that the latent variables satisfy the symmetry assumption. 
This assumption means that all latent variables play interchangeable roles, or at an intuitive level, that there is no `preferred' latent variable. 
 To understand this better, consider the extreme case $K=n_c^L$ and $\{\bz_1, \ldots, \bz_K\} = \mathcal Z$, meaning that  all latent variables in $\mathcal Z$ are involved in the task.  The identity \eqref{sym00} then holds by simple combinatorics. 
 We may therefore think of  \eqref{sym00} as an equality that holds in the large $K$ limit, so it is neither impossible nor  unnatural. We refer to section  \ref{section:sym} of the appendix for a more in depth discussion of assumption \ref{assumption:symmetry}.

%This assumption means that all latent variables play interchangeable roles, and that there is no `preferred' latent variable. The extreme case $K=n_c^L$ and $\{\bz_1, \ldots, \bz_K\} = \mathcal Z$ corresponds to a situation in which all latent variables in $\mathcal Z$ are involved in the task. In this case,  the identity \eqref{sym}   holds by simple combinatorics. 
% We may therefore think of  \eqref{sym} as an equality that holds in the large $K$ limit, so it is neither impossible nor  unnatural.

While  theorem \ref{theorem:1} proves global optimality of  type-I collapse configurations in the limit of large $\nspl$ and large $K$, these solutions still provide valuable predictions when $K$ and $\nspl$ have small to moderate values. For example, in the setting of figure \ref{figure:first_experiment} ($\nspl=5$ and $K=1000$) the theorem predicts that word embeddings should  have a norm $c = \sqrt{\tau / n_w} =1.42214$ (with $\tau$ obtained by minimizing  $H(t)$ numerically). By experiment we find that, on average, word embeddings have norm $1.41$ with standard deviation $0.13$. To take another example, when $K=50$ and $\nspl=100$ (and keeping $n_c=3$, $n_w=1200$, $L=15$) the theorem predicts that words embeddings should have norm $0.61602$. This compares well against the values $0.61 \pm 0.06$ observed in experiments. The idealized solutions of the theorem capture their  empirical counterparts very well.

%Theorem \ref{theorem:1} holds in the large sample limit and under the  symmetry assumption on the latent variables. In contrast, the experiment depicted in figure \ref{figure:first_experiment} has only $5$ samples per class and the symmetry assumption holds only approximately. Despite these discrepancies, the idealized theoretical solution captures the empirical solution very well. For example, theorem \ref{theorem:1} predicts that word embeddings should  have a norm $c = \sqrt{\tau / n_w} =1.42214$ (with $\tau$ obtained by minimizing  $H(t)$). By experiment we find that, on average, word embeddings have norm $1.41$ with standard deviation $0.13$.

%The theorem predicts the observed magnitude of collapsed features with uniform sampling. Indeed, if either condition $c = (\tau / n_w)^{1/2}$ or $c' = (\tau /(KL))^{1/2}$ fails then the corresponding collapsed weights do not even yield a critical point of the true risk. For figure \ref{figure:first_experiment} the therorem predicts a norm equal to $1.42214$ for word embeddings, while empirically word embeddings have a norm equal to $1.41 \pm 0.13$ in experiment. Despite using the true risk rather than the empirical risk with only $n_{{\rm spl}} = 5$ samples per class, the theoretical solution captures the empirical solution very well. 

For non-uniform $\mu_\beta$ we expect $h_{W,\,U}$ to exhibit type-III collapse rather than type-I collapse. Additionally, in our long-tail experiments, we observe that frequent words (i.e. large $\mu_\beta$) receive large embeddings. We now prove that this is the case in our next theorem. To state it, consider the following system of $s_c+1$ equations
\begin{align}
 & \frac{\lambda}{L} \; \frac{r_\beta}{c}  \left(  n_c-1 +  \exp\left(  \frac{n_c}{n_c-1} c\, r_\beta   \right)\right)   =  {\mu_\beta}   \qquad \text{ for all } 1 \le \beta \le s_c \label{00sys1} \\
 &  \sum_{\beta = 1}^{s_c} \left(\frac{r_\beta}{c} \right)^2 = L n_c^{L-1} \label{00sys2}
\end{align}
for the unknowns $(c,r_1, \ldots, r_{s_c})$. If the regularization parameter $\lambda$ is small enough, namely
 \begin{equation}
 \lambda^2<  \frac{L}{n_c^{L+1}} \sum_{\beta=1}^{s_c} \mu_\beta^2 \label{00lambdabound}
\end{equation}
then \eqref{00sys1}--\eqref{00sys2} has a unique solution. This solution defines the magnitudes of the word embeddings.  The left hand side of \eqref{00sys1} is an increasing function of $r_\beta$,  so $\mu_\beta < \mu_{\beta'}$ implies  $r_\beta < r_{\beta'}$ and more frequent words receive larger embeddings.

\begin{theorem}[Directional Collapse of $h$] \label{theorem:2}
Assume $\lambda$ satisfies \eqref{00lambdabound}, $K=n_c^L$ and $\{\bz_1, \ldots, \bz_K\} = \mathcal Z$.  Suppose $(W,\,U)$ is in a type-III collapse configuration for some constants  $(c,r_1, \ldots, r_{s_c})$. Then $(W,U)$ is a critical point of the true risk \eqref{true_risk} if and only if $(c,r_1, \ldots, r_{s_c})$ solve the system \eqref{00sys1}--\eqref{00sys2}. 
\end{theorem}
Essentially this theorem shows that word embeddings \emph{must} depend on word frequency and so feature collapse fails. Even in the fully-sampled case $K=n_c^L$ and $\{\bz_1, \ldots, \bz_K\} = \mathcal Z$ a network exhibiting type-I collapse is never critical if the word distributions are non-uniform. While we conjecture global optimality of the solutions in theorem \ref{theorem:2} under appropriate symmetry assumptions, we have no proof of this yet. Inequality \eqref{00lambdabound} is the natural one for  theorem  \ref{theorem:2}, for if $\lambda$ is too large the trivial solution $(W,U)=(0,0)$ is the only one. 

Our final theorem completes the picture; it shows that normalization restores global optimality of fully-collapsed configurations. For the network $h^{*}_{W,\,U}$ with LayerNorm, we use the appropriate limit
 \begin{align} 
&\mathcal R^*(W,U) = \frac{1}{K}  \sum_{k=1}^{K}   \E_{  \;\;\bx \sim \mathcal D_{\bz_k}} \Big[  \ell(h^*_{W,U}(\bx) , k ) \Big]  + \frac{\lambda}{2}   \|U\|_F^2    \label{true_risk_LN}  
\end{align}
of the associated empirical risk and place no assumptions on the sampling distribution.
\begin{theorem}[Full Collapse of $h^*$] \label{theorem:3} 
Assume the non-degenerate condition $\mu_\beta > 0$ holds. Let $\tau \ge 0$ denote the unique minimizer of the strictly convex function
\[
H^*(t) =  \log \left( 1 -  \frac{K}{n_c^L}  +  \frac{K}{n_c^L} \Big( 1+ (n_c-1) e^{-\eta^* t } \Big)^L \right) + \frac{\lambda}{2} t^2 \qquad  \text{where } \eta^* =    \frac{n_c}{n_c-1} \; \frac{1}{\sqrt{ KL/d}} 
\]
and assume the latent variables $\bz_1, \ldots, \bz_K$ are mutually distinct and satisfy  assumption \ref{assumption:symmetry}. Then any $(W,\,U)$ in a type-II collapse configuration with constant $c = \tau / \sqrt{KL}$   is a  global minimizer of \eqref{true_risk_LN}.
\end{theorem}

As for theorem \ref{theorem:1}, we prove the converse  under an additional technical assumption on the latent variables. Any $(W,U)$ that minimizes \eqref{true_risk_LN} must be in a type-II collapse configuration with $c = \tau / \sqrt{KL}$. The proof and exact statement  can be found in section \ref{section:thm3} of the appendix.

\paragraph{Acknowledgements.}
%Xavier Bresson is supported by NRF Fellowship NRFF2017-10 and NUS-R-252-000-B97-133.
 Xavier Bresson is supported by NUS Grant ID R-252-000-B97-133.

\bibliography{bibliography}
\bibliographystyle{plain}
%\bibliographystyle{unsrt}

%\bibliography{bibliography}
%\bibliographystyle{iclr2023_conference}

%\end{document}
%%%%%%%%%%%%%%%%%%%%%%%%%%%%%
%%%%%%%%%%%%%%%%%%%%%%%%%%%%%
%%%%%%%%%%%%%%%%%%%%%%%%%%%%%
%%%%%%%%%%%%%%%%%%%%%%%%%%%%%
%%%%%%%%%%%%%%%%%%%%%%%%%%%%%
%%%%%%%%%%%%%%%%%%%%%%%%%%%%%
%%%%%%%%%%%%%%%%%%%%%%%%%%%%%
%%%%%%%%%%%%%%%%%%%%%%%%%%%%%
\newpage
\appendix

\newtheorem*{assumption1}{Assumption 1}
\newtheorem*{definition1}{Definition 1}
\newtheorem*{definition2}{Definition 2}
\newtheorem*{definition3}{Definition 3}
\newtheorem*{theorem1}{Theorem 1}
\newtheorem*{theorem2}{Theorem 2}
\newtheorem*{theorem3}{Theorem 3}

 \setcounter{lemma}{0}
  \setcounter{definition}{0}
   \setcounter{proposition}{0}
    \setcounter{assumption}{0}

\renewcommand*{\thetheorem}{\Alph{theorem}}
\renewcommand*{\thelemma}{\Alph{lemma}}
\renewcommand*{\thedefinition}{\Alph{definition}}
\renewcommand*{\thecorollary}{\Alph{corollary}}
\renewcommand*{\theproposition}{\Alph{proposition}}
\renewcommand*{\theassumption}{\Alph{assumption}}
\renewcommand*{\theclaim}{\Alph{claim}}

%%%%%%%%%%%%%%%%%%%%%%%%%%%%%
%%%%%%%%%%%%%%%%%%%%%%%%%%%%%
%%%%%%%%%%%%%%%%%%%%%%%%%%%%%
%%%%%%%%%%%%%%%%%%%%%%%%%%%%%
%%%%%%%%%%%%%%%%%%%%%%%%%%%%%
%%%%%%%%%%%%%%%%%%%%%%%%%%%%%
%%%%%%%%%%%%%%%%%%%%%%%%%%%%%
%%%%%%%%%%%%%%%%%%%%%%%%%%%%%

\begin{center}
{\Large \bf Appendix }
\end{center}

\

Section \ref{section:notation}  provides formulas for the networks $h_{W,U}$ and $h^*_{W,U}$  depicted on figure \ref{figure:net} of the main paper, and formula  for the distribution $\mathcal D_{\bz_k}: \data \to [0,1]$  underlying the data model depicted on figure \ref{figure:datamodel} of the main paper. We also use this section to introduce various notations that our proofs will rely on.

Section \ref{section:sym} is devoted to the symmetry assumptions that we impose on the latent variables.
We start with an in depth discussion of assumption 1 from the main paper. This assumption is  required for theorem \ref{theorem:1} and \ref{theorem:3} to hold. We then present and discuss an additional technical assumption on the latent variables (c.f. assumption \ref{symconv}) that we will use to prove the \emph{converse} of theorems \ref{theorem:1} and \ref{theorem:3}. 

Whereas the first two sections are essentially devoted to notations and discussions, most of the analysis occurs in section  \ref{section:lowerbound}, \ref{section:thm1}, \ref{section:thm3} and \ref{section:thm2}. 
We start by deriving a sharp lower bound for  the unregularized  risk in section \ref{section:lowerbound}. Theorem 1 from the main paper, as well as its converse, are proven in section \ref{section:thm1}.
Theorem \ref{theorem:3} and its converse are proven in section \ref{section:thm3}. Finally we prove theorem \ref{theorem:2} in section \ref{section:thm2}.

We conclude this appendix by proving in section \ref{section:spurious} 
that if $d> \min(n_w,KL)$,  then the risk associated to the network $h_{W,U}$   does not have spurious local minimizers; all local minimizers are global. This proof follows the same strategy that was used in \cite{zhu2021geometric}.

\section{Preliminaries and notations}\label{section:notation}

\subsection{Formula for the neural networks}
Recall that the vocabulary is the set
\[
\voc = \{(\alpha,\beta) \in \nat^2 :  1 \le \alpha \le n_c \text{ and } 1 \le \beta \le s_c\}, 
\]
and that we think of  the tuple $(\alpha,\beta) \in \voc $ as representing 
 the $\beta^{th}$ word of the $\alpha^{th}$ concept. The data space is $\data = \voc^L$, and  a  sentence $\bx \in \mathcal X$ is a sequence of 
 $L$ words:
$$
\bx = [ (\alpha_1,\beta_1),\ldots, (\alpha_L,\beta_L)]  \qquad  1 \le \alpha_\ell \le n_c \text{ and }   1 \le \beta_\ell \le s_c.
$$
The two neural networks $h,h^*$ studied in this work process such a sentence $\bx \in \data$ in multiple steps: 
\begin{enumerate}
\item Each word $(\alpha_\ell,\beta_\ell)$ of the sentence is encoded into  a one-hot vector.
\item These one-hot vectors are multiplied by a matrix $W$ to produce word embeddings that live in a $d$-dimensional space.
\item Optionally (i.e. in the case of the network $h^*$),  these word embeddings  go through a LayerNorm module without learnable parameters. %The neural network $h$ skips this step.
\item The word embeddings are concatenated   and  then goes through a linear transformation $U$. 
\end{enumerate}
We now formalize these 4 steps, and in the process,  we set the notations on which we will rely in all our proofs.

\paragraph{Step 1: One-hot encoding.}  Without loss of generality, we choose the following one-hot encoding scheme: 
  word $(\alpha,\beta) \in \voc $ receives the one-hot vector
which has a $1$ in entry $(\alpha-1) s_c + \beta$ and $0$ everywhere else. To formalize this, we define the one-hot encoding function
\begin{equation} \label{hot}
\hot(\alpha,\beta)  = \be_{ (\alpha-1) s_c + \beta }
\end{equation}
where $\be_i$  denotes the $i^{th}$ basis vector of $\real^{n_w}$.
The  one-hot encoding function $\hot$ can also be applied to a sequence of words.  
Given a sentence $\bx = [ (\alpha_1,\beta_1),\ldots, (\alpha_L,\beta_L)] \in \data$ we let
\begin{equation} \label{hot2}
 \hot(\bx) :=  \begin{bmatrix}
 \vert & \vert & & \vert \\
  \hot(\alpha_1,\beta_1) & \hot(\alpha_2,\beta_2) & \ldots & \hot(\alpha_L,\beta_L)  \\
   \vert & \vert & & \vert 
  \end{bmatrix} \in \real^{n_w \times L}
\end{equation}
and so $\hot$ maps sentences to $n_w \times L$ matrices.

\paragraph{Step 2: Embedding.} The embedding matrix $W$ has $n_w$ columns and each of these columns belongs to $\real^d$.  Since $\hot(\alpha,\beta)$ denote the one-hot vector associated to word $(\alpha,\beta) \in \voc$, we define the \emph{embedding} of word $(\alpha,\beta)$ by 
\begin{equation}
\bw_{(\alpha,\beta)} := W \;  \hot(\alpha,\beta)  \; \in \real^d \label{wordembedding} .
\end{equation}
Due to   \eqref{hot}, this means that $\bw_{(\alpha,\beta)}$ is the $j^{th}$ column of $W$, where $j = (\alpha-1) s_c + \beta $.
The embedding matrix $W$  can  therefore be visualized as follow (for concreteness we choose $n_c=3$ and $n_w=12$ as in figure 1 of the main paper):
\begin{align*}
W = &
\begingroup 
\setlength\arraycolsep{3pt}
\begin{bmatrix}
%\vert & \vert & \vert & \vert & \vert & \vert & \vert & \vert & \vert & \vert & \vert & \vert \\
&\vert & \vert & \vert & \vert & &\vert & \vert & \vert & \vert& & \vert & \vert & \vert & \vert &\\
&\bw_{(1,1)} & \bw_{(1,2)} & \bw_{(1,3)} & \bw_{(1,4)} & \qquad  & \bw_{(2,1)} & \bw_{(2,2)} & \bw_{(2,3)} & \bw_{(2,4)} & \qquad &\bw_{(3,1)} & \bw_{(3,2)} & \bw_{(3,3)} & \bw_{(3,4)}  &\\ \vspace{-0.3cm} \\
%\vert & \vert & \vert & \vert & \vert & \vert & \vert & \vert & \vert & \vert & \vert & \vert \\
&\vert & \vert & \vert & \vert & & \vert & \vert & \vert & \vert & & \vert & \vert & \vert & \vert &
\end{bmatrix}
\endgroup
\\
& \;\;  \underbrace{\hspace{3.6cm}}_{\text{\tiny Embeddings of the  words  in the $1^{st}$ concept.}}
\hspace{0.5cm}\underbrace{\hspace{3.5cm}}_{\text{\tiny Embeddings of the  words in the $2^{nd}$ concept.}}
\hspace{0.5cm}\underbrace{\hspace{3.5cm}}_{\text{\tiny Embeddings of the  words in the $3^{rd}$ concept.}}
 \end{align*}
Given a sentence $\bx = [ (\alpha_1,\beta_1),\ldots, (\alpha_L,\beta_L)] \in \data$, appealing to \eqref{hot2} and \eqref{wordembedding}, we find that   
\begin{equation} \label{zebra}
W \hot(\bx) =
\begingroup 
\begin{bmatrix}
\vert & \vert &  & \vert  \\
\bw_{(\alpha_1,\beta_1)} & \bw_{(\alpha_2,\beta_2)} & \cdots & \bw_{(\alpha_L, \beta_L)} \\ \vspace{-0.3cm} \\
\vert & \vert &  & \vert  \\
\end{bmatrix}
\endgroup
\in \real^{d \times L}
\end{equation}
and therefore  $W \hot(\bx)$ is the matrix that contains the $d$-dimensional embeddings of the words that constitute the sentence $\bx \in \data$.

 \paragraph{Step 3: LayerNorm.}  Recall from the main paper that the  LayerNorm function $\vphi: \real^d \to \real^d$ is defined by
 \begin{equation*}
 \vphi(\bv)= \frac{\bv - {\rm mean}(\bv)  \ones_d }{ \sigma(\bv)}    \quad \text{where } \quad
{\rm mean}(\bv) = \frac{1}{d} \sum_{i=1}^d v_i
\;\;\; \text{and} \;\;\;
\sigma^2(\bv) =  \frac{1}{d} \sum_{i=1}^d  \big(v_i - {\rm mean}(\bv) \big)^2,
\end{equation*}
We will often apply this function column-wise to a matrix. For example if $V$ is the $d \times m$ matrix
$$
V = \begin{bmatrix}
& \vert & \vert &  & \vert & \\
&\bv_1 &\bv_2 & \cdots & \bv_m &\\
&\vert & \vert & & \vert &
\end{bmatrix}, 
\qquad \text{then} \qquad 
\vphi(V) = \begin{bmatrix}
& \vert & \vert &  & \vert & \\
&\vphi(\bv_1) &\vphi(\bv_2) & \cdots & \vphi(\bv_m )&\\
&\vert & \vert & & \vert &
\end{bmatrix} 
$$
Applying $\vphi$ to \eqref{zebra} gives 
\begin{equation} \label{zebra2}
\vphi\Big( W \hot(\bx) \Big) =
\begingroup 
\begin{bmatrix}
\vert & \vert &  & \vert  \\ \vspace{-0.1cm}  \\
\vphi\Big( \bw_{(\alpha_1,\beta_1)}  \Big) & \vphi\Big( \bw_{(\alpha_2,\beta_2)}  \Big) & \cdots & \vphi\Big( \bw_{(\alpha_L, \beta_L)}  \Big)  \\ \vspace{-0.1cm} \\
\vert & \vert &  & \vert  \\
\end{bmatrix}
\endgroup
\in \real^{d \times L}
\end{equation}
and so $\vphi\left( W \hot(\bx) \right)$
contains the \emph{word representations} of the words from the input sentence (recall that by word representations we mean the word embeddings \emph{after} the LayerNorm). 

\paragraph{Step 4: Linear Transformation.} Recall from the main paper that 
\begin{align}\label{eq:upart}
U = \begin{bmatrix}
 \hor\bu_{1,1} \hor &  \hor\bu_{1,2} \hor & \cdots  &\hor\bu_{1,L} \hor \\
 \hor\bu_{2,1} \hor &  \hor\bu_{2,2} \hor & \cdots  &\hor\bu_{2,L} \hor \\
 \vdots  &\vdots  & & \vdots \\
  \hor\bu_{K,1} \hor &  \hor\bu_{K,2} \hor & \cdots  &\hor\bu_{K,L} \hor 
\end{bmatrix} \in \real^{K \times Ld}
\end{align}
where each vector $\bu_{k,\ell}$ belongs to $\real^d$. The neural networks $h_{W,U}$ and $h_{W,U}^*$ are then given by the formula
\begin{align}
& h_{W,U}(\bx) = U \;  \text{Vec} \left[  W \hot(\bx)   \right] \\
& h_{W,U}^*(\bx) = U \;  \text{Vec} \left[  \vphi\Big( W \hot(\bx) \Big)   \right]
\end{align}
where ${\rm Vec}: \real^{d \times L} \to \real^{dL}$ is the function that takes as input a $d \times L$ matrix and flatten it out into a vector with $dL$ entries (with the first column filling the first $d$ entries of the vector, the second column filling the next $d$ entries, and so forth). It will prove convenient to gather the $L$ vectors $\bu_{k,\ell}$ that constitute the $k^{th}$ row of $U$ into the matrix
\begin{equation} \label{U_k}
\hat U_k  = 
 \begin{bmatrix}
 \vert & \vert &  & \vert  \\
\bu_{k,1} &\bu_{k,2} & \cdots & \bu_{k,L} \\
\vert & \vert & & \vert 
\end{bmatrix} \in \real^{d \times L}
\end{equation}
With this notation, we have the following alternative expressions for the networks $h_{W,U}$ and $h_{W,U}^*$
\begin{equation} \label{def:net}
 h_{W,U}(\bx) = \begin{bmatrix} 
  \vspace{-0.2cm} \\
 \dotprodbig{ \; \hat U_1 \; , \;  W \, \hot(\bx)} \\  \vspace{-0.2cm} \\
 \dotprodbig{\; \hat U_2 \;, \; W \, \hot(\bx)} \\  \\
 \vdots \\   \\
 \dotprodbig{ \; \hat U_K  \;, \; W \hot(\bx)} \\ \vspace{-0.2cm}
 \end{bmatrix}
 \qquad \text{ and } \qquad
  h_{W,U}^*(\bx) = \begin{bmatrix} 
  \vspace{-0.2cm} \\
 \dotprodbig{ \; \hat U_1 \; , \;  \vphi\Big( W \, \hot(\bx) \Big) }  \\  \vspace{-0.2cm} \\
 \dotprodbig{\; \hat U_2 \;, \;  \vphi\Big( W \, \hot(\bx)  \Big)} \\  \\
 \vdots \\   \\
 \dotprodbig{ \; \hat U_K  \;, \;  \vphi\Big( W \hot(\bx)  \Big)} \\ \vspace{-0.2cm}
 \end{bmatrix}
 \end{equation}
 where $\langle \cdot, \cdot \rangle_F$ denote the Frobenius inner product between matrices (see next subsection for a definition).
 
  Finally, we use $\hat U$ to denote the matrix obtained by  concatenating the matrices $\hat U_1, \ldots, \hat U_K$, that is
 \begin{equation} \label{hatU}
 \hat U := \begin{bmatrix}
 \hat U_1 & \hat U_2 & \cdots & \hat U_K
 \end{bmatrix} \in \real^{d \times KL}
 \end{equation}
 The matrix $\hat U$, which is nothing but a reshaped version of the original weight matrix $U \in \real^{K \times Ld}$, will play a crucial role in our analysis.

 \subsection{Basic properties of the Frobenius inner product}
 We recall that the Frobenius inner product between two matrices $A,B \in \real^{m\times n}$ is defined by
 $$
 \dotprodbis{A,B} = \sum_{i=1}^m \sum_{j=1}^n A_{ij} B_{ij}
 $$
 and that the Frobenius norm of a matrix $A\in \real^{m\times n}$ is given by 
 $
 \|A \|_F  = \sqrt{ \dotprodbis{A,A} }
 $.
 In the course of our proofs, we will constantly appeal to the following property of the Frobenius inner product, so we state it in a lemma once and for all.
 \begin{lemma} \label{lemma:magictrace} Suppose $A \in \real^{m \times n}$, $B \in \real^{m \times r}$ and $C \in \real^{r \times n}$. Then 
 \[
  \dotprodbis{A,BC}  =  \dotprodbis{B^TA,C}  \qquad \text{ and } \qquad  \dotprodbis{A,BC}  =  \dotprodbis{AC^T,B}
  \]
 \end{lemma}
 \begin{proof}
 The Frobenius inner product can be expressed as
 $
  \dotprodbis{A,B} = \tr(A^TB)
 $,
 and so we have
 \begin{align*}
  \dotprodbis{A,BC} & =   \tr(A^TBC)  =  \tr\left( \left( B^TA \right)^TC\right)  =  \dotprodbis{B^TA,C}. 
 \end{align*}
 Using the cyclic property of the trace, we also get
 \begin{align*}
  \dotprodbis{A,BC} & =   \tr(A^TBC) =  \tr(C A^TB) =  \tr\left( \left( A C^T \right)^TB\right)=  \dotprodbis{AC^T,B} 
 \end{align*}
\end{proof}

\subsection{The task,  the  data model, and the distribution $\mathcal D_{\bz_k}$}
Recall that  $\mathcal C = \{1, \ldots, n_c\}$ represents the set of concepts, and that $\mathcal Z = \mathcal C^L$ is the latent space. 
We aim to study a classification task in which the $K$ classes are defined by  $K$ latent variables
$$\bz_1 , \ldots, \bz_k \in \mathcal Z$$
We write $\bx \sim \mathcal D_{\bz_k}$ to indicate that the sentence $\bx \in \data$ is generated by the latent variable $\bz_k \in \mathcal Z$  (see figure 1 of the main paper for a visual illustration). Formally, $\mathcal D_{\bz_k}$ is a probability distribution on the data space $\data$, and we now give the formula for its p.d.f.
First, recall that $\mu_\beta>0$ stands for the probability of sampling the $\beta^{th}$ word of the $\alpha^{th}$ concept. 
Let us denote the  $k^{th}$ latent variable  by
$$
\bz_k = [\; z_{k,1} \; ,  \; z_{k,2} \; ,  \; \ldots \; ,  \;z_{k,L} \; ]  \in \mathcal Z
$$   
where  $1 \le z_{k,\ell}  \le n_c$.
The probability of sampling the sentence
$$
\bx = [ \; (\alpha_1,\beta_1) \; ,   \; (\alpha_2,\beta_2) \; \;\ldots \;, \; (\alpha_L,\beta_L) \;]  \in \mathcal X  %\qquad  1 \le \alpha_\ell \le n_c \text{ and }   1 \le \beta_\ell \le s_c
$$
according to $\mathcal D_{\bz_k}$ is then given by the formula
$$
\mathcal D_{\bz_k} \left( \{ \bx\} \right) =  \prod_{\ell=1}^L \;  \ones_{\{ \alpha_\ell = z_{k,\ell}\}} \; \mu_{\beta_\ell} %=  \begin{cases}
%\prod_{\ell=1}^L \mu_{\beta_\ell} & \text{ if }  [z_{k,1},  \ldots,  \;z_{k,L} ] =  [ \alpha_1,\ldots, \alpha_L] \\
 %0 & \text{otherwise}
%\end{cases}
$$
Note that $\mathcal D_{\bz_k} \left( \{ \bx\} \right)>0$ if and only if $[z_{k,1},  \ldots,  z_{k,L} ] =  [ \alpha_1,\ldots, \alpha_L]$. So a sentence $\bx$ has a non-zero probability of being  generated by the latent variable $\bz_k$ only if its words match the concepts in $\bz_k$. If this is the case, then the probability of sampling $\bx$ according to $\mathcal D_{\bz_k}$  is simply given by the product of the frequencies of the words contained in $\bx$. 

We use $\data_k$ to denote the support of the distribution $ \mathcal D_{\bz_k}$, that is
$$
\data_k := \{ \bx \in \data : \mathcal D_{\bz_k}(\bx) >0 \}
$$
and we note that if the latent variables $\bz_1, \ldots, \bz_K$ are mutually distinct, then
$
\data_j \cap \data_{k}  = \emptyset 
$
for all $j\neq k$.
Since the $K$ latent variables define the $K$ classes of our classification problem, we may alternatively define $\data_k$ by
$$
\data_k = \{ \bx \in \data :  \bx \text{ belongs to the $k^{th}$ category} \}
$$
%With this in mind we  use the notation
%$$
%\cl(\bx) = k
%$$
%to indicate that $\bx$ belongs to class $k$.
%Formally, $\cl$ is a function from $\cup_{k=1}^K \data_k$ to $\{1,\ldots,K\}$.

To each latent variable 
$
\bz_k = [\; z_{k,1} \; ,  \; z_{k,2} \; ,  \; \ldots \; ,  \;z_{k,L} \; ] 
$ we associate
a matrix
\begin{gather} \label{Z_k}
Z_k = 
 \begin{bmatrix}
 \vert & \vert &  & \vert \\
 \be_{z_{k,1}} &
\be_{z_{k,2}} &
\cdots &
 \be_{z_{k,L}} \\
 \vert & \vert &  & \vert 
\end{bmatrix}   \in \real^{n_c \times L}  %\qquad  \text{and} \qquad Z = \begin{bmatrix} Z_1 & Z_2 & \cdots & Z_R \end{bmatrix} \in \real^{n_c \times KL}
\end{gather}
In other words, the matrix $Z_k$ provides a one-hot representation of the concepts contained in the latent variable $\bz_k$. Concatenating the matrices $Z_1, \ldots, Z_K$ gives the matrix 
\begin{equation} \label{Z}
 Z = \begin{bmatrix} Z_1 & Z_2 & \cdots & Z_K \end{bmatrix} \in \real^{n_c \times KL}
\end{equation}
which is reminiscent of  the matrix $\hat U$ defined by \eqref{hatU}.

We encode the way words are partitioned into concepts into a `partition matrix' $P \in \real^{n_c \times n_w}$.
For example, if we have $12$ words and $3$ concepts,  then the partition matrix  is 
\begin{equation} \label{P}
P = 
\begin{bmatrix}
1 & 1 & 1 & 1 & 0 & 0 & 0 & 0 & 0 & 0 & 0 & 0    \\
 0 & 0 & 0 & 0   & 1 & 1 & 1 & 1 & 0 & 0 & 0 & 0  \\
 0 & 0 & 0 & 0  & 0 & 0 & 0 & 0 & 1 & 1 & 1 & 1   
\end{bmatrix} \in \real^{ n_c \times n_w} ,
\end{equation}
indicating that the first 4 words belong to concept 1, the next 4 words belongs to concept 2, and so forth.
Formally,  recalling that  $\hot(\alpha,\beta)$ is the the one-hot encoding of word $(\alpha,\beta) \in \voc$, the matrix $P$ is defined 
 the relationship
\begin{equation}
P \;  \hot(\alpha,\beta) = \be_\alpha \qquad \text{ for all } (\alpha,\beta) \in \voc.
\end{equation}
Importantly, note that the matrix $P$ maps datapoints to their associated latent variables. Indeed, if $\bx=  [ (\alpha_1,\beta_1),  \ldots , (\alpha_L,\beta_L)] $ is generated by the latent variable $\bz_k$ (meaning that $\bx \in \data_k$), then we have that
\begin{equation} \label{Phot}
P\, \hot(\bx) = P
\begin{bmatrix}
 \vert & \vert & & \vert \\
  \hot(\alpha_1,\beta_1) & \hot(\alpha_2,\beta_2) & \ldots & \hot(\alpha_L,\beta_L)  \\
   \vert & \vert & & \vert 
  \end{bmatrix} 
  = 
  \begin{bmatrix}
 \vert & \vert & & \vert \\
  \be_{\alpha_1} & \be_{\alpha_2} & \ldots & \be_{\alpha_L}  \\
   \vert & \vert & & \vert 
  \end{bmatrix} =
 Z_k 
\end{equation}
where the last equality is due to  definition \eqref{Z_k} of the matrix $Z_k$.

Another important matrix for our analysis will be the matrix $Q \in \real^{n_c \times n_w}$. In the concrete case where we have $12$ words and $3$ concepts, this matrix takes the form
\begin{equation} \label{matrixQ}
Q = 
\begin{bmatrix}
\mu_1 & \mu_2 & \mu_3 & \mu_4 & 0 & 0 & 0 & 0 & 0 & 0 & 0 & 0    \\
 0 & 0 & 0 & 0   &\mu_1 & \mu_2 & \mu_3 & \mu_4 & 0 & 0 & 0 & 0  \\
 0 & 0 & 0 & 0  & 0 & 0 & 0 & 0 & \mu_1 & \mu_2 & \mu_3 & \mu_4   
\end{bmatrix} \in \real^{ n_c \times n_w} 
\end{equation}
and, in general, it is defined by the relationship
\begin{equation} \label{Qhot}
Q \;  \hot(\alpha,\beta) = \mu_\beta \, \be_\alpha \qquad \text{ for all } (\alpha,\beta) \in \voc.
\end{equation}

\section{Symmetry assumptions on the latent variables} \label{section:sym}

In subsection \ref{sub:B1} we provide an in depth discussion of the symmetry assumption  required for theorems 1 and 3 to hold. 
In subsection \ref{sub:B2} we present and discuss the assumption that will be needed to prove the \emph{converse} of theorems 1 and 3.

\subsection{Symmetry assumption needed for theorem \ref{theorem:1} and \ref{theorem:3}} \label{sub:B1}
To better understand  the symmetry assumption 1 from the main paper, let us start by considering the extreme case
\begin{equation} \label{alllatent}
K = n_c^L \qquad \text{and} \qquad \{\bz_1, \bz_2,  \ldots, \bz_K \} = \mathcal Z,
\end{equation}
meaning that $\bz_1, \ldots, \bz_K$  are mutually distinct and represent all possible latent variables in $\mathcal Z$.  In this case, we easily obtain the formula
\begin{equation} \label{babycount}
\left|\Big\{ j \in[K]   : \dist(\bz_j,\bz_{1}) = r  \text{ and } z_{j,L} = z_{1,L} \Big\}\right| =  { L-1 \choose r} (n_c-1)^r 
\end{equation}
where $\dist(\bz_j,\bz_1)$ is the Hamming distance between the latent variables $\bz_j$ and $\bz_1$.
To see this, note that the left side of \eqref{babycount}
counts the number of latent variables $\bz_j$ that \emph{differs} from $\bz_1$ at $r$ locations and \emph{agrees} with $\bz_1$ at the last location $\ell=L$.  This number is clearly equal to 
the right side of \eqref{babycount} since
 we need to choose $r$ positions out of the first $L-1$ positions, and then, for each chosen position $\ell$,  we  need to choose a concept out of the $n_c-1$ concepts that differs  from $\bz_{1,\ell}$. A similar reasoning shows that, if 
 $z_{1,L} \neq \alpha$, then
 \begin{equation} \label{babycount2}
\left|\Big\{ j \in[K]   : \dist(\bz_j,\bz_{1}) = r  \text{ and } z_{j,L} = \alpha \Big\}\right| =  { L-1 \choose r-1} (n_c-1)^{r-1} 
\end{equation}
where the term ${ L-1 \choose r-1}$ arises from the fact that we only need to choose $r-1$ positions, since $\bz_1$ and $\bz_j$ differ in their last position $\ell=L$.
Suppose now that the random variables $\bz_1, \ldots, \bz_K$ are selected uniformly at random from $\mathcal Z$, and say, for the sake of concreteness, that 
$$
K =    \frac{1}{5} \; n_c^L 
$$   
so that $\bz_1, \ldots, \bz_K$ represent $20\%$ of all possible latent variables (note that  $|\mathcal Z| = n_c^L$). Then  \eqref{babycount} -- \eqref{babycount2} should be replaced by
\begin{align} \label{babycount3}
&\left|\Big\{ j \in[K]   : \dist(\bz_j,\bz_{1}) = r  \text{ and } z_{j,L} = z_{1,L} \Big\}\right| \approx \frac{1}{5}  { L-1 \choose r} (n_c-1)^r  \\
&\left|\Big\{ j \in[K]   : \dist(\bz_j,\bz_{1}) = r  \text{ and } z_{j,L} = \alpha \Big\}\right| \approx \frac{1}{5}  { L-1 \choose r-1} (n_c-1)^{r-1} \quad \text{for } \alpha \neq z_{1,L}  \label{babycount4}
\end{align}
where the equality only holds approximatively due to the random choice of the latent variables. 
In the above example, we chose $\bz_1$ as our `reference' latent variables and we `froze' the concept appearing in position $\ell=L$. These choices were clearly arbitrary.
In general, when $K$ is large, we have 
\begin{equation} \label{sym_approx}
 \left| \Big\{ j \in[K]   : \dist(\bz_j,\bz_{k}) = r \text{ and }  z_{j,\ell} = \alpha \Big\} \right| \approx  \begin{cases}
\frac{K}{n_c^L}  { L-1 \choose r} (n_c-1)^r  & \text{ if } z_{k,\ell} = \alpha \\ \\
\frac{K}{n_c^L}  { L-1 \choose r-1} (n_c-1)^{r-1} & \text{ if } z_{k,\ell} \neq \alpha 
  \end{cases}
\end{equation}
and this approximate equality  hold for most $k\in[K]$, $r \in [L]$, $\ell \in[L]$, and $\alpha \in [n_c]$.
The symmetry assumption \ref{assumption:symmetry} from the main paper requires \eqref{sym_approx} to hold not approximatively, but exactly. For convenience we restate below this symmetry assumption:
\begin{assumption}[Latent Symmetry] \label{symstrong}
 For every $k \in [K]$, $r \in [L]$, $\ell \in[L]$, and $\alpha \in [n_c]$ the identities
\begin{equation} \label{sym}
 \left| \Big\{ j \in[K]   : \dist(\bz_j,\bz_{k}) = r \text{ and }  z_{j,\ell} = \alpha \Big\} \right| =  \begin{cases}
\frac{K}{n_c^L}  { L-1 \choose r} (n_c-1)^r  & \text{ if } z_{k,\ell} = \alpha \\ \\
\frac{K}{n_c^L}  { L-1 \choose r-1} (n_c-1)^{r-1} & \text{ if } z_{k,\ell} \neq \alpha 
  \end{cases}
\end{equation}
hold.
\end{assumption}
To be clear, if the latent variables  $\bz_1, \ldots, \bz_K$ are selected uniformly at random from $\mathcal Z$, then  they will only    \emph{approximatively}  satisfy assumption \ref{symstrong}.  Our analysis, however,  is conducted in  the idealized  case where the latent variables \emph{exactly} satisfy the symmetry assumption.  Specifically, we show that, in the idealized case where 
assumption \ref{symstrong} is \emph{exactly} satisfied, then the weights $W$ and $U$ of the network are  given by some explicit analytical formula. Importantly, as it is explained in the main paper, our experiments demonstrate that these idealized analytical formula 
provide very good approximations for the weights observed in experiments when the latent variables are selected uniformly at random.

In the next lemma, we isolate  three  properties which hold for any latent variables satisfying  assumption \ref{symstrong}. 
Importantly, when proving collapse, we will only rely on these three properties --- we will never explicitly need assumption \ref{symstrong}.
We will see shortly that these three properties, in essence, amount to saying that all position $\ell \in [L]$ and all concepts $\alpha \in [n_c]$ plays interchangeable roles for the latent variables. There are no `preferred' $\ell$ or $\alpha$, and this is exactly what will allow us to derive symmetric analytical solutions.

Before stating our lemma, let us define the `sphere' of radius $r$ centered around the $k^{th}$ latent variable
\begin{equation}
S_r(k)  := \Big\{ j \in[K]   : \dist(\bz_j,\bz_{k}) = r  \Big\} \qquad \text{ for } r,k \in [L]
\end{equation}
With this notation in hand we may now state 
\begin{lemma} \label{lemma:sym} Suppose the latent variables $\bz_1, \ldots, \bz_K$ satisfy the symmetry assumption \ref{symstrong}.  Then $\bz_1, \ldots, \bz_K$ satisfies the following properties:
\begin{enumerate}
\item[(i)]  $|S_r(j)| = |S_r(k)|$ for all $r \in[L]$ and all $j,k \in [K]$.
\item[(ii)] The equalities  $$ \sum_{k=1}^K Z_k = \frac{K}{n_c} \ones_{n_c} \ones^T_L \qquad \text{and} \qquad  \displaystyle ZZ^T = \frac{KL}{n_c} I_{n_c}$$
hold, with $I_{n_c}$ denoting the $n_c \times n_c$ identity matrix.
\item[(iii)] There exists $\theta_1, \ldots, \theta_L > 0$ and matrices $A_1, \ldots, A_L \in \real^{n_c \times L}$ such that
$$
Z_k - \frac{1}{|S_r(k)|} \sum_{j \in S_r(k)}  Z_{j}  =    \theta_r Z_k  + A_r
$$
holds for all $r \in[L]$,  all $j \in [K]$, and all $k \in [K]$.
\end{enumerate}
\end{lemma}
We will prove this lemma shortly, but for now let us start by getting some intuition about properties (i), (ii) and (iii). Property  (i) is transparent:  it  states that all latent variables have the same number of `distance-$r$ neighbors'. Recalling how matrix $Z_k$  was defined (c.f. \eqref{Z_k}), we see that the first identity of  (ii)   is equivalent to
\begin{equation} \label{equal_num_cpt}
\left| \left\{ k \in [K] : z_{k,\ell} = \alpha \right\} \right| = \frac{K}{n_c} \qquad \text{for all $\ell \in[L]$ and all $\alpha \in [n_c]$}.
\end{equation}
This means  that the number of latent variables that have concept $\alpha$ in position $\ell$ is equal to $K/n_c$. In other words, 
 each concept is equally represented at each position $\ell$. We now  turn to the second identity of statement (ii).  Recalling the definition \eqref{Z} of matrix $Z$, we see that $ZZ^T \in \real^{n_c \times n_c}$  is a diagonal matrix since each column of $Z$ contains a single nonzero entry. One can also easily see that  the $\alpha^{th}$ entry of the diagonal is 
 $$
 \left[ZZ^T\right]_{\alpha,\alpha} = |\{ (k,\ell) \in [K] \times [L] : z_{k,\ell}=\alpha \}|, 
 $$
 which is the total number of times concept $\alpha$ appears in the latent variables. Overall, the identity  $ZZ^T = \frac{KL}{n_c} I_{n_c}$ is therefore equivalent to the statement
 \begin{equation*}
  |\{ (k,\ell) \in [K] \times [L] : z_{k,\ell}=\alpha \}  |=  \frac{KL}{n_c}  \qquad \text{for all $\alpha \in [n_c]$} % \label{equal_num_cpt2}
 \end{equation*}
and it is therefore a direct consequence of \eqref{equal_num_cpt}.

Property (iii) is harder to interpret. Essentially it is a type of mean value property that states that summing over the latent variables  which are at distance $r$ of $\bz_k$ gives back $\bz_k$. We will see that this mean value property plays a key role in our analysis.

To conclude this subsection, we prove lemma \ref{lemma:sym}.
\begin{proof}[Proof of lemma \ref{lemma:sym}] We start by proving statement (i). Since  $S_r(k)  = \{ j \in[K]   : \dist(\bz_j,\bz_{k}) = r\}$, we clearly have that 
\begin{align}
|S_r(k)|  &= \sum_{\alpha=1}^{n_c} \left| \Big\{ j \in[K]   : \dist(\bz_j,\bz_{k}) = r \text{ and }  z_{j,\ell} = \alpha \Big\} \right| 
\end{align}
We then use identity \eqref{sym} and Pascal's rule to find
\begin{align}
 |S_r(k)|  & = \left(n_c-1\right)  \left( \frac{K}{|\mathcal Z|}  { L-1 \choose r-1} (n_c-1)^{r-1}\right) \nonumber
+ \frac{K}{|\mathcal Z|}  { L-1 \choose r} (n_c-1)^r  \nonumber \\
& = \frac{K}{|\mathcal Z|} (n_c-1)^r   \left(  {L-1 \choose r-1}  +   { L-1 \choose r}  \right) \nonumber \\
& = \frac{K}{|\mathcal Z|}  {L \choose r}  (n_c-1)^r   \label{size_of_SSS}
\end{align}
which clearly implies that  $|S_r(k)| = |S_r(j)|$ for all $j,k \in [K]$ and all $r \in [L]$.

We now turn to the first identity of t (ii). As previously mentioned, this identity is equivalent to \eqref{equal_num_cpt}.  Choose $k$  such that  $\bz_{k,\ell} \neq \alpha$. 
Then any any latent variable $\bz_j$ with 
$z_{j,\ell} = \alpha$ is at least at a distance $1$ of $\bz_k$ and we may write
\begin{align}
\left| \left\{ j  \in [K]: z_{j,\ell} = \alpha \right\} \right| & = 
\sum_{r=1}^L  \left| \Big\{ j \in[K]   : \dist(\bz_j,\bz_{k}) = r \text{ and }  z_{j,\ell} = \alpha \Big\} \right| \\
& = \sum_{r=1}^L \frac{K}{n_c^L}  { L-1 \choose r-1} (n_c-1)^{r-1} 
\end{align}
which is equal to $K/n_c$ according to the binomial theorem. The second identity of  (ii), as mentioned earlier, is a direct consequence of the first identity.

We finally turn to statement (iii).   Appealing to \eqref{size_of_SSS}, we find that,  
\begin{align*}
 \frac{ \left| \Big\{ j \in[K]   : \dist(\bz_j,\bz_{k}) = r \text{ and }  z_{j,\ell} = \alpha \Big\} \right| }{|S_r(k)| }  = \frac{\frac{K}{|\mathcal Z|}  { L-1 \choose r} (n_c-1)^r}{\frac{K}{|\mathcal Z|}  { L \choose r} (n_c-1)^r} = \frac{{L-1 \choose r} }{{L \choose r}} 
 % = \frac{(L-1)!}{r! (L-1-r)!} \; \frac{r! (L-r)!}{L!} = \frac{L-r}{L} 
 = \frac{L-r}{L}  
\end{align*}
 if $z_{k,\ell} = \alpha$. 
On the other hand, if  $z_{k,\ell} \neq \alpha$, we obtain 
\begin{align*}
 \frac{ \left| \Big\{ j \in[K]   : \dist(\bz_j,\bz_{k}) = r \text{ and }  z_{j,\ell} = \alpha \Big\} \right|  }{|S_r(k)| }  = \frac{ \frac{K}{|\mathcal Z|}   { L-1 \choose r-1} (n_c-1)^{r-1}}{\frac{K}{|\mathcal Z|}  { L \choose r} (n_c-1)^r} = \frac{1}{n_c-1} \frac{{L-1 \choose r-1} }{{L \choose r}}
 % = \frac{(L-1)!}{r! (L-1-r)!} \; \frac{r! (L-r)!}{L!} = \frac{L-r}{L} 
 =  \frac{1}{n_c-1}  \frac{r}{L}  
\end{align*}
 Fix $\ell \in [L]$ and assume that $z_{k,\ell}=\alpha^\star $. We then have
\begin{align*}
\frac{1}{|S_r(k)|} \sum_{j \in S_r(k)}  \be_{z_{j,\ell}}  & = \frac{1}{|S_r(k)|} \sum_{\alpha =1}^{n_c}  \left| \Big\{j \in S_r(k) : \bz_{j,\ell} = \be_\alpha \Big\}  \right|  \; \be_\alpha\\
& =  \sum_{\alpha=1}^{n_c}  \frac{ \left| \Big\{ j \in[K]   : \dist(\bz_j,\bz_{k}) = r \text{ and }  z_{j,\ell} = \alpha \Big\} \right|  }{|S_r(k)| }  \;\; \be_\alpha\\
& = \frac{L-r}{L} \, \be_{\alpha^\star}+   \frac{1}{n_c-1}  \frac{r}{L}    \sum_{\alpha \neq \alpha^\star} \be_\alpha \\
& = \frac{L-r}{L} \, \be_{\alpha^\star} -  \frac{1}{n_c-1}  \frac{r}{L}  \; \be_{\alpha^\star} +   \frac{1}{n_c-1}  \frac{r}{L}    \sum_{\alpha=1}^{n_c} \be_\alpha \\
& = \left(1 - \frac{n_c}{n_c-1}  \frac{r}{L}  \right) \; \be_{\alpha^\star} +   \frac{1}{n_c-1}  \frac{r}{L}  \ones_{n_c}
\end{align*}
Recalling that $z_{k,\ell}=\alpha^\star$, the above implies that
\begin{align} \label{vlam}
\be_{z_{k,\ell}} - \frac{1}{|S_r(k)|} \sum_{j \in S_r(k)}  \be_{z_{j,\ell}} =   \frac{n_c}{n_c-1}  \frac{r}{L} \; \be_{z_{k,\ell}} -   \frac{1}{n_c-1}  \frac{r}{L}  \ones_{n_c}
\end{align}
Finally, recalling that 
\begin{gather*}
Z_k = 
 \begin{bmatrix}
 \vert & \vert &  & \vert \\
 \be_{z_{k,1}} &
\be_{z_{k,2}} &
\cdots &
 \be_{z_{k,L}} \\
 \vert & \vert &  & \vert 
\end{bmatrix}   \in \real^{n_c \times L}  %\qquad  \text{and} \qquad Z = \begin{bmatrix} Z_1 & Z_2 & \cdots & Z_R \end{bmatrix} \in \real^{n_c \times KL}
\end{gather*}
we see that \eqref{vlam} can be written in matrix format as
$$
Z_k - \frac{1}{|S_r(k)|} \sum_{j \in S_r(k)}  Z_{j}  =     \frac{n_c}{n_c-1}  \frac{r}{L}  Z_k    -   \frac{1}{n_c-1}  \frac{r}{L}  \ones_{n_c} \ones^T_L
$$
and therefore the scalars $\theta_r$ and the matrices $A_r$ appearing in statement (iii) are given by the formula
$\theta_r =   \frac{n_c}{n_c-1}  \frac{r}{L}  $ and $A_r= -   \frac{1}{n_c-1}  \frac{r}{L}  \ones_{n_c} \ones^T_L$.
\end{proof}

\subsection{Symmetry assumption needed for the converse of theorem \ref{theorem:1} and \ref{theorem:3}} \label{sub:B2}
In this subsection we present the symmetry assumption that will be needed to prove the converse of theorem \ref{theorem:1} and \ref{theorem:3}.  This assumption, as we will shortly see, is quite mild and is typically satisfied even for small values of $K$.

For each pair of latent variables $(\bz_j,\bz_k)$ we define the matrix
$$
\Gamma^{(j,k)} :=  Z_j(Z_j-Z_{k})^T \in \real^{n_c \times n_c}.
$$
We also define 
\begin{equation} \label{calA}
\mathcal A := \Big\{ A \in \real^{n_c \times n_c} :  \text{There exists $a,b \in \real$ s.t. $A = a I_{n_c} + b \ones_{n_c} \ones^T_{n_c} $} \Big\}
% \{
%a I_{n_c} + b \ones_{n_c} \ones^T_{n_c}  : a, b\in \real \}
\end{equation}
which is  the set of  matrices whose diagonal entries are equal to some constant  and whose off-diagonal entries are equal to some possibly different constant.
We may now state our symmetry assumption.
\begin{assumption}  \label{symconv}Any positive semi-definite matrix $A \in \real^{n_c \times n_c}$ that satisfies
\begin{equation} \label{system_sym}
\dotprodbis{ \; A \; , \;   \Gamma^{(j,k)} -  \Gamma^{(j',k')}  } = 0 \qquad   \forall j,k,j',k' \in [K]  \text{ s.t. } \dist(\bz_{j},\bz_{k}) = \dist(\bz_{j'},\bz_{k'})
\end{equation}
must belongs to $\mathcal A$.
\end{assumption}
Note that \eqref{system_sym} can be viewed as a linear system of equations for the unknown  $A \in \real^{n_c \times n_c}$,  with one equation for each quadruplet  $(j,k,j',k')$  satisfying  $\dist(\bz_{j},\bz_{k}) = \dist(\bz_{j'},\bz_{k'})$. To put it differently, each quadruplet  $(j,k,j',k')$  satisfying  $\dist(\bz_{j},\bz_{k}) = \dist(\bz_{j'},\bz_{k'})$ adds one equation to the system, and our assumption requires that we have enough of these equations so that all  positive semi-definite solutions are constrained  to live in the set $\mathcal A$. Since a symmetric matrix has $(n_c+1)n_c/2$ distinct entries, we would expect that $(n_c+1)n_c/2$ quadruplets should be enough to fully determine the matrix. This number of quadruplets is easily achieved even for small values of $K$. So  assumption \ref{symconv} is  quite mild.

 The next lemma  states that assumption \ref{symconv} is satisfied when $K=n_c^L$. In light of the above discussion this is not surprising, since the choice $K=n_c^L$ leads to a system with a number of equations much larger than
 $(n_c+1) n_c/2$. 
  The proof, however, is instructive: it simply handpicks $(n_c+1) n_c/2 - 2$
 quadruplets  to determine  the entries of the matrix $A$. The  `$-2$' arises from the fact   $\mathcal A$ is a 2 dimensional subspace, and therefore $(n_c+1) n_c/2 - 2$ equations are `enough' to constrain $A$ to be in $\mathcal A$.
 \begin{lemma}
 Suppose $K=n_c^L$ and $\{\bz_1, \ldots,\bz_K\} = \mathcal Z$.
Then $\bz_1, \ldots,\bz_K$ satisfy the symmetry assumption \ref{symconv}.
\end{lemma}
\begin{proof}
Let $A = C^T C$ be a positive semi-definite matrix that solve satisfies \eqref{system_sym}. We use $\bc_\alpha$ to denote the $\alpha^{th}$  column of $C$.
Since $\{\bz_1, \ldots,\bz_K\} = \mathcal Z$,  we can find $i,j,k \in [K]$ such that 
  \begin{align*}
  \bz_i & = [2 , 1 , 1,  \ldots, 1] \in \mathcal Z \\
   \bz_j & = [3 , 1, 1, \ldots, 1]  \in \mathcal Z\\
    \bz_k & = [4 , 1 , 1,  \ldots, 1]  \in \mathcal Z
  \end{align*}
  Using lemma \ref{lemma:magictrace} and recalling the definition \eqref{Z_k} of the matrix $Z_k$, we  get
  \begin{align*}
  \dotprodbis{ \; A \; , \;   \Gamma^{(i,j)}   }  &=
  \dotprodbis{C^T C ,  Z_i(Z_i -Z_j)^T}  \\
  &=     \dotprodbis{C(Z_i-Z_j), CZ_i}    \\ 
  &= \dotprodbis{CZ_i, CZ_i} - \dotprodbis{CZ_j, CZ_i} \\
  & =  \Big(  \dotprod{\bc_{2} , \bc_{2} } +  (L-1) \dotprod{\bc_{1} , \bc_{1} }  \Big)    - \Big( \dotprod{\bc_{2} , \bc_{3} } +   (L-1) \dotprod{\bc_{1} , \bc_{1} } \Big) \\
  & =  \dotprod{\bc_{2} , \bc_{2} }  -  \dotprod{\bc_{2} , \bc_{3} } 
  \end{align*}
  Similarly we obtain that 
  $$
   \dotprodbis{ \; A \; , \;   \Gamma^{(i,k)}   }  =   \dotprod{\bc_{2} , \bc_{2} }  -  \dotprod{\bc_{2} , \bc_{4} } 
  $$
 Since $\dist(\bz_i,\bz_j) = \dist(\bz_i,\bz_k)=1$, and since $A$ satisfies  \eqref{system_sym}, we must have 
  $$
   \dotprodbis{ \; A \; , \;   \Gamma^{(i,j)}   } =  \dotprodbis{ \; A \; , \;   \Gamma^{(i,k)}   } 
  $$
which in turn implies that
  $$
 A_{2,3} =  \dotprod{\bc_{2} , \bc_{3} }  =  \dotprod{\bc_{2} , \bc_{4} }  = A_{2,4}
  $$
This argument easily generalizes to show that all off-diagonal entries of the matrix $A$ must be equal to some constant $b \in \real$.

  We now take care of the diagonal entries. Since $\{\bz_1, \ldots,\bz_K\} = \mathcal Z$,  we can find $i',j',k' \in [K]$ such that 
   \begin{align*}
  \bz_{i'} & = [1 , 1, \ldots , 1] \in \mathcal Z  \\
  \bz_{j'} & = [2, 2, \ldots, 2] \in \mathcal Z  \\
   \bz_{k'}& = [3 , 3, \ldots, 3]  \in \mathcal Z 
   \end{align*}
 As before, we compute
 \begin{align*}
  \dotprodbis{ \; A \; , \;   \Gamma^{(i',j')}   }  
  = \dotprodbis{CZ_{i'}, CZ_{i'}} - \dotprodbis{CZ_{j'}, CZ_{i'}} 
   = L  \dotprod{\bc_{1} , \bc_{1} }      -    L \dotprod{\bc_{1} , \bc_{2} }
   =  L  \dotprod{\bc_{1} , \bc_{1} }      -    L b
  \end{align*}
  where we have used the fact that the off diagonal entries are all equal to $b$. Similarly we obtain 
 \begin{align*}
  \dotprodbis{ \; A \; , \;   \Gamma^{(j',k')}   }  
  & =  L  \dotprod{\bc_{2} , \bc_{2} }      -    L b
  \end{align*}
 Since $\dist(\bz_{i'} , \bz_{j'}) = \dist( \bz_j' , \bz_{k'}) = L$, we must have  
  $
   \dotprodbis{ A  ,   \Gamma^{(i',j')}   } =  \dotprodbis{  A ,    \Gamma^{(j',k')}   } 
  $
   which implies that $A_{1,1} = A_{2,2}$. This argument generalizes to show that all diagonal entries of $A$ are equal.
\end{proof}

\section{Sharp lower bound on the unregularized risk} \label{section:lowerbound}

In this section we derive a sharp lower bound for the \emph{unregularized} risk associated with the network $h_{W,U}$,
  \begin{align} \label{calR0} 
 \mathcal R_0(W,U) := \frac{1}{K}  \sum_{k=1}^{K}   \E_{  \;\;\bx \sim \mathcal D_{\bz_k}} \Big[  \ell(h_{W,U}(\bx) , k ) \Big],  
\end{align}
where  $\ell: \real^K \to \real$ is the  cross entropy loss
\[
\ell(\by, k) = -   \log \left(  \frac{\exp\left( y_k\right)}{\sum_{j=1}^K \exp\left( y_{j}\right)}\right) \qquad \text{ for } \by \in \real^K
\]
The $k^{th}$ entry of the output $\by = h_{W,U}(\bx)$ of the neural network, according to formula \eqref{def:net}, is given by
\[
y_k =  \dotprodbig{ \; \hat U_k \; , \;  W \, \hot(\bx)} 
\]
Recalling that $\data_k$ is the support of the distribution 
$\mathcal D_{\bz_k}: \data \to [0,1]$,
 we find that the unregularized risk can be expressed as
\begin{align*} 
 \mathcal R_0(W,U) 
 %& =
% \frac{1}{K}  \sum_{k=1}^{K}  \E_{  \;\;\bx \sim \mathcal D_{\bz_k}} \Big[  \ell(h_{W,U}(\bx) , k ) \Big]  \\
%& =    \sum _{\bx \in \data}  \ell(h_{W,U}(\bx) , k ) \;   \mathcal D_{\bz_k}(\bx) \\
& =    \frac{1}{K}  \sum_{k=1}^{K}  \sum _{\bx \in \data_k}  \ell(h_{W,U}(\bx) , k ) \;   \mathcal D_{\bz_k}(\bx) \\
& =   \frac{1}{K}  \sum_{k=1}^{K}  \sum _{\bx \in \data_k}  -\log\left(  \frac{e^{ \dotprodbis{\hat U_k, W \hot(\bx)}}}{ \sum_{j=1}^K e^{ \dotprodbis{\hat U_{j}, W  \hot(\bx) }}  }\right)    \;   \mathcal D_{\bz_k}(\bx) \\
%& =     \sum _{\bx \in \data_k}  \log\left(  \sum_{j=1}^K e^{ - \dotprodbis{\hat U_k - \hat U_{j} , W  \hot(\bx) }}  \right)    \;   \mathcal D_{\bz_k}(\bx) \\
& =   \frac{1}{K}  \sum_{k=1}^{K}   \sum _{\bx \in \data_k}  \log\left( 1 +  \sum_{j \neq k}   e^{ - \dotprodbis{\hat U_k - \hat U_{j} , W  \hot(\bx) }}  \right)    \;   \mathcal D_{\bz_k}(\bx) 
%& =   \sum _{\bx \in \data_k}  \log\left( 1 +  \sum_{k' \neq k}   e^{ - \margin_{W,U}(\bx , k')  } \right)  \;   \mathcal D_{\bz_k}(\bx)
\end{align*}
where we did the slight abuse of notation of writing  $\mathcal D_{\bz_k}(\bx)$ instead of  $\mathcal D_{\bz_k}( \{ \bx \})$. Note that a data points $\bx$ that belongs to class $k$ is correctly classified by the the network $h_{W,U}$ if and only if
$$
\dotprodbig{ \; \hat U_k \; , \;  W \, \hot(\bx)}  > \dotprodbig{ \; \hat U_j \; , \;  W \, \hot(\bx)}  \quad \text{for all } j \neq k
$$
With this in mind, we introduce the following definition:
\begin{definition}[Margin] Suppose $\bx \in \data_k$. Then the margin between data point $\bx$ and class $j$ is 
$$
\margin_{W,U}(\bx,j):=  \dotprodbis{\hat U_{k} - \hat U_j, W  \hot(\bx)}
$$
\end{definition}
With this definition in hand, the unregularized risk can conveniently be expressed as 
\begin{equation}  \label{surf}
 \mathcal R_0(W,U)  =  \frac{1}{K}  \sum_{k=1}^{K}   \sum _{\bx \in \data_k}  \log\left( 1 +  \sum_{j \neq k}   e^{ - \margin_{W,U}(\bx , j)  } \right)  \;   \mathcal D_{\bz_k}(\bx)
\end{equation}
and  a data point $\bx \in \data_k$ is correctly classified by the network if and only if the margins $\margin_{W,U}(\bx,j)$ are all strictly positive (for $j\neq k$). We then introduce a definition that will play crucial role in our analysis.
\begin{definition}[Equimargin Property] 
If
$$
 \dist(\bz_{k},\bz_{j}) = \dist(\bz_{k'},\bz_{j'})  \quad  \Longrightarrow    \quad  \margin_{W,U}(\bx,j) = \margin_{W,U}(\bx',j') \quad   \forall \bx \in \data_k \text{ and }  \forall \bx' \in \data_{k'} 
$$
then we say that $(W,U)$  satisfies the equimargin property. 
\end{definition}
To put it simply, $(W,U)$ satisfies the equimargin property if the margin between data point $\bx \in \data_k$ and class $j$ only depends on $\dist(\bz_{k},\bz_{j})$. 
We denote by $\mathcal E$ the set of all the weights that satisfy the equimargin property
\begin{align}
\mathcal E & = \left\{ (W,U) : (W,U) \text{ satisfies the equimargin property} \right\} 
\end{align}
and by $\mathcal N$ the set of weights for which the submatrices $\hat U_k$ defined by \eqref{U_k} sum to $0$,
\begin{align}
\mathcal N & = \left\{ (W,U) :  \sum_{k=1}^K \hat U_k = 0\right\} 
\end{align}
We will work under the assumption that the latent variables $\bz_1, \ldots, \bz_K$ satisfy the symmetry assumption \ref{symstrong}. According to lemma \ref{lemma:sym}, $|S_r(k)|$ then doesn't depend on $k$, and so we will simply use $|S_r|$ to denote the size of the set $S_r(k)$.
 Lemma \ref{lemma:sym} also states that
$$
Z_k - \frac{1}{|S_r(k)|} \sum_{j \in S_r(k)}  Z_{j}  =    \theta_r Z_k  + A_r
$$
for some matrices $A_1, \ldots, A_L$ and some scalars $\theta_1, \ldots, \theta_L>0$. We use these scalars to define 
\begin{equation}
g(x) :=  \log\left(  \;  1 +  \sum_{r=1}^L   |S_r|   \;  e^{\theta_r x / K}   \;  \right)  \quad \text{}
\end{equation}
and we note that $g: \real \to \real$ is a strictly increasing function. With these definitions in hand we may state the main theorem of this section.
\begin{theorem} \label{theorem:buffalo} 
If the latent variables satisfy the symmetry  assumption \ref{symstrong}, then 
\begin{align}
\mathcal R_0(W,U) &=  g \Big( -   \dotprodbis{\hat U,  W Q^TZ }  \Big) \qquad \text{ for all } (W,U) \in \mathcal N \cap \mathcal E \\
\mathcal R_0(W,U) &>  g \Big( -   \dotprodbis{\hat U,  W Q^TZ }  \Big) \qquad \text{ for all } (W,U) \in \mathcal N \cap \mathcal E^c
\end{align}
 \end{theorem}
 We recall that the matrices $\hat U$, $Q$, and $Z$ where defined in section \ref{section:notation} (c.f. \eqref{hatU}, \eqref{matrixQ} and \eqref{Z}).
The remainder of this section is devoted to the proof of the above theorem.

 %%%%%%%%%%%%%%%%%%%%%%%%%%
 %%%%%%%%%%%%%%%%%%%%%%%%%%%
 % !!!!!!!!!!!!!!!!!!!!!!!!!!!!!!!!!!!!!! ********************** !!!!!!!!!!!!!!!!!!!!!!!!
 %%%%%%%%%%%%%%%%%%%%%%%%
 %%%%%%%%%%%%%%%%%%%%%%%%
 
  \subsection{Proof of the theorem}
We will use two lemmas to prove the theorem. The first one (lemma \ref{lemma:buffalo1} below) simply leverages the strict convexity of the various components defining the unregularized risk $\mathcal R_0$. Recall that if $f: \real^d \to \real$ is strictly convex, and if the strictly positive scalars $p_1, \ldots, p_n>0$  sum to $1$, then
\begin{equation}
f\left(\sum_{i=1}^n p_i \bv_i \right) \le  \sum_{i=1}^n p_i f(\bv_i)
\end{equation}
and that equality holds if and only if $\bv_1 = \bv_2 = \ldots = \bv_n$. For this first lemma, the only property we need on the latent variables is that $|S_{r}(k)| = |S_r(j)| = |S_r|$ for all $j,k \in [K]$ and all $r \in [L]$.

Define the quantity
\begin{equation} \label{average_margin}
  \mathfrak N_{W,U}(r) =   \frac{1}{K}  \frac{1}{|S_r|}    \sum_{k=1}^{K}   \sum_{j \in S_r(k)} \sum _{\bx \in \data_k}  \margin_{W,U}(\bx , j)   \;  \mathcal D_{\bz_k}(\bx) 
\end{equation}
which should be viewed as the averaged margin between data points and classes which are at a distance $r$ of one another. We then have the following lemma:
  
\begin{lemma} \label{lemma:buffalo1} If $|S_{r}(k)| = |S_r(j)|$ for all $j,k \in [K]$ and all $r \in [L]$, then
\begin{align}
& \mathcal R_{0}(W,U) =   \log\left( 1 +  \sum_{r=1}^L   |S_r|  e^{ -  \mathfrak N_{W,U}(r)  } \right)  \qquad \text{ for all } (W,U) \in \mathcal E \\ 
& \mathcal R_{0}(W,U) >   \log\left( 1 +  \sum_{r=1}^L   |S_r|  e^{ -  \mathfrak N_{W,U}(r)  } \right)  \qquad \text{ for all } (W,U) \notin \mathcal E
\end{align}
\end{lemma}
\begin{proof}
Using the strict convexity of the function $f: \real^{K-1} \to \real$ defined by
 $$f(v_1, \ldots, v_{k-1}, v_{k+1}, \ldots , v_K) = \log \Big( 1+\sum_{j \neq k} e^{v_j}\Big)$$
  we obtain
\begin{multline*}  \mathcal R_0(W,U)  =  \frac{1}{K}  \sum_{k=1}^{K}   \sum _{\bx \in \data_k}  \log\left( 1 +  \sum_{j \neq k}   e^{ - \margin(\bx , j)  } \right)  \;   \mathcal D_{\bz_k}(\bx) 
 \ge    \frac{1}{K}  \sum_{k=1}^{K}   \log\left( 1 +  \sum_{j \neq k}   e^{ - \sum _{\bx \in \data_k}  \margin(\bx , j)    \mathcal D_{\bz_k}(\bx)  } \right)  
\end{multline*}
and equality holds if and only if, for all $k \in [K]$, we have that
\begin{equation} \label{buffalo1}
 \margin(\bx , j) =  \margin(\by , j) \qquad \text{ for all $\bx,\by \in \data_k$ and all $j \neq k$}
\end{equation}
We then let
\[
\overline{\mathfrak M}(k,j) = \sum _{\bx \in \data_k}  \margin(\bx , j)    \mathcal D_{\bz_k}(\bx)  
\]
and  use the strict convexity of the exponential function to obtain
\begin{align*}
   \frac{1}{K}  \sum_{k=1}^{K} 
     \log\left( 1 +  \sum_{j \neq k}   e^{ - \overline{\mathfrak M}(k,j)  } \right) 
    &  =   \frac{1}{K}  \sum_{k=1}^{K} 
    \log\left( 1 +  \sum_{r=1}^L \sum_{j \in S_r(k)}   e^{ - \overline{\mathfrak M}(k,j)  } \right)  \\
    &  =   \frac{1}{K}  \sum_{k=1}^{K} 
    \log\left( 1 +  \sum_{r=1}^L   |S_r|   \;  \frac{1}{|S_r|} \sum_{j \in S_r(k)}  e^{ - \overline{\mathfrak M}(k,j) } \right) \\
    &   \ge   \frac{1}{K}  \sum_{k=1}^{K}    \log\left( 1 +  \sum_{r=1}^L   |S_r|  e^{ -   \frac{1}{|S_r|}  \sum_{j \in S_r(k)} \overline{\mathfrak M}(k,j) } \right)
\end{align*}
Moreover, equality holds if and only if, for all $k\in[K]$ and all $r \in [L]$, we have that 
\begin{equation}  \label{buffalo2}
 \overline{\mathfrak M}(k,i)  =  \overline{\mathfrak M}(k,j) \quad \text{for all } i,j \in S_r(k) 
\end{equation}
We finally set 
\begin{equation*}
\overline{\overline{\mathfrak M}}(k,r) =  \frac{1}{|S_r|}  \sum_{j \in S_r(k)} \overline{\mathfrak M}(k,j)
\end{equation*}
and use  the strict convexity of the function %$f : \real^L \to \real$ defined by  
 $f(v_1, \ldots, v_{L} ) = \log\left( 1+\sum_{r=1}^{L} |S_r| e^{v_r}\right)$ to get
\begin{align*}
 % \frac{1}{K}  \sum_{k=1}^{K}    \log\left( 1 +  \sum_{r=1}^L   |S_r|  e^{ -   \frac{1}{|S_r|}  \sum_{k' \in S_r(k)} \overline{\mathfrak M}(k,k') } \right)
  %& = 
   \frac{1}{K}  \sum_{k=1}^{K}    \log\left( 1 +  \sum_{r=1}^L   |S_r|  e^{ -  \overline{\overline{\mathfrak M}}(k,r)  } \right) 
  & \ge   \log\left( 1 +  \sum_{r=1}^L   |S_r|  e^{ -    \frac{1}{K}  \sum_{k=1}^{K}   \overline{\overline{\mathfrak M}}(k,r)  } \right) 
\end{align*}
Moreover equality holds if and only if, for all $k\in[K]$ and all $r \in [L]$, we have that 
\begin{equation}  \label{buffalo3}
 \overline{\overline{\mathfrak M}}(k,r)  =  \overline{\overline{\mathfrak M}}(k',r)   \quad \text{for all } k,k' \in [K]  \text{ and all } r \in [L] 
\end{equation}

Importantly, note that
 $$\frac{1}{K}  \sum_{k=1}^{K}   \overline{\overline{\mathfrak M}}(k,r) =  \frac{1}{K}  \frac{1}{|S_r|}    \sum_{k=1}^{K}   \sum_{j \in S_r(k)} \sum _{\bx \in \data_k}  \margin_{W,U}(\bx , j)   \;  \mathcal D_{\bz_k}(\bx)  $$
 which is precisely how  $\mathfrak N_{W,U}(r)$ was defined (c.f. \eqref{average_margin}).
To conclude the proof, we remark that conditions \eqref{buffalo1}, \eqref{buffalo2} and \eqref{buffalo3} are all satisfied if and only if $(W,U)$ satisfies the equi-margin property. 
\end{proof}
We now show that, if  assumption \ref{symstrong} holds,    $\mathfrak N_{W,  U}(r)$ can be expressed in a simple way.
\begin{lemma}  \label{lemma:buffalo2}
 Assume that the latent variables satisfy the symmetry assumption \ref{symstrong}. Then
\begin{equation} \label{formula00}
 \mathfrak N_{W,U}(r) =
\frac{ \theta_r}{K}  \dotprodbis{\hat U,  W Q^T Z }  \qquad \text{ for all } (W,U) \in \mathcal N
 \end{equation}
\end{lemma}
\begin{proof}
We let 
$$
 \overline X_k =  \sum _{\bx \in \data_k}  \hot(\bx)  \mathcal D_{\bz_k}(\bx) 
$$
and note that the averaged margin can be expressed as
\begin{align}
  \mathfrak N_{W,U}(r) & =   \frac{1}{K}  \frac{1}{|S_r|}    \sum_{k=1}^{K}   \sum_{j \in S_r(k)} \sum _{\bx \in \data_k}  \margin_{W,U}(\bx , j)  \;   \mathcal D_{\bz_k}(\bx) \nonumber \\
  & =    \frac{1}{K}  \frac{1}{|S_r|}    \sum_{k=1}^{K}   \sum_{j \in S_r(k)} \sum _{\bx \in \data_k}  \dotprodbis{\hat U_k -\hat U_{j}, W  \hot(\bx)}  \mathcal D_{\bz_k}(\bx)  \nonumber \\
    & =    \frac{1}{K}  \frac{1}{|S_r|}    \sum_{k=1}^{K}   \sum_{j \in S_r(k)}   \dotprodbis{\hat U_k -\hat U_{j},   \overline X_k}  \nonumber \\
    & =  \frac{1}{K}    \sum_{k=1}^K \dotprodbis{\hat U_k, W\overline X_k}  \;\; - \;\;    \frac{1}{K}  \frac{1}{|S_r|}  \sum_{k=1}^K \sum_{j \in S_r(k)} \dotprodbis{\hat U_{j}, W\overline X_k}  \label{tututu} 
\end{align}
Let
$$
a^{(r)}_{k,j} = \begin{cases} 1 & \text{ if } \dist(\bz_k, \bz_{j} ) =r \\
0 &  \text{otherwise}
\end{cases}
$$
and rewrite the  second term in \eqref{tututu} as
\begin{align*}
 \frac{1}{K}  \frac{1}{|S_r|}  \sum_{k=1}^K \sum_{j \in S_r(k)} \dotprodbis{\hat U_{j}, W\overline X_k} &=  \frac{1}{K}  \frac{1}{|S_r|}  \sum_{k=1}^K \sum_{j=1}^K  a^{(r)}_{k,j}  \dotprodbis{U_{j}, W\overline X_k} \\
 &  =  \frac{1}{K}  \frac{1}{|S_r|}  \sum_{j=1}^K   \sum_{k=1}^K a^{(r)}_{j,k}  \dotprodbis{\hat U_{k}, W\overline X_j} \\
  &  =  \frac{1}{K}  \frac{1}{|S_r|}  \sum_{k=1}^K \sum_{j \in S_r(k)} \dotprodbis{\hat U_{k}, W\overline X_j} \\
   &  =  \frac{1}{K}  \sum_{k=1}^K \dotprodbis{ \; \hat U_{k} \; , \;  W  \;  \frac{1}{|S_r|}   \sum_{j \in S_r(k)}  \overline X_j \; }
\end{align*}  
Combining this with \eqref{tututu} we obtain
\begin{equation} \label{sususu}
 \mathfrak N_{W,U}(r) =  \frac{1}{K}  \sum_{k=1}^K \dotprodbis{ \;\hat U_k \; , \;  W \Big( \overline X_k-    \frac{1}{|S_r|} \sum_{j \in S_r(k)} \overline X_{j} \Big) \;}
\end{equation}
From
formula \eqref{matrixQ}, we see that  row  $\alpha$ of the matrix $Q$ is given by the formula
\begin{equation} \label{Qrow}
Q^T \be_\alpha = \sum_{\beta=1}^{s_c} \hot(\alpha,\beta) \; \mu_\beta.  %\qquad 1 \le \alpha \le n_c
\end{equation}
We then write  $\bz_k = [\alpha_1, \ldots, \alpha_L]$ and note   that the $\ell^{th}$ column of  $\overline X_k$ can be expressed as
\begin{align}
 \left[\overline X_k \right]_{:,\ell} =  \sum _{\beta=1}^{n_c}  \hot(\alpha_\ell,\beta)  \mu_\beta = Q^T \be_{\alpha_\ell}.
\end{align}
From this we obtain that 
\begin{equation*} 
\overline  X_k =   Q^T Z_k 
\end{equation*}
and therefore \eqref{sususu} becomes
\begin{align*}
 \mathfrak N_{W,U}(r) &=  \frac{1}{K}  \sum_{k=1}^K \dotprodbis{ \;\hat U_k \; , \;  WQ^T \Big( Z_k-    \frac{1}{|S_r|} \sum_{j \in S_r(k)}  Z_{j} \Big) \;} \\
   & = \frac{1}{K}  \sum_{k=1}^K \dotprodbis{ \;\hat U_k \; , \;  W Q^T \Big( \theta_r Z_k  + A_r \Big) \;}
\end{align*}
where we have used the identity
$
Z_k - \frac{1}{|S_r|} \sum_{j \in S_r(k)}  Z_{j}  =    \theta_r Z_k  + A_r
$ to obtain the second equality.
Finally, we use the fact that  $\sum_{k} \hat  U_k =0$ to obtain
\begin{equation*}
 \mathfrak N_{W,U}(r) =   \frac{\theta_r}{K}  \sum_{k=1}^K \dotprodbis{ \;\hat U_k \; , \;  W Q^T Z_k  \;} = \frac{ \theta_r}{K}  \dotprodbis{\hat U,  W Q^T Z }
\end{equation*}
\end{proof}

Combining lemma \ref{lemma:buffalo1} and \ref{lemma:buffalo2} concludes the proof of theorem \ref{theorem:buffalo}.

\section{Proof of theorem \ref{theorem:1} and its converse} \label{section:thm1}
In this section we prove theorem \ref{theorem:1} under assumption \ref{symstrong},  and its converse under assumptions \ref{symstrong} and \ref{symconv}. We start by recalling the definition of a type-I collapse configuration.
\begin{definition}[Type-I Collapse] \label{def:1ap}  The weights $(W,U)$ of the network $h_{W,U}$ form a type-I collapse configuration if and only if the conditions
\begin{enumerate}[topsep=0pt]
\item[\rm i)] There exists $c \ge 0$ so that $\bw_{(\alpha,\beta)} = c    \,\cpt_\alpha$ \;  for all $(\alpha,\beta) \in \voc$.
\item[\rm ii)] There exists $c' \ge 0$ so that $\bu_{k, \ell} =  c'\,  \cpt_\alpha$ \; for all  $(k,\ell)$ satisfying $z_{k, \ell}=\alpha$ and all $\alpha \in \mathcal C$.
\end{enumerate}
hold for some collection $\cpt_1, \ldots, \cpt_{n_c} \in \real^d$ of equiangular vectors.
\end{definition}
It will prove convenient to reformulate this definition using matrix notations. Toward this goal, we define equiangular matrices as follow:
\begin{definition}(Equiangular Matrices) A matrix $\mathfrak F \in \real^{d \times n_c}$ is said to be equiangular if and only if the relations
$$
 \mathfrak F \, \ones_{n_c} = 0 \qquad \text{ and } \qquad \mathfrak F^T \mathfrak F = \frac{n_c}{n_c-1} \;I_{n_c} - \frac{1}{n_c-1} \; \ones_{n_c} \ones_{n_c}^T
$$
hold.
\end{definition}
Comparing the above definition with the definition of equiangular vectors provided in the main paper, we easily see that
a matrix
$$
\mathfrak F = \begin{bmatrix} 
\vert & \vert & & \vert \\
\mathfrak f_1 & \mathfrak f_2 & \cdots & \mathfrak f_{n_c} \\ 
\vert & \vert & & \vert \\
\end{bmatrix} \in \real^{d \times n_c}
$$
is equiangular if and only if its columns $\cpt_1, \ldots, \cpt_{n_c} \in \real^d$ are equiangular. Relations (i) and (ii)   defining a type-I collapse configuration  can now  be expressed  in matrix format as
$$
W = c\; \mathfrak F \; P \qquad \text{ and } \qquad \hat U = c' \; \mathfrak F  \; Z  \qquad \text{ for some equiangular matrix $\mathfrak F$}
$$
where the matrices $Z$ and $P$ are given by formula  \eqref{Z} and \eqref{P}.  We then let
\begin{multline} \label{setOmegaI}
\Omega^{I}_{c} := \Big\{ (W, U) :  \text{There exist an equiangular matrix $\mathfrak F$ such that } \\   W = c\; \mathfrak F \; P \quad   \text{ and } \quad    \hat U = c \; \sqrt{\frac{n_w}{KL}} \;  \mathfrak F  \; Z    \Big\}
\end{multline}
and note that $\Omega^{I}_{c}$ is simply  the set of weights $(W,U)$ which are in a type-I collapse configuration with constant $c$ and $c' = c \; \sqrt{ n_w / (KL)}$.  We now state the main theorem of this section.

\begin{theorem}  \label{theorem:thm1}
Assume uniform sampling $\mu_\beta = 1 / s_c\,$ for each word distribution.  Let $\tau \ge 0$ denote the unique minimizer of the strictly convex function
$$
H(t) :=  \log \left( 1 -  \frac{K}{n_c^L}  +  \frac{K}{n_c^L} \Big( 1+ (n_c-1) e^{-\eta t } \Big)^L \right) + \lambda t \qquad \text{where }\quad  \eta =    \frac{n_c}{n_c-1} \; \frac{1}{\sqrt{n_w KL}}
$$
and let   $c = \sqrt{\tau / n_w}$. 
Then we have the following:
\begin{enumerate}
\item[(i)] If the latent variables $\bz_1, \ldots, \bz_K$ are mutually distinct and satisfy  assumption \ref{symstrong}, then 
$$
 \Omega^{I}_{c}  \subset \arg \min \mathcal R
$$
\item[(ii)] If the latent variables $\bz_1, \ldots, \bz_K$ are mutually distinct and satisfy  assumptions \ref{symstrong} and \ref{symconv}, then 
$$
 \Omega^{I}_{c}  = \arg \min \mathcal R
$$
\end{enumerate}
\end{theorem}
Note that (i) states that any $(W,U) \in \Omega_{c}^I$ is a minimizer of the regularized risk --- this corresponds to theorem \ref{theorem:1} from the main paper. Statement (ii) assert that any minimizer of the regularized risk must belong to $\Omega_{c}^I$ --- this is the converse of theorem \ref{theorem:1}. 
The remainder of this section is devoted to the proof of theorem \ref{theorem:thm1}. We will assume uniform sampling 
\begin{equation*} 
\mu_\beta = 1 / s_c \qquad \text{for all } \beta \in [s_c]
\end{equation*}
 everywhere in this section --- all  lemmas and propositions are proven under this assumption, even when not explicitly stated.

\subsection{The bilinear optimization problem} \label{sub:bilinear}
From theorem \ref{theorem:buffalo}, it is clear that the quantity
$$
 \dotprodbis{\hat U,  W Q^TZ }
$$
plays an important role in our analysis.
In this subsection   we consider the bilinear optimization problem
\begin{align}
& \text{maximize } \;  \dotprodbis{\hat U,  W Q^TZ }    \label{base1}  \\
& \text{subject to } \quad  \frac{1}{2} \left( \|W\|_F^2 + \|\hat U\|_F^2 \right) =  c^2  \, n_w  \label {base2}
\end{align}
where $c \in \real$ is some  constant.
The following lemma identifies all solutions of this optimization problem.
\begin{lemma} \label{lemma:bilinearI}  Assume the latent variables  satisfy  assumption \ref{symstrong}.
Then  $(W,U)$ is a solution of the optimization problem \eqref{base1} -- \eqref{base2}  if and only if it belongs to the set
\begin{multline} \label{bigB}
\mathcal B^I_{c} =\Big\{ (W, U) :  \text{There exist a matrix $F \in \real^{d \times n_c}$ with $\|F\|_F^2 = n_c$  }  \\ \text{ such that }     W = c \,   F  P \text{ and }  \hat U  =  c \; \sqrt{\frac{n_w}{KL}} \;  F   Z  \Big\}
\end{multline}
\end{lemma}
 Note that the set $\mathcal B^I_{c}$  is very similar to the set $\Omega^I_{c}$  that defines type-I collapse configuration (c.f.  \eqref{setOmegaI}). In particular,  since an equiangular matrix has $n_c$ columns of norm $1$, it  always satisfies 
 $\| \mathfrak F \|_F^2 = n_c$, 
and therefore we have the inclusion
\begin{equation} \label{inclusionI}
\Omega^I_{c} \subset \mathcal B^I_{c}. 
\end{equation}
The remainder of this subsection is devoted to the proof of the lemma.

First note that the lemma is trivially true if $c=0$, so we may assume $c\neq0$ for the remainder of the proof.
Second, we note that since $\mu_\beta = 1/s_c$, then the matrices $P$ and $Q$ defined by \eqref{P} and \eqref{matrixQ} are scalar multiple of one another. We may therefore replace the matrix $Q$ appearing in \eqref{base1} by $P$, wich leads to
\begin{align}
& \text{maximize } \;  \dotprodbis{\hat U,  W P^TZ }    \label{optA}  \\
& \text{subject to } \quad  \frac{1}{2} \left( \|W\|_F^2 + \|\hat U\|_F^2 \right) =  c^2  \, n_w  \label {optB}
\end{align}
We now  show that any $(W,\hat U) \in \mathcal B_c^I$ satisfies the constraint \eqref{optB} and have objective value equal to $ s_c\, c^2    \sqrt{ KL n_w} $.
\begin{claim} If $(W,\hat U) \in \mathcal B_c^I$, then 
\begin{align*}
 \frac{1}{2} \left( \|W\|_F^2 + \|\hat U\|_F^2 \right) =  c^2  \, n_w      \qquad \text{ and } \qquad  \dotprodbis{\hat U,  W P^TZ }  =  c^2 \, s_c \,  \sqrt{ KL n_w}  
\end{align*}
\end{claim}
\begin{proof}
Assume  $(W,U) \in \mathcal B_c^I$.
From  definition \eqref{P} of the matrix $P$, we have  $PP^T = s_c  I_{n_c}$, and therefore
\begin{align*}
\| W \|^2_F = c^2\|FP \|^2_F = c^2 \dotprodbis{  FP ,FP }  = c^2 \dotprodbis{ FPP^T , F } =    c^2 \, s_c \,  \| F \|_F^2  = c^2 \, s_c \, n_c  = c^2 \, n_w
\end{align*}
where we have used the fact that $s_c = n_w/n_c$. Using  $ZZ^T = \frac{KL}{n_c} I$ from lemma \ref{lemma:sym}, we obtain 
\begin{align*}
\| FZ  \|^2_F =  \dotprodbis{  FZ ,FZ }  =  \dotprodbis{  FZZ^T ,F }  =   \left(   \frac{KL}{n_c} \right) \| F \|_F^2  = KL
\end{align*}
As a consequence we have 
\begin{align*}
\| \hat U \|^2_F = c^2 \frac{n_w}{KL} \| FZ  \|^2_F = c^2 \, n_w
\end{align*}
and, using  $PP^T = s_c  I_{n_c}$ one more time, 
\begin{align*}
 \dotprodbis{\hat U,  W P^TZ } = c^2 \sqrt{ \frac{n_w}{KL} }  \dotprodbis{ FZ,  FP P^TZ } =  c^2 \, s_c \,  \sqrt{ \frac{n_w}{KL} }  \dotprodbis{ FZ,  FZ } =  c^2 \, s_c \,  \sqrt{ KL n_w}  
\end{align*}
\end{proof}

We then prove that $W$ and $\hat U$ must have same Frobenius norm if they solve the optimization problem.
\begin{claim}
If $(W,U)$ is a solution of \eqref{optA} -- \eqref{optB}, then
\begin{equation} \label{samenorm}
\|W\|_F^2 = \|\hat U\|_F^2  = c^2  \, n_w
\end{equation}
\end{claim}
\begin{proof}
We prove it by contradiction. Suppose $(W,\hat U)$ is a solution of \eqref{base1}--\eqref{base2} with   $\|W \|_F^2 \neq \|\hat U \|_F^2$. 
Since the average of $\|W\|_F^2$ and  $\|\hat U\|_F^2$ is equal to  $c^2   n_w>0$ according to the constaint, there must then exists
 $\epsilon \neq 0$ such that
$$
\|W \|_F^2 = c^2 n_w + \epsilon \qquad \text{ and } \qquad  \|\hat U\|_F^2  = c^2 n_w - \epsilon 
$$ 
Let 
$$
W_0=  \sqrt{\frac{c^2 n_w}{c^2 n_w+ \epsilon}}  \; W \qquad \text{ and } \qquad \hat U_0 =  \sqrt{\frac{c^2 n_w}{c^2 n_w- \epsilon}}  \; \hat U
$$ 
and note that  
$$
\|W_0\|_F^2 = \|\hat U_0\|_F^2  = c^2  \, n_w
$$
and therefore $(W_0,\hat U_0)$  clearly satisfies the constraint.  We also have 
$$
 \dotprodbis{\hat U_0,  W_0 P^TZ } =  \sqrt{\frac{c^4 n_w^2}{c^4 n_w^2-\epsilon^2}}    \dotprodbis{\hat U ,  W  P^TZ } > \dotprodbis{\hat U ,  W  P^TZ } 
$$
since $\epsilon \neq 0$ and therefore $(W,\hat U)$ can not be a maximizer, which is a contradiction. 
\end{proof}

As a consequence of the above claim, the optimization problem \eqref{optA} -- \eqref{optB} is equivalent to
\begin{align}
& \text{maximize } \; \dotprodbis{\hat U,  W P^TZ }  \label{optAAA}  \\
& \text{subject to } \quad  \|W\|_F^2  = c^2 \, n_w  \quad \text{and} \quad \|\hat U\|_F^2 = c^2 \, n_w \label{optBBB}
\end{align}
We then have
\begin{claim} If $(W,\hat U)$ is a solution of \eqref{optAAA} -- \eqref{optBBB}, then $(W,\hat U) \in \mathcal B_c^I$. 
\end{claim}
Note that according to the first claim, all $(W,\hat U) \in \mathcal B_c^I$ have same objective value, and therefore, according to the above claim, they must all be maximizer.  As a consequence,  proving the above claim will conclude the proof of lemma \ref{lemma:bilinearI}.

\begin{proof}[Proof of the claim]
Maximizing \eqref{optAAA} over $\hat U$ first gives
\begin{equation} \label{recover}
\hat U = c\, \sqrt{n_w}  \frac{W P^TZ}{\| W P^TZ \|_F}
\end{equation}
and therefore the optimization problem \eqref{optAAA} -- \eqref{optBBB}  reduces to
\begin{align*}
& \text{maximize } \quad   \|W P^TZ\|^2_F     \\
& \text{subject to } \quad  \|W\|_F^2  = c^2 \, n_w 
\end{align*}
Using  $ZZ^T = \frac{KL}{n_c} I$ from lemma \ref{lemma:sym} we then get
\begin{equation*}
 \|  W P^T Z  \|^2_F =\dotprodbis{  W P^T Z,   W P^T Z }  = \dotprodbis{  W P^T ZZ^T,   W P^T } =   \frac{KL}{n_c}  \| W P^T \|_F^2
\end{equation*}
and therefore the  problem further reduces to
\begin{align*}
& \text{maximize } \; \|W P^T\|_F^2     \\
& \text{subject to } \quad  \|W\|_F^2  = c^2 \, n_w 
\end{align*}
The KKT conditions for this optimization problem are
\begin{align}
& WP^TP = \nu W \label{babar} \\
& \| W \|_F^2 = c^2 \, n_w  \label{bibir}
\end{align} 
where $\nu \in \real$ is the Lagrange multiplier.   

Assume  that $(W,\hat U)$ is a solution of the original optimization problem \eqref{optAAA} -- \eqref{optBBB}. Then, according to the above discussion,  $W$ must satisfy \eqref{babar} -- \eqref{bibir}. Right multiplying \eqref{babar} by $P^T$, and using $PP^T = s_c  I_{n_c}$,  gives
$$
s_c WP^T = \nu WP^T
$$
So either $\nu = s_c$ or $WP^T=0$. The latter is not possible since the choice $WP^T=0$ leads to an objective value equal to zero in the original optimization problem  \eqref{optAAA} -- \eqref{optBBB}. We must therefore have 
$\nu = s_c$, and equation \eqref{babar} becomes
\begin{equation} \label{hahaha}
W = \frac{1}{s_c} W P^T P
\end{equation}
which can obviously be written as 
$$
W = c \, FP
$$
by setting $F:= \frac{1}{c \, s_c} W P^T$.
Since $W$ satisfies  \eqref{bibir} we must have 
\begin{equation}
c^2 \, n_w =  \| W \|^2_F =  c^2 \|  FP \|^2_F = c^2 \dotprodbis{  FP ,FP }  = c^2  \dotprodbis{ FPP^T , F } =   c^2 \, s_c  \| F \|_F^2,
\end{equation}
and so   $\| F \|_F^2 =  n_w / s_c = n_c$. 

According to \eqref{recover}, $\hat U$ bust be a scalar multiple of  the matrix 
$$
W P^T Z =  (c FP) P^T Z =  c \, s_c \, F Z
$$ 
Using the fact that  $ZZ^T = \frac{KL}{n_c} I$  and $\|F\|_F^2 = n_c$ we then obtain that 
\begin{equation}
 \|  FZ \|^2_F = \dotprodbis{  FZ ,FZ }  =\dotprodbis{ FZZ^T , F } =    \frac{KL}{n_c}  \| F \|_F^2 =  KL
\end{equation}
and so equation \eqref{recover} becomes
\begin{align}
\hat U = c\, \sqrt{n_w}  \frac{W P^TZ}{\| W P^TZ \|_F} =  c\, \sqrt{n_w}  \frac{F Z}{\sqrt{KL}} 
\end{align}
which concludes the proof.
\end{proof}

\subsection{Proof of collapse}
Recall that the regularized risk associated with the network $h_{W,U}$ is defined by
\begin{equation} \label{flute}
\mathcal R(W,U) =  \mathcal R_0(W,U) +  \frac{\lambda}{2} \left( \|W\|_F^2 +  \|U\|_F^2  \right) 
\end{equation}
and recall that the set of weights in type-I collapse configuration is
\begin{multline} \label{OmegaI}
\Omega^{I}_{c} = \Big\{ (W, U) :  \text{There exist an equiangular matrix $\mathfrak F$ such that } \\   W = c\; \mathfrak F \; P \quad   \text{ and } \quad    \hat U = c \; \sqrt{\frac{n_w}{KL}} \;  \mathfrak F  \; Z    \Big\}
\end{multline}

This subsection is devoted to the proof of the following proposition.
\begin{proposition}  \label{proposition:collapseI}
We have the following:
\begin{enumerate}
\item[(i)] If the latent variables $\bz_1, \ldots, \bz_K$ are mutually distinct and satisfy  assumption \ref{symstrong}, then there exists $c\in \real$ such that
$$
 \Omega^{I}_{c}  \subset \arg \min \mathcal R
$$
\item[(ii)] If the latent variables $\bz_1, \ldots, \bz_K$ are mutually distinct and satisfy  assumptions \ref{symstrong} and \ref{symconv}, then any $(W, U)$ that minimizes $\mathcal R$ must belong to $ \Omega^{I}_{c}$ for some $c\in \real$. 
\end{enumerate}
\end{proposition}
This proposition states that, under appropriate symmetry assumption, the weights of the network $h_{W,U}$ do collapse into a type-I configuration. This proposition however does not provide the value of the constant $c$ involved in the collapse. Determining this constant will be done in the subsection  \ref{sub:constants}. 

We start with a simple lemma.
\begin{lemma} \label{lemma:inN} Any  global minimizer of \eqref{flute} must belong to $\mathcal N$.
\end{lemma}
\begin{proof}
Let $(W^\star,U^\star)$ be a global minimizer. Define
$
B = \frac{1}{K} \sum_{k=1}^K U^\star_k 
$
and 
$$
 U_0= \begin{bmatrix} U_1^\star - B & U_2^\star - B & \cdots & U_K^\star-B \end{bmatrix}
$$
From the definition of the unregularized risk we have
$
\mathcal R_0(W^\star;  U_0) =   \mathcal R_0(W^\star;  U^\star)
$
and therefore
\begin{align*}
\frac{1}{K} \left( \mathcal R(W^\star;  U_0) -   \mathcal R(W^\star;  U^\star) \right) &=
\frac{\lambda}{2} \frac{1}{K}\sum_{k=1}^K \left(  \|U^\star_k-B\|_F^2 - \|U^\star_k\|_F^2 \right) \\
&= \frac{\lambda}{2} \frac{1}{K}\sum_{k=1}^K \left(  \|B\|_F^2  -2 \dotprodbis{B, U^\star_k}  \right) \\
& =  \frac{\lambda}{2}  \left(  \|B\|_F^2  -2 \dotprodbis{B, \frac{1}{K}\sum_{k=1}^KU^\star_k}  \right) \\
& =  - \frac{\lambda}{2}    \|B\|_F^2 
\end{align*}
So $B$ must be equal to zero, otherwise we would have $\mathcal R(W^\star,  U_0) < \mathcal R(W^\star, U^\star)  $.
\end{proof}

The next lemma bring together the bilinear optimization problem from subsection \ref{sub:bilinear} and  the sharp lower bound on the unregularized risk that we derived in section \ref{section:lowerbound}.

\begin{lemma} \label{lemma:silvermonkeyI}
Assume the latent variables  satisfy  assumption \ref{symstrong}.
Assume also that  $(W^\star,U^\star)$ is a global minimizer of  \eqref{flute} and let $c\in \real$ be such that
$$
  \frac{1}{2} \left( \|W^\star\|_F^2 +  \|U^\star\|_F^2  \right) = c^2 \, n_w.
$$
 Then the following hold:
\begin{enumerate}
\item[(i)] Any $(W,U)$ that belongs to
$
  \mathcal N \cap \mathcal E \cap \mathcal B^I_c
$
is also a global minimizer of \eqref{flute}.
\item[(ii)] If $\mathcal N \cap \mathcal E \cap \mathcal B^I_{c} \neq \emptyset$, then $(W^\star,U^\star)$ must belong to
$
  \mathcal N \cap \mathcal E \cap \mathcal B^I_c.
$

\end{enumerate}\end{lemma}

\begin{proof}
Recall from theorem \ref{theorem:buffalo} that
\begin{align}
\mathcal R_0(W,U) &=  g \Big( -   \dotprodbis{\hat U,  W Q^TZ }  \Big) \qquad \text{ for all } (W,U) \in \mathcal N \cap \mathcal E \\
\mathcal R_0(W,U) &>  g \Big( -   \dotprodbis{\hat U,  W Q^TZ }  \Big) \qquad \text{ for all } (W,U) \in \mathcal N \cap \mathcal E^c \label{kokokaa}
\end{align}
We start by proving (i).  
If  $(W, U) \in \mathcal N \cap \mathcal E \cap \mathcal B^I_c$,  then we have 
\begin{align*}
\mathcal R_0 (W^\star,U^\star) 
& \ge g\left(  - \dotprodbis{\hat U^\star,   W^\star Q^TZ }\right)  && \text{[because $( W^\star,U^\star) \in \mathcal N$ due to lemma \ref{lemma:inN} ]} \\
&\ge g\left(  - \dotprodbis{\hat U,  \ W Q^TZ }\right)  && \text{[because $( W,U) \in \mathcal B^I_c$ and $g$ is increasing]}  \\
 & = \mathcal R_0(W,U) & & \text{[because $( W,U) \in \mathcal N \cap \mathcal E$ ]} 
\end{align*}
Since
$(W, U) \in \mathcal B_c^I$ we must have 
 $ \frac{1}{2} \left( \|W\|_F^2 + \| U\|_F^2 \right) = c^2 \, n_c =  \frac{1}{2} \left( \|W^\star\|_F^2 + \| U^\star\|_F^2 \right)  $. Therefore  $\mathcal R(W,U) \le \mathcal R(W^\star,U^\star) $ and  $(W,U)$ is a minimizer. 

We now prove (ii) by contradiction.
Suppose that $( W^\star, U^\star) \notin  \mathcal N \cap \mathcal E \cap \mathcal B^I_c$. This must mean that  
$$ ( W^\star, U^\star) \notin \mathcal E \cap \mathcal B^I_c $$
since it clearly belongs to $\mathcal N$.
If ($ W^\star, U^\star) \notin \mathcal E$ then the first inequality in the above computation is strict according to \eqref{kokokaa}. If ($ W^\star, U^\star) \notin  \mathcal B^I_c$ then the second inequality is strict because $g$ is strictly increasing.
\end{proof}

The above lemma establishes connections between the set of minimizers of the risk and the set $ \mathcal E \cap \mathcal N \cap \mathcal B^I_c$. The next two 
lemmas shows that the set $ \mathcal E \cap \mathcal N \cap \mathcal B^I_c$ is closely related to the set of collapsed configurations  $\Omega_c^{I}$. In other words we use the set $ \mathcal E \cap \mathcal N \cap \mathcal B^I_c$ as a bridge between the set of minimizers and the set of type-I collapse configurations.

\begin{lemma} \label{tournette}
 If the latent variables satisfy the  symmetry assumption \ref{symstrong}, then
\[
\Omega_c^{I} \subset \mathcal E \cap \mathcal N \cap \mathcal B^I_c
\]
\end{lemma}
\begin{proof}
We already know from \eqref{inclusionI} that $\Omega^I_{c} \subset \mathcal B^I_{c}$.  
We now show that $\Omega^I_{c} \subset  \mathcal E$.
Suppose 
  $(W,U) \in \Omega^I_c$. Then there exists an equiangular  matrix $\mathfrak F  \in \real^{d \times n_c}$  such that
\begin{equation*}
   W = c \;   \mathfrak  F \;  P \qquad  \text{ and }  \qquad  \hat U  = c' \;    \mathfrak F \;  Z  
\end{equation*}
where $c' =  c \sqrt{ n_w / (KL)}$. Recall from \eqref{Phot} that
$$
P \hot(\bx) =  Z_k \qquad \text{for all $\bx \in \data_k$.}
$$
Consider two latent variables
$$
\bz_k = [\alpha_1, \ldots, \alpha_L] \quad \text{ and } \bz_j = [\alpha'_1, \ldots, \alpha'_L] 
$$
and assume $\bx$ is generated by $\bz_k$, meaning that  $\bx \in \data_k$. We then have 
\begin{align*}
\margin_{W,U}(\bx , j) &= \dotprodbis{\hat U_{k} - \hat U_{j}, W \hot(\bx)} \\
& = c \, c'  \;  \dotprodbis{\mathfrak F\; Z_{k} -\mathfrak F\; Z_{j}, \mathfrak F\; P \hot(\bx)} \\
%& =   \sqrt{\frac{n_c}{RL}}  \dotprodbis{CZ_{k} -CZ_s, CQ^T X_i} \\
& =    c\, c' \;   \dotprodbis{\mathfrak F\; Z_{k} -\mathfrak F\; Z_{j}, \mathfrak F\; Z_k}  \\
& =   c\, c'  \; \sum_{\ell =1}^L \dotprodbis{ \; \cpt_{\alpha_{\ell } } -\cpt_{\alpha'_{\ell } }  \;,\;  \cpt_{\alpha_{\ell} }  \; } \\
& = c\, c' \;\; \left( L -  \sum_{\ell =1}^L \dotprodbis{ \; \cpt_{\alpha'_{\ell } }  \;,\;  \cpt_{\alpha_{\ell} }  \; }     \right) 
\end{align*}
Since $\mathfrak f_1, \ldots, \mathfrak f_{n_c}$ are equiangular, we have 
$$
\sum_{\ell =1}^L \dotprodbis{ \; \cpt_{\alpha'_{\ell } }  \;,\;  \cpt_{\alpha_{\ell} }  \; }      =   \Big( L - \dist(\bz_j,\bz_{k}) \Big) - \frac{1}{n_c-1} \dist(\bz_j,\bz_{k}) = L - \frac{n_c}{n_c-1} \dist(\bz_j,\bz_{k}). 
$$
Therefore 
$$
\margin_{W,U}(\bx , j) = cc' \frac{n_c}{n_c-1} \dist(\bz_j,\bz_{k}) 
$$
and it is  clear that the margin  only depends on  $\dist(\bz_j,\bz_{k})$,  and therefore $(W,U)$ satisfies the equimargin property.

Finally we show that $\Omega^I_{c} \subset  \mathcal N$.  
 Suppose  $(W,U) \in \Omega^I_c$.  From property (ii) of lemma \ref{lemma:sym} we have
$$
\sum_{k=1}^K Z_k = \frac{K}{n_c} \ones_{n_c} \ones^T_L
$$
Therefore, 
\begin{align*}
\sum_{k=1}^K \hat U_k =   c' \sum_{k=1}^K    \mathfrak F \; Z_k =  c' \;\;  \frac{K}{n_c}  \mathfrak F \; \ones_{n_c} \ones^T_L = 0
\end{align*}
where we have used the fact that  $\mathfrak F \; \ones_{n_c} = 0$.
\end{proof}

\begin{lemma} \label{tournette2} 
If  the latent variables satisfy  assumptions \ref{symstrong} and \ref{symconv},  then 
\[
\Omega^{I}_{c}   =  \mathcal E \cap \mathcal N \cap \mathcal B^I_{c} 
\]
\end{lemma}
\begin{proof}
From the previous lemma we know that $
\Omega^{I}_{c}   \subset \mathcal E \cap \mathcal N \cap \mathcal B^I_{c}  
$ so we need to  show that
$$
 \mathcal E \cap \mathcal N \cap \mathcal B^I_{c} \subset \Omega^{I}_{c}. 
$$
Let $ (W,U) \in \mathcal E \cap \mathcal N \cap \mathcal B^I_{c} $. Since $(W,U)$  belongs to $\mathcal B^I_c$, there exists a matrix $F \in \real^{d \times n_c}$ with $\|F\|_F^2 = n_c$ such that
\begin{equation}
   W = c \;   F \;  P \qquad  \text{ and }  \qquad  U  = c' \;    F \;  Z  \label{montblanc}
\end{equation}
where $c' =  c \sqrt{ n_w / (KL)}$.
Our goal is to show that $F$ is equiangular, meaning that it satisfies the two relations
\begin{equation} \label{lac_leman}
 F \, \ones_{n_c} = 0 \qquad \text{ and } \qquad  F^T  F = \frac{n_c}{n_c-1} \;I_{n_c} - \frac{1}{n_c-1} \; \ones_{n_c} \ones_{n_c}^T.
\end{equation}

The first relation is easily obtained. Indeed, using the fact that $ (W,U) \in \mathcal N$ together with the identity
$
\sum_{k=1}^K Z_k = \frac{K}{n_c} \ones_{n_c} \ones^T_L
$
(which hold due to lemma \ref{lemma:sym}),  we obtain  
$$
0=\sum_{k=0}^K U_k =  c'   \sum_{k=0}^K F Z_k  =  c' \frac{K}{n_c} F \ones_{n_c} \ones_L^T. 
$$
We then note that the matrix $F \ones_{n_c} \ones_L^T$ is the zero matrix if and only if
 $F \ones_{n_c} =0$. 
 
We now prove the second equality of \eqref{lac_leman}. Assume that  $\bx \in \data_k$.  Using the fact that $
P \hot(\bx) =  Z_k
$ together with \eqref{montblanc}, we obtain
\begin{align}
\margin_{W,U}(\bx , j) &= \dotprodbis{\hat U_{k} - \hat U_{j}, W \hot(\bx)} \nonumber \\
& = c\, c' \;  \dotprodbis{F\; Z_{k} - F\; Z_{j},  F\; P \hot(\bx)}  \nonumber \\
%& =   \sqrt{\frac{n_c}{RL}}  \dotprodbis{CZ_{k} -CZ_s, CQ^T X_i} \\
& =    c\,c'   \dotprodbis{F\; Z_{k} -F\; Z_{j}, F\; Z_k}   \nonumber \\
& =  c\,c'   \dotprodbis{F^TF(Z_{k} -Z_{j}), Z_k}  \nonumber \\
& =  c\,c'    \dotprodbis{\; F^TF \; , \; \Gamma^{(k,j)} \;} \label{clement}
%& = c\,c'  \sum_{\ell =1}^L \dotprodbis{ \; {\bf f}_{\alpha_{\ell } } -{\bf f}_{\alpha'_{\ell } }  \;,\;  {\bf f}_{\alpha_{\ell} }  \; } \\
\end{align}
We recall that the matrices
$$
\Gamma^{(k,j)} =  Z_k(Z_k-Z_{j})^T \in \real^{n_c \times n_c}.
$$
are precisely the ones involved in the statement of assumption \ref{symconv}.
Since $(W,U) \in \mathcal E$, the margins must only depend on the distance between the latent variables.  Due to \eqref{clement},  we can be express this as 
\begin{equation*} 
\dotprodbis{ \; F^TF \; , \;   \Gamma^{(j,k)}} =  \dotprodbis{ \; F^TF \; ,  \; \Gamma^{(j',k')}  }  \qquad   \forall j,k,j',k' \in [K]  \text{ s.t. } \dist(\bz_{j},\bz_{k}) = \dist(\bz_{j'},\bz_{k'})
\end{equation*}
Since the  $F^TF$ is clearly positive semi-definite, we may then use assumption  \ref{symconv} to conclude that
$
F^TF \in \mathcal A
$. Recalling definition \eqref{calA} of the set $\mathcal A$, we therefore have
\begin{equation} \label{lac_annecy}
F^T F =  a  \;I_{n_c} + b\; \ones_{n_c} \ones_{n_c}^T
\end{equation}
for some $a,b \in \real $. To conclude our proof, we need to show that
\begin{equation} \label{sol99}
a = \frac{n_c}{n_c-1} \qquad \text{and} \qquad b = -\frac{1}{n_c-1}.
\end{equation}

Combining \eqref{lac_annecy} with the first equality of \eqref{lac_leman},  we obtain
\begin{equation} \label{hh1}
0 = F^T F\,  \ones_{n_c} = a  \;\ones_{n_c} + b\; \ones_{n_c} \ones_{n_c}^T \ones_{n_c} = (a+b n_c) \ones_{n_c}
\end{equation}
 Combining \eqref{lac_annecy} with the fact that $\|F\|_F^2 = n_c$,  we obtain
 \begin{equation}\label{hh2}
 n_c = \|F\|_F^{2} = \tr(F^TF) = n_c(a+b)
 \end{equation}
 The constants $a,b \in \real$, according to \eqref{hh1} and \eqref{hh2} must therefore solve the system
 $$
 \begin{cases}
 a+ bn_c &= 0 \\
 a + b &= 1
 \end{cases}
 $$
 and one can easily check that the solution of this system is precisely given by \eqref{sol99}.
\end{proof}

We conlude this subsection by proving proposition \ref{proposition:collapseI}.

\begin{proof}[Proof of Proposition  \ref{proposition:collapseI}]  Let $(W^\star, U^\star)$ be a global minimizer of $\mathcal R$ and let $c \in \real$ be such that
$$
 \frac{1}{2} \left( \|W^\star\|_F^2 +  \|U^\star\|_F^2  \right) = c^2 \, n_w
$$
If the  latent variables satisfies assumption \ref{symstrong}, we can use  lemma \ref{tournette} together with  the first statement of lemma \ref{lemma:silvermonkeyI} to obtain
\[
\Omega_c^{I} \subset \mathcal E \cap \mathcal N \cap \mathcal B^I_c \subset \arg \min \mathcal R,
\]
which is precisely statement (i) of the proposition.

We now prove statement (ii) of the proposition.
If the  latent variables satisfies assumption \ref{symstrong} and \ref{symconv} then  lemma \ref{tournette2} asserts that
\[
\Omega_c^{I} = \mathcal E \cap \mathcal N \cap \mathcal B^I_c 
\]
The set $\Omega_c^{I}$ is clearly not empty (because the set of equiangular matrices is not empty), and we may therefore  use  the second statement of lemma \ref{lemma:silvermonkeyI} to obtain that
$$
(W^\star, U^\star) \; \in \;  \mathcal E \cap \mathcal N \cap \mathcal B^I_c  \;  = \;  \Omega_c^{I}
$$
\end{proof}

\subsection{Determining the constant $c$} \label{sub:constants} 
The next lemma provides an explicit formula for the regularized risk of a network whose weights are in type-I collapse configuration with constant $c$. 

\begin{lemma} \label{lemma:valueI} 
Assume the latent variables satisfy assumption \ref{symstrong}.  If  the pair of weights  $(W,U)$ belongs to  $\Omega^{I}_{c}$, then
\begin{equation} \label{evaleval}
\mathcal R(W,U) =   \log \left( 1 -  \frac{K}{n_c^L}  +  \frac{K}{n_c^L} \Big( 1+ (n_c-1) e^{-\eta \, n_w c^2} \Big)^L \right) \; + \;  \lambda \, n_w c^2 
\end{equation}
where  $\eta =   \frac{n_c}{n_c-1} \sqrt{\frac{1}{n_w KL}} $.
\end{lemma}
From the above lemma it is clear that if the pair $(W,U) \in \Omega^I_c$  minimizes  $\mathcal R$, then the constant $c$ must minimize  the right hand side of \eqref{evaleval}.   
Therefore combining lemma  \ref{lemma:valueI} with proposition \ref{proposition:collapseI} concludes the proof of theorem  \ref{theorem:thm1}.

\paragraph{Remark} In the previous subsections, we only relied on relations (i), (ii) and (iii) of lemma \ref{lemma:sym} to prove collapse. Assumption \ref{symstrong} was never fully needed.
In this section however, in order to determine the specific values of the constant involved in the collapse, we will need  the actual combinatorial values  provided by assumption \ref{symstrong}.

The remainder of this section is devoted to the proof of lemma \ref{lemma:valueI}.
\begin{proof}[Proof of lemma \ref{lemma:valueI}]
Recall from \eqref{surf} that the unregularized risk can be expressed as
\[
 \mathcal R_0(W,U)  =  \frac{1}{K}  \sum_{k=1}^{K}   \sum _{\bx \in \data_k}  \log\left( 1 +  \sum_{j \neq k}   e^{ - \margin_{W,U}(\bx , j)  } \right)  \;   \mathcal D_{\bz_k}(\bx)
\]
We also recall that the set $\Omega^{I}_{c}$ is given by 
\begin{multline} \label{setOmegaI}
\Omega^{I}_{c} = \Big\{ (W, U) :  \text{There exist an equiangular matrix $\mathfrak F$ such that } \\   W = c\; \mathfrak F \; P \quad   \text{ and } \quad    \hat U = c \; \sqrt{\frac{n_w}{KL}} \;  \mathfrak F  \; Z    \Big\}
\end{multline}
 and that
$
P \hot(\bx) =  Z_k 
$
for all $\bx \in \data_k$ (see equation \eqref{Phot} from section \ref{section:notation}).
Consider two latent variables
$$
\bz_k = [\alpha_1, \ldots, \alpha_L] \quad \text{ and } \bz_j = [\alpha'_1, \ldots, \alpha'_L] 
$$
and assume $\bx$ is generated by $\bz_k$, meaning that  $\bx \in \data_k$. 
\begin{align*}
\margin_{W,U}(\bx , j) &= \dotprodbis{\hat U_{k} - \hat U_{j}, W \hot(\bx)} \\
& = c^2 \,  \sqrt{\frac{n_w}{KL}}  \;  \dotprodbis{\mathfrak F\; Z_{k} -\mathfrak F\; Z_{j}, \mathfrak F\; P \hot(\bx)} \\
%& =   \sqrt{\frac{n_c}{RL}}  \dotprodbis{CZ_{k} -CZ_s, CQ^T X_i} \\
& =    c^2\,  \sqrt{\frac{n_w}{KL}} \;   \dotprodbis{\mathfrak F\; Z_{k} -\mathfrak F\; Z_{j}, \mathfrak F\; Z_k}  \\
& =   c^2\,  \sqrt{\frac{n_w}{KL}}   \; \sum_{\ell =1}^L \dotprodbis{ \; \cpt_{\alpha_{\ell } } -\cpt_{\alpha'_{\ell } }  \;,\;  \cpt_{\alpha_{\ell} }  \; } \\
& = c^2\, \sqrt{\frac{n_w}{KL}}  \;\; \left( L -  \sum_{\ell =1}^L \dotprodbis{ \; \cpt_{\alpha'_{\ell } }  \;,\;  \cpt_{\alpha_{\ell} }  \; }     \right) 
\end{align*}
Since $\mathfrak f_1, \ldots, \mathfrak f_{n_c}$ are equiangular, we have 
$$
\sum_{\ell =1}^L \dotprodbis{ \; \cpt_{\alpha'_{\ell } }  \;,\;  \cpt_{\alpha_{\ell} }  \; }      =   \Big( L - \dist(\bz_j,\bz_{k}) \Big) - \frac{1}{n_c-1} \dist(\bz_j,\bz_{k}) = L - \frac{n_c}{n_c-1} \dist(\bz_j,\bz_{k}). 
$$
Therefore 
$$
\margin_{W,U}(\bx , j) = c^2\, \sqrt{\frac{n_w}{KL}}  \frac{n_c}{n_c-1} \dist(\bz_j,\bz_{k}) 
$$
Letting $\omega =  \sqrt{\frac{n_w}{KL}}  \frac{n_c}{n_c-1} $ we therefore obtain
\begin{align}
 \mathcal R_0(W,U)  &=  \frac{1}{K}  \sum_{k=1}^{K}   \sum _{\bx \in \data_k}  \log\left( 1 +  \sum_{j \neq k}   e^{ - \omega c^2 \dist(\bz_j,\bz_k) } \right)  \;   \mathcal D_{\bz_k}(\bx) \nonumber \\
 & = \frac{1}{K}  \sum_{k=1}^{K}   \log\left( 1 +  \sum_{j \neq k}   e^{ - \omega c^2 \dist(\bz_j,\bz_k) } \right)  \label{vavava}
\end{align}
where we have used the quantity inside the $\log$ does not depends on $\bx$.
We
proved in section \ref{section:sym} (see  equation \eqref{size_of_SSS})  that if the latent variables satisfy assumption \ref{symstrong}, then
$$
|S_r| = \frac{K}{n_c^L}  {L \choose r}  (n_c-1)^r
$$
Using this identity we obtain
 \begin{align*}
  \sum_{j \neq k}   e^{ - \omega c^2 \dist(\bz_j,\bz_k) } &  = \sum_{r=1}^L  \left| \left\{ j: \dist(\bz_j, \bz_k) = r \right\} \right|  \;  e^{- \omega c^2  r  }  \\
  & =    \frac{K}{n_c^L} \sum_{r=1}^L  { L \choose r} (n_c-1)^r \; e^{- \omega c^2  r  }  \\
  & =   -  \frac{K}{n_c^L}  +  \frac{K}{n_c^L} \sum_{r=0}^L  { L \choose r} (n_c-1)^r \;  e^{- \omega c^2  r  }  \\
   & =  -  \frac{K}{n_c^L}  +  \frac{K}{n_c^L} \Big( 1+ (n_c-1) e^{-\omega c^2} \Big)^L 
 \end{align*}
where we have used the binomial theorem to obtain the last equality. The above quantity does not depends on $k$, therefore \eqref{vavava} can be expressed as
$$
 \mathcal R_0(W,U) = \log \left( 1 -  \frac{K}{n_c^L}  +  \frac{K}{n_c^L} \Big( 1+ (n_c-1) e^{-\omega \, c^2} \Big)^L \right) 
$$
We then remark that the matrix   $\mathfrak F \; P$ has $n_w$  columns, and that each of these columns has norm $1$. Similarly,  the    $\mathfrak F  \; Z$ has $KL$ columns of length $1$. We therefore have
\[
\frac{1}{2} \left( \|W\|_F^2 + \|\hat U\|_F^2 \right)  = \frac{1}{2} \left(  c^2  \| \mathfrak F \; P\|_F^2 +  c^2 \; \frac{n_w}{KL}   \left\|  \mathfrak F  \; Z  \right\|_F^2 \right) = c^2 n_w.
\]
To conclude the proof we simply remark that $\omega = n_w \eta$.
\end{proof}

\section{Proof of theorem \ref{theorem:3} and its converse} \label{section:thm3}
In this section we prove theorem \ref{theorem:3} under assumption \ref{symstrong},  and its converse under assumptions \ref{symstrong} and \ref{symconv}. We start by recalling the definition of a type-II collapse configuration.

\begin{definition}[Type-II Collapse]   The weights $(W,U)$ of the network $h^*_{W,U}$ form a type-II collapse configuration if and only if the conditions 
\begin{enumerate}[topsep=0pt]
\item[\rm i)] $\vphi(\bw_{(\alpha,\beta)}) = \sqrt{d}    \,\cpt_\alpha$ \;  for all $(\alpha,\beta) \in \voc$.
\item[\rm ii)] There exists $c \ge 0$ so that $\bu_{k, \ell} =  c\,  \cpt_\alpha$ \; for all  $(k,\ell)$ satisfying $z_{k,\ell}=\alpha$ and all $\alpha \in \mathcal C$.
\end{enumerate}
hold for some collection $\cpt_1, \ldots, \cpt_{n_c} \in \real^d$ of mean-zero equiangular vectors.
\end{definition}

As in the previous section we will reformulate the above definition using matrix notations. Toward this aim we make the following definition: 

\begin{definition}(Mean-Zero Equiangular Matrices) 
\label{def:2ap} 
A matrix $\mathfrak F \in \real^{d \times n_c}$ is said to be a  mean-zero  equiangular matrix  if and only if the relations
$$
  \ones^T_{d} \; \mathfrak F  = 0, \qquad   \mathfrak F \, \ones_{n_c} = 0 \qquad \text{ and } \qquad \mathfrak F^T \mathfrak F = \frac{n_c}{n_c-1} \;I_{n_c} - \frac{1}{n_c-1} \; \ones_{n_c} \ones_{n_c}^T
$$
hold.
\end{definition}
Comparing the above definition with the definition of equiangular vectors provided in the main paper, we easily see that
$
\mathfrak F
$
is a mean-zero equiangular matrix  if and only if its columns  are  mean-zero equiangular vectors. 
Relations (i) and (ii) of  definition \ref{def:2ap}  can be conveniently expressed as
$$
\vphi(W) = \sqrt{d}\; \mathfrak F \; P \qquad \text{ and } \qquad \hat U = c \; \mathfrak F  \; Z 
$$
for some equiangular matrix $\mathfrak F$.
We  then set 
\begin{multline} \label{setOmegaII}
\Omega^{II}_{c} = \Big\{ (W, U) :  \text{There exist a mean-zero equiangular matrix $\mathfrak F$ such that } \\   \vphi(W) = \sqrt{d}\; \mathfrak F \; P \quad   \text{ and } \quad    \hat U = c \;  \mathfrak F  \; Z    \Big\}
\end{multline}
and note that $\Omega^{II}_{c}$ is simply  the set of weights $(W,U)$ which are in a type-II collapse configuration.  We now state the main theorem of this section.

\begin{theorem}  \label{theorem:thm3}
Assume the non-degenerate condition $\mu_\beta > 0$ holds. Let $\tau \ge 0$ denote the unique minimizer of the strictly convex function
\[
H^*(t) =  \log \left( 1 -  \frac{K}{n_c^L}  +  \frac{K}{n_c^L} \Big( 1+ (n_c-1) e^{-\eta^* t } \Big)^L \right) + \frac{\lambda}{2} t^2 \qquad  \text{where } \eta^* =    \frac{n_c}{n_c-1} \; \frac{1}{\sqrt{ KL/d}} 
\]

and let   $c = \tau / \sqrt{KL}$. 
Then we have the following:
\begin{enumerate}
\item[(i)] If the latent variables $\bz_1, \ldots, \bz_K$ are mutually distinct and satisfy  assumption \ref{symstrong}, then 
$$
 \Omega^{II}_{c}  \subset \arg \min \mathcal R^*
$$
\item[(ii)] If the latent variables $\bz_1, \ldots, \bz_K$ are mutually distinct and satisfy  assumptions \ref{symstrong} and \ref{symconv}, then 
$$
 \Omega^{II}_{c}  = \arg \min \mathcal R^*
$$
\end{enumerate}
\end{theorem}
Note that statement (i)   corresponds to theorem \ref{theorem:3} of the main paper, whereas statement (ii) can be viewed as its  converse. To prove  \ref{theorem:thm3} we will follow the same steps than in the previous section. 
 The main difference occurs in the study of the bilinear problem, as we will see in the next subsection.  We will assume 
\begin{equation*} 
\mu_\beta >0 
\end{equation*}
 everywhere in this section --- all  lemmas and propositions are proven under this assumption, even when not explicitly stated. 
 
 Before to go deeper in our study let us state a very simple lemma that expresses the regularized risk $\mathcal R^*$ associated with network $h^*$ in term of the function $\mathcal R_0$ defined by equation \eqref{calR0}.
 \begin{lemma} \label{lemma:simple} Given a pair of weights  $(W,U)$, we have
 \begin{equation} \label{flute88}
\mathcal R^*(W,U) =  \mathcal R_0\Big(\, \vphi(W) \, , \, U \,\Big) +  \frac{\lambda}{2}  \|U\|_F^2  
\end{equation}
 \end{lemma}
 \begin{proof}
Recall from section \ref{section:notation} that
\begin{align*}
& h_{W,U}(\bx) = U \;  \text{Vec} \left[  W \hot(\bx)   \right] \\
& h_{W,U}^*(\bx) = U \;  \text{Vec} \left[  \vphi\Big( W \hot(\bx) \Big)   \right]
\end{align*}
Note that since $\hot(\alpha,\beta)$ is a one hot vector, we obviously have that $ \vphi\left( W \hot(\alpha,\beta)\right) =  \vphi\left( W \right)  \hot(\alpha,\beta) $.
Therefore the the network $h^*$ and $h$ are related as follow:
$$
 h_{W,U}^*(\bx) = U \;  \text{Vec} \left[  \vphi\Big( W \hot(\bx) \Big)   \right] =  U \;  \text{Vec} \Big[  \vphi( W ) \, \hot(\bx)   \Big]  = h_{\vphi(W) , \,U}(\bx) 
$$
As a consequence, the  regularized risk associated with the network $h^*_{W,U}$ can be expressed as
 \begin{align*} 
\mathcal R^*(W,U)  & = \frac{1}{K}  \sum_{k=1}^{K}   \E_{  \;\;\bx \sim \mathcal D_{\bz_k}} \Big[  \ell(h^*_{W,U}(\bx) , k ) \Big]  + \frac{\lambda}{2}   \|U\|_F^2    \\
& = \frac{1}{K}  \sum_{k=1}^{K}   \E_{  \;\;\bx \sim \mathcal D_{\bz_k}} \Big[  \ell(h_{\vphi(W),U}(\bx) , k ) \Big]  + \frac{\lambda}{2}   \|U\|_F^2    \\
& = \mathcal R_0( \,\vphi(W) \, , \, U)  + \frac{\lambda}{2}   \|U\|_F^2
\end{align*}
where $\mathcal R_0$ is the unregularized risk defined in \eqref{calR0}.
\end{proof}

\subsection{The bilinear optimization problem} \label{sub:bilinearII}
Let 
$$
{\rm Range}(\vphi) = \{ V \in \real^{d \times n_w} :  \text{There exist $W \in \real^{d \times n_w}$ such that $V = \vphi(W)$ }   \}
$$
and consider the optimization problem
\begin{align}
& \text{maximize } \;    \dotprodbis{\hat U,  V Q^TZ }  \label{basea}  \\
& \text{subject to } \quad V \in  {\rm Range}(\vphi)   \quad \text{and} \quad   \|\hat U\|_F^2 = KL \,c^2   \label{baseb} %\quad \text{and } \quad \sum_k U_k =0
\end{align}
where the optimization variables are the matrix $V \in \real^{d \times n_w}$ and the matrix $\hat U \in \real^{d \times KL}$. 
\begin{lemma} \label{lemma:bilinearII} Assume the latent variables satisfy assumption \ref{symstrong}.
Then  $(V,U)$ is a solution of the optimization problem \eqref{basea} -- \eqref{baseb} if and only if it belongs to the set
\begin{align} \label{bigBB}
{\mathcal B}^{II}_{c} =\Big\{ (V , U) :  \text{There exist a matrix $F \in  \mathcal F$    such that  }     V = \sqrt{d} \,   F  P \text{ and }  \hat U  =  c \;   F   Z  \Big\}
\end{align}
where  $\mathcal F$  denotes the   set of matrices whose columns have unit length and mean zero, that is
$$
\mathcal F = \{F \in  \real^{d \times n_c} :  \ones_d^T F = 0 \text{ and  the columns of $F$ have unit length}   \}.
$$
\end{lemma}
The remainder of this subsection is devoted to the proof of the above lemma.

We start by showing that all $(V,U) \in \mathcal B_c^{II}$ have same objective values and satisfy the constraints.
\begin{claim} If $(V,U) \in \mathcal B_c^{II}$, then 
$$
V \in  {\rm Range}(\vphi)  \quad, \quad    \|\hat U\|_F^2 = KL \,c^2 , \quad \text{ and } \quad  \dotprodbis{\hat U,  V Q^TZ } = c  \sqrt{d} \, KL
$$
\end{claim}
\begin{proof}  Assume $(V,U) \in \mathcal B_c^{II}$.
Since the columns of $P$ are one hot vectors in $\real^{n_c}$, the columns of $FP$ have unit length and mean zero. Therefore the columns of $V$ have norm equal to $\sqrt{d}$ and mean zero. Therefore $V \in  {\rm Range}(\vphi)  $.

 Using  $ZZ^T = \frac{KL}{n_c} I$ from lemma \ref{lemma:sym},  together with the fact that $\|F\|_F^2 = n_c$ since its columns have unit length,  we obtain 
\begin{align} \label{FZnorm}
\| FZ  \|^2_F =  \dotprodbis{  FZ ,FZ }  =  \dotprodbis{  FZZ^T ,F }  =   \left(   \frac{KL}{n_c} \right) \| F \|_F^2  = KL
\end{align}
As a consequence we have 
$
\| \hat U \|^2_F  = c^2 \, KL
$. Finally, note that
$$
PQ^T = I_{n_c}
$$
as can clearly be seen from formulas \eqref{P} and \eqref{matrixQ}. We therefore have
\begin{align*}
 \dotprodbis{\hat U,  V Q^TZ } = c \sqrt{d}  \dotprodbis{ FZ,  FP Q^TZ } =   c \sqrt{d}  \dotprodbis{ FZ,  FZ } =  c  \sqrt{d} \, KL
\end{align*}

\end{proof}

We then prove that
\begin{claim} If $(V,\hat U)$ is a solution of \eqref{basea} -- \eqref{baseb}, then $(V,\hat U) \in \mathcal B_c^{II}$. 
\end{claim}
Note that according to the first claim, all $(V,\hat U) \in \mathcal B_c^{II}$ have same objective value, and therefore, according to the above claim, they must all be maximizer.  As a consequence,  proving the above claim will conclude the proof of lemma \ref{lemma:bilinearII}. 

\begin{proof}[Proof of the claim]
Maximizing  \eqref{basea} -- \eqref{baseb} over  $\hat U$ first gives
\begin{equation} \label{getUU}
\hat U =  {c}{\sqrt{KL}} \;  \frac{V Q^TZ }{\|V Q^TZ\|_F }
\end{equation}
and therefore the optimization problem reduces to 
\begin{align}
& \text{maximize } \; \|  V Q^T Z  \|^2_F     \label{base5500}  \\
& \text{subject to } \quad   V  \in  {\rm Range}(\vphi)  \label{base5600}
\end{align}
Using the fact that  $ZZ^T = \frac{KL}{n_c} I$ we then get
\begin{equation}
 \|  V Q^T Z  \|^2_F =\dotprodbis{ V Q^T Z,  V Q^T Z }  = \dotprodbis{  V Q^T ZZ^T,  V Q^T } =   \frac{KL}{n_c}  \| V  Q^T \|_F^2
\end{equation}
and so the problem further reduces to 
\begin{align}
& \text{maximize } \; \|  V Q^T   \|^2_F  \label{base55000}  \\
& \text{subject to } \quad   V  \in  {\rm Range}(\vphi)  \label{base56000}
\end{align}
%We remark that any $V$ of the form $V=\sqrt{d} FP $ with $F\in \mathcal F$ satisfies the constraints \ref{base56000} and have objective value
%\begin{align*}
%\|  V Q^T   \|^2_F = d \| FP Q^T   \|^2_F =  d \| F  \|^2_F  = d n_c 
%\end{align*}
%where we have used the fact that $PQ^T = I_{n_c}$ and $\| F  \|^2_F = n_c$.  Therefore any solution of \eqref{base55000} --  \eqref{base56000} must have objective value at least equal to 
% $d n_c$. 
Let us define
$$
\bv_{(\alpha,\beta)} : = V \hot(\alpha,\beta)
$$
In other words $\bv_{(\alpha,\beta)}$ is the $j^{th}$ column of $V$,  where $j = (\alpha-1) s_c + \beta$. 
The KKT conditions for the optimization problem \eqref{base55000} --  \eqref{base56000} then   amount to solving the system
\begin{align}
V Q^T Q &= V D_{\nu} + \mathbf{1}_{d} \;  \boldsymbol{\lambda}^T \\
\dotprod{  \bv_{(\alpha,\beta)} ,   \mathbf{1}_d} &= 0  \qquad \text{for all } (\alpha,\beta) \in \voc  \label{constraint-a} \\
\| \bv_{(\alpha,\beta)} \|^2 &= d   \qquad \text{for all } (\alpha,\beta) \in \voc \label{constraint-b}
\end{align}
for $D_{\nu}$ some $n_w \times n_w$ diagonal matrix of Lagrange multipliers for the  constraint \eqref{constraint-b} and $\boldsymbol{\lambda} \in \mathbb{R}^{n_w}$ a vector of Lagrange multipliers for the mean zero constraints. Left multiplying the first equation by $\mathbf{1}^{T}_{d}$ and using the second shows $\boldsymbol{\lambda} = \mathbf{0}_{n_w}$, and so it proves equivalent to find solutions of the reduced system
\begin{align}
V Q^T Q &= V D_{\nu} \label{bnbn}  \\
\dotprod{  \bv_{(\alpha,\beta)} ,   \mathbf{1}_d} &= 0  \qquad \text{for all } (\alpha,\beta) \in \voc  \label{constraint-aa} \\
\| \bv_{(\alpha,\beta)} \|^2 &= d   \qquad \text{for all } (\alpha,\beta) \in \voc \label{constraint-bb}
\end{align}
instead. 
%Assume  that $(V,\hat U)$ is a solution of the original optimization problem \eqref{basea} -- \eqref{baseb}. Then, according to the above discussion,  $V$ must satisfy \eqref{bnbn}, \eqref{constraint-aa}, \eqref{constraint-bb} and have objective value  $ \|  V Q^T   \|^2_F  \ge d n_c$.
Recalling  the identity $Q \, \hot(\alpha,\beta) = \mu_\beta \be_\alpha$ (see \eqref{Qhot} in section \ref{section:notation})  we obtain
$$
Q^T Q \, \hot(\alpha,\beta) = \mu_\beta \;  Q^T \; \be_\alpha
$$
and so right multiplying \eqref{bnbn}  by  $\hot(\alpha,\beta)$  gives
$$
V  Q^T \; \be_\alpha = \frac{\nu(\alpha,\,\beta)}{\mu_\beta} \,  \bv_{(\alpha,\,\beta)} \qquad \text{for all } (\alpha,\beta) \in \voc
$$
where  we have denoted by $\nu(\alpha,\beta)$  the Lagrange multiplier corresponding to the   constraint  \eqref{constraint-bb}. 
Define the support sets
$$
\Xi_\alpha := \left\{ \beta \in [s_c] : \nu(\alpha,\,\beta) \neq 0 \right\} \qquad \text{and} \qquad \Xi := \left\{ \alpha : \Xi_\alpha \neq \emptyset\right\}
$$
of the Lagrange multipliers. If $\alpha \in \Xi$ then imposing the norm constraint \eqref{constraint-bb} gives
$$
\|V  Q^T \; \be_\alpha\| = \frac{\nu(\alpha,\,\beta)}{\mu_\beta}  \sqrt{d},
$$
and so $\|V  Q^T \; \be_\alpha\|> 0$ if $\alpha \in \Xi$ since $\nu(\alpha,\, \beta) > 0$ for some $ \beta \in [s_c]$ by definition. This implies that the relation
$$
 \bv_{(\alpha,\,\beta)} =  \sqrt{d} \; \frac{V  Q^T \; \be_\alpha}{\|V  Q^T \; \be_\alpha\|} \qquad \text{for all} \qquad (\alpha,\,\beta) \in \Xi \times [s_c]
$$
must hold. As a consequence there exist mean-zero, unit length vectors ${\bf f}_1,\,\ldots,\,{\bf f}_{n_c}$ (namely the normalized $V  Q^T \; \be_\alpha$) so that
$$
\mathbf{v}_{(\alpha,\,\beta)} = \sqrt{d} \; \mathbf{f}_{\alpha} \qquad 
$$
holds for all pairs $(\alpha, \beta)$ with $\alpha \in \Xi$.  Taking a look at \eqref{matrixQ}, we easily see that its $\alpha^{th}$ row of the matrix $Q$ can be written as   $Q^T \; \be_\alpha = \sum_{\beta} \mu_\beta \hot(\alpha,\beta)$, and therefore
$$
V  Q^T \; \be_\alpha = \sum_{\beta \in [s_c]} \mu_\beta V \hot(\alpha,\beta) = \sum_{\beta \in [s_c]} \mu_\beta  \bv_{(\alpha,\,\beta)} = \sqrt{d} \; \mathbf{f}_\alpha\left(\sum_{\beta \in [s_c]} \mu_\beta\right) = \sqrt{d} \; \mathbf{f}_\alpha
$$
holds as well. If $\alpha \notin \Xi$ then $V  Q^T \; \be_\alpha = \mathbf{0}$ since the corresponding Lagrange multiplier vanishes. It therefore follows that
$$
\|VQ^T \|_F^2 = \sum_{\alpha \in [n_c]} \| V  Q^T  \be_\alpha\|^2 = d \sum_{\alpha \in \Xi} \|\mathbf{f}_\alpha\|^2 = d \; |\Xi|
$$
and so global maximizers of \eqref{base55000} -- \eqref{base56000}   must have full support. In other words, there exist mean-zero, unit-length vectors $\mathbf{f}_1,\,\ldots,\,\mathbf{f}_{n_c}$ so that
\begin{equation} \label{arara}
\mathbf{v}_{(\alpha,\, \beta)} = \sqrt{d} \; \mathbf{f}_\alpha 
\end{equation}
holds.  %Any such collection of vectors has objective value $\| V Q^T\|^2_F =  d\, n_c$, and so any such collection of vectors defines a global optimum.
Equivalently 
$
V = \sqrt{d} \, FP
$ for some $F \in \mathcal F$. We then recover $\hat U$ using \eqref{getUU}.
\begin{align}
\hat U = {c}{\sqrt{KL}} \;  \frac{V Q^TZ }{\|V Q^TZ\|_F } =  {c}{\sqrt{KL}} \;  \frac{FP Q^TZ }{\| FP Q^TZ\|_F } =   {c}{\sqrt{KL}} \;  \frac{ FZ }{\| FZ\|_F } 
\end{align}
where we have used the fact that $PQ^T = I_{n_c}$. To conclude the proof, we use the fact $\| FZ\|_F = \sqrt{KL}$, as was shown in  \eqref{FZnorm}.
\end{proof}

\subsection{Proof of collapse}

Recall from lemma \ref{lemma:simple} that the regularized risk associated with the network $h^*_{W,U}$ can be expressed as
\begin{equation} \label{flute55}
\mathcal R^*(W,U) =  \mathcal R_0\Big( \,\vphi(W) \,, \,U \, \Big) +  \frac{\lambda}{2}  \|U\|_F^2  
\end{equation}
and recall that the set of weights in type-II collapse configuration is
\begin{multline} \label{OmegaII}
\Omega^{II}_{c} = \Big\{ (W, U) :  \text{There exist a mean-zero equiangular matrix $\mathfrak F$ such that } \\   \vphi(W) = \sqrt{d}\; \mathfrak F \; P \quad   \text{ and } \quad    \hat U = c \;  \mathfrak F  \; Z    \Big\}
\end{multline}

This subsection is devoted to the proof of the following proposition.

\

\begin{proposition}  \label{proposition:collapseII} We have the following:
\begin{enumerate}
\item[(i)] If the latent variables $\bz_1, \ldots, \bz_K$ are mutually distinct and satisfy  assumption \ref{symstrong}, then there exists $c\in \real$ such that
$$
 \Omega^{II}_{c}  \subset \arg \min \mathcal R^*
$$
\item[(ii)] If the latent variables $\bz_1, \ldots, \bz_K$ are mutually distinct and satisfy  assumptions \ref{symstrong} and \ref{symconv}, then any $(W, U)$ that minimizes $\mathcal R^*$ must belong to $ \Omega^{II}_{c}$ for some $c\in \real$. 
\end{enumerate}
\end{proposition}

As in the previous section, we have the following lemma.
\begin{lemma} \label{lemma:inN2} Any  global minimizer of \eqref{flute55} must belong to $\mathcal N$.
\end{lemma}
The proof is identical to the proof of lemma   \ref{lemma:inN}.
The next lemma bring together the bilinear optimization problem from subsection \ref{sub:bilinearII} and  the sharp lower bound on the unregularized risk that we derived in section \ref{section:lowerbound}.

\begin{lemma} \label{lemma:silvermonkeyII}  Assume the latent variables satisfy assumption \ref{symstrong}.
Assume also that $(W^\star,U^\star)$ is a global minimizer of  \eqref{flute55} and let $c\in \real$ be such that
$$
 \|U^\star\|_F^2   =  KL \, c^2 
$$
 The the following hold:
 \begin{enumerate}
\item[(i)] Any $(W,U)$ that satisfies
$$
(\vphi(W),U) \in  \mathcal N \cap \mathcal E \cap \mathcal B^{II}_c
$$
is also a global minimizer of $\mathcal R^*$.
\item[(ii)] If $\mathcal N \cap \mathcal E \cap \mathcal B^{II}_{c} \neq \emptyset$, then 
 $$(\vphi(W^\star), U^\star) \in \mathcal N \cap \mathcal E \cap \mathcal B^{II}_{c}$$
\end{enumerate}
 \end{lemma}

\begin{proof}
Recall from theorem \ref{theorem:buffalo} that
\begin{align}
\mathcal R_0(V,U) &=  g \Big( -   \dotprodbis{\hat U,  V Q^TZ }  \Big) \qquad \text{ for all } (V,U) \in \mathcal N \cap \mathcal E \\
\mathcal R_0(V,U) &>  g \Big( -   \dotprodbis{\hat U,  V Q^TZ }  \Big) \qquad \text{ for all } (V,U) \in \mathcal N \cap \mathcal E^c \label{toto}
\end{align}
We start by proving (i). Define $ V^\star = \vphi( W^\star)$, and  assume that $U, V, W$  are such that $\vphi(W) = V$ and  $(V, U) \in \mathcal N \cap \mathcal E \cap \mathcal B_c$.  Then we have 
\begin{align*}
\mathcal R_0(\vphi(W^\star),U^\star) &= \mathcal R_0(V^\star,U^\star) \\
& \ge g\left(  - \dotprodbis{U^\star,  V^\star QZ }\right)  & \text{[because $(V^\star,U^\star) \in \mathcal N$ ]} \\
&\ge g\left(  - \dotprodbis{U,  V QZ }\right)  &  \text{[because $(V,U) \in \mathcal B^{II}_{c} $   ]} \\
 & = \mathcal R_0(V,U)  & \text{[because $(V,U) \in \mathcal N \cap \mathcal E$ ]} \\
 & =  \mathcal R_0(\vphi(W),U) 
\end{align*}
Since $\|U\|_F^2 =  KL \, c^2  = \|U^\star\|_F^2$, we have  $\mathcal R^*(W,U) \le \mathcal R^*(W^\star,U^\star) $ and therefore $(W,U)$ is a minimizer. 

We now prove (ii) by contradiction.
Suppose that $( \vphi(W^\star), U^\star) \notin  \mathcal N \cap \mathcal E \cap \mathcal B^{II}_c$. This must mean that  
$$ ( \vphi(W^\star), U^\star) \notin \mathcal E \cap \mathcal B^{II}_c $$
since it clearly belongs to $\mathcal N$.
If ($ \vphi(W^\star), U^\star) \notin \mathcal E$ then the first inequality in the above computation is strict according to \eqref{toto}. If ($ \vphi(W^\star), U^\star) \notin  \mathcal B^{II}_c$ then the second inequality is strict because $g$ is strictly increasing.
\end{proof}

 The next two 
lemmas shows that the set $ \mathcal E \cap \mathcal N \cap \mathcal B^{II}_c$ is closely related to the set of collapsed configurations  $\Omega_c^{II}$. In order to states these lemmas, the following definition will prove convenient
\begin{multline} \label{barOmegaII}
\overline{\Omega}^{II}_{c} = \Big\{ (V, U) :  \text{There exist a mean-zero equiangular matrix $\mathfrak F$ such that } \\   V = \sqrt{d}\; \mathfrak F \; P \quad   \text{ and } \quad    \hat U = c \;  \mathfrak F  \; Z    \Big\}
\end{multline}
Note that  $(W,U) \in {\Omega}^{II}_{c} $ if and only if  $(\vphi(W),U) \in  \overline{\Omega}^{II}_{c} $. Also, in light of \eqref{bigBB}, the inclusion 
$$
\overline{\Omega}^{II}_{c} \subset  \mathcal  B^{II}_c
$$
is obvious. We now prove the following lemma.
\begin{lemma} \label{tournette-other}
 If the latent variables satisfy the  symmetry assumption \ref{symstrong}, then
\[
\overline{\Omega}_c^{II} \subset \mathcal E \cap \mathcal N \cap \mathcal B^{II}_c
\]
\end{lemma}
\begin{proof} The proof is almost identical to the one of lemma \ref{tournette}. We repeat it for completeness. 
We already know that $\overline{\Omega}^{II}_{c} \subset \mathcal B^{II}_{c}$.  
We the show that $\overline{\Omega}^{II}_{c}  \subset  \mathcal E$.
Suppose 
  $(V,U) \in \overline{\Omega}^{II}_c$. Then there exists a mean-zero equiangular  matrix $\mathfrak F  \in \real^{d \times n_c}$  such that
\begin{equation*}
   V = \sqrt{d} \;   \mathfrak  F \;  P \qquad  \text{ and }  \qquad  \hat U  = c \;    \mathfrak F \;  Z  
\end{equation*}
Recall from \eqref{Phot} that
$
P \hot(\bx) =  Z_k $ for all $\bx \in \data_k$.
Consider two latent variables
$$
\bz_k = [\alpha_1, \ldots, \alpha_L] \quad \text{ and } \bz_j = [\alpha'_1, \ldots, \alpha'_L] 
$$
and assume $\bx$ is generated by $\bz_k$, meaning that  $\bx \in \data_k$. We then have 
\begin{align*}
\margin_{V,U}(\bx , j) &= \dotprodbis{\hat U_{k} - \hat U_{j}, V \hot(\bx)} \\
& = c \sqrt{d}\;  \dotprodbis{\mathfrak F\; Z_{k} -\mathfrak F\; Z_{j}, \mathfrak F\; P \hot(\bx)} \\
%& =   \sqrt{\frac{n_c}{RL}}  \dotprodbis{CZ_{k} -CZ_s, CQ^T X_i} \\
& =    c\, \sqrt{d}\;   \dotprodbis{\mathfrak F\; Z_{k} -\mathfrak F\; Z_{j}, \mathfrak F\; Z_k}  \\
& =   c\, \sqrt{d}  \; \sum_{\ell =1}^L \dotprodbis{ \; \cpt_{\alpha_{\ell } } -\cpt_{\alpha'_{\ell } }  \;,\;  \cpt_{\alpha_{\ell} }  \; } \\
& =  c\, \sqrt{d} \;\;   \dist(\bz_j,\bz_{k})
\end{align*}
From the above computation it is clear that the margin  only depends on  $\dist(\bz_j,\bz_{k})$,  and therefore $(V,U)$ satisfies the equimargin property.

Finally we show that $\overline\Omega^{II}_{c} \subset  \mathcal N$.  Suppose 
  $(V,U) \in \overline \Omega^{II}_c$.  Using the identity
$
\sum_{k=1}^K Z_k = \frac{K}{n_c} \ones_{n_c} \ones^T_L
$
we obtain  
\begin{align*}
\sum_{k=1}^K \hat U_k =   c \sum_{k=1}^K    \mathfrak F \; Z_k =  c \;\;  \frac{K}{n_c}  \mathfrak F \; \ones_{n_c} \ones^T_L = 0
\end{align*}
where we have used the fact that  $\mathfrak F \; \ones_{n_c} = 0$.
\end{proof}

Finally, we have the following lemma.

\begin{lemma} \label{tournette-other2} If  the latent variables satisfy  assumptions \ref{symstrong} and \ref{symconv},  then 
\[
\overline \Omega^{II}_{c}   =  \mathcal E \cap \mathcal N \cap \mathcal B^{II}_{c} 
\]
\end{lemma}
\begin{proof} The proof, again, is very similar to the one of lemma \ref{tournette2}.
From the previous lemma we know that $
\overline \Omega^{II}_{c}   \subset \mathcal E \cap \mathcal N \cap \mathcal B^{II}_{c}  
$ so we need to  show that
$$
 \mathcal E \cap \mathcal N \cap \mathcal B^{II}_{c} \subset \overline \Omega^{II}_{c}. 
$$
Let $ (V,U) \in \mathcal E \cap \mathcal N \cap \mathcal B^{II}_{c} $. Since $(V,U)$  belongs to $\mathcal B^{II}_c$, there exists a matrix $F \in \real^{d \times n_c}$ whose columns have unit length and mean $0$  such that
\begin{equation*}
   V = \sqrt{d} \;   F \;  P \qquad  \text{ and }  \qquad  U  = c \;    F \;  Z  
\end{equation*}
Our goal is to show that $F$ is a mean-zero equiangular matrix, meaning that it satisfies the three relations
\begin{equation} \label{lac_leman55}
 \ones^T_{n_c} \,  F = 0 , \qquad  F \, \ones_{n_c} = 0 \qquad \text{ and } \qquad  F^T  F = \frac{n_c}{n_c-1} \;I_{n_c} - \frac{1}{n_c-1} \; \ones_{n_c} \ones_{n_c}^T.
\end{equation}

We already know that the first relation is satisfied since the columns of $F$ have mean $0$.
The second relation is easily obtained. Indeed, using the fact that $ (V,U) \in \mathcal N$ together with the identity
$
\sum_{k=1}^K Z_k = \frac{K}{n_c} \ones_{n_c} \ones^T_L
$
(which hold due to lemma \ref{lemma:sym}),  we obtain  
$$
0=\sum_{k=0}^K U_k =  c'  \sum_{k=0}^K F Z_k  =  c \frac{K}{n_c} F \ones_{n_c} \ones_L^T. 
$$
which implies
 $F \ones_{n_c} =0$. 
 
We now prove the third equality of \eqref{lac_leman55}. Assume that  $\bx \in \data_k$.  Using the fact that $
P \hot(\bx) =  Z_k
$ together with \eqref{montblanc}, we obtain
\begin{align}
\margin_{V,U}(\bx , j) &= \dotprodbis{\hat U_{k} - \hat U_{j}, V \hot(\bx)} \nonumber \\
& = c\,  \sqrt{d}  \;  \dotprodbis{F\; Z_{k} - F\; Z_{j},  F\; P \hot(\bx)}  \nonumber \\
%& =   \sqrt{\frac{n_c}{RL}}  \dotprodbis{CZ_{k} -CZ_s, CQ^T X_i} \\
& =    c\, \sqrt{d}   \dotprodbis{F\; Z_{k} -F\; Z_{j}, F\; Z_k}   \nonumber \\
& =  c\, \sqrt{d}    \dotprodbis{F^TF(Z_{k} -Z_{j}), Z_k}  \nonumber \\
& =  c\, \sqrt{d}    \dotprodbis{\; F^TF \; , \; \Gamma^{(k,j)} \;} \label{clement55}
%& = c\,c'  \sum_{\ell =1}^L \dotprodbis{ \; {\bf f}_{\alpha_{\ell } } -{\bf f}_{\alpha'_{\ell } }  \;,\;  {\bf f}_{\alpha_{\ell} }  \; } \\
\end{align}
Since $(V,U) \in \mathcal E$, the margins must only depend on the distance between the latent variables.  Due to \eqref{clement55},  we can be express this as 
\begin{equation*} 
\dotprodbis{ \; F^TF \; , \;   \Gamma^{(j,k)}} =  \dotprodbis{ \; F^TF \; ,  \; \Gamma^{(j',k')}  }  \qquad   \forall j,k,j',k' \in [K]  \text{ s.t. } \dist(\bz_{j},\bz_{k}) = \dist(\bz_{j'},\bz_{k'})
\end{equation*}
Since the  $F^TF$ is clearly positive semi-definite, we may then use assumption  \ref{symconv} to conclude that
$
F^TF \in \mathcal A
$. Recalling definition \eqref{calA} of the set $\mathcal A$, we therefore have
\begin{equation} \label{lac_annecy55}
F^T F =  a  \;I_{n_c} + b\; \ones_{n_c} \ones_{n_c}^T
\end{equation}
for some $a,b \in \real $. To conclude our proof, we need to show that
\begin{equation} \label{sol999}
a = \frac{n_c}{n_c-1} \qquad \text{and} \qquad b = -\frac{1}{n_c-1}.
\end{equation}

Combining \eqref{lac_annecy55} with the first equality of \eqref{lac_leman55},  we obtain
\begin{equation*} 
0 = F^T F\,  \ones_{n_c} = a  \;\ones_{n_c} + b\; \ones_{n_c} \ones_{n_c}^T \ones_{n_c} = (a+b n_c) \ones_{n_c}
\end{equation*}
 Since the columns of $F$ have unit length,  the diagonal entries of $F^TF$ must all be equal to $1$, and therefore \eqref{lac_annecy55} implies that  $a+b=1$.
  The constants $a,b \in \real$, according must therefore solve the system
 $$
 \begin{cases}
 a+ bn_c &= 0 \\
 a + b &= 1
 \end{cases}
 $$
 and one can easily check that the solution of this system is precisely given by \eqref{sol999}.
\end{proof}

We conlude this subsection by proving proposition \ref{proposition:collapseII}.

\begin{proof}[Proof of Proposition  \ref{proposition:collapseII}]  Let $(W^\star, U^\star)$ be a global minimizer of $\mathcal R$ and let $c \in \real$ be such that
$$
 \|U^\star\|_F^2   = KL \, c^2 
$$

 We first prove statement (i) of the proposition.
 If the  latent variables satisfies assumption \ref{symstrong}  then  lemma \ref{tournette-other} asserts that
  \[
\overline \Omega_c^{II} \subset \mathcal E \cap \mathcal N \cap \mathcal B^{II}_c 
\]
Assume  $(W,U) \in \Omega_c^{II}$. This implies that  $(\vphi(W),U) \in \overline\Omega_c^{II}$, and 
and therefore $(\vphi(W),U) \in \mathcal E \cap \mathcal N \cap \mathcal B^{II}_c  $. We can then use lemma \ref{lemma:silvermonkeyII} to conclude that $(W,U)$ is a global minimizer of $\mathcal R^*$.

We now prove statement (ii) of the proposition.
If the  latent variables satisfies assumption \ref{symstrong} and \ref{symconv} then  lemma \ref{tournette-other2} asserts that
\[
\overline\Omega_c^{II} = \mathcal E \cap \mathcal N \cap \mathcal B^{II}_c 
\]
The set $\overline \Omega_c^{II}$ is clearly not empty (because the set of mean-zero equiangular matrices is not empty), and we may therefore  use  the second statement of lemma \ref{lemma:silvermonkeyII} to obtain that
$$
(\vphi(W^\star), U^\star) \; \in \;  \mathcal E \cap \mathcal N \cap \mathcal B^{II}_c  \;  = \;  \overline\Omega_c^{II}
$$
which in turn implies  $(W^\star, U^\star)  \in   \Omega_c^{II}$.
\end{proof}

\subsection{Determining the constant $c$} \label{sub:constants-LN}

The next lemma provides an explicit formula for the regularized risk of a network $h^*_{W,U}$  whose weights are in type-II collapse configuration with constant $c$. 

\begin{lemma} \label{lemma:value-LN} Assume the latent variables satisfy assumption \ref{symstrong}.  If  the pair of weights  $(W,U)$ belongs to  $\Omega^{II}_{c}$, then
\begin{equation} \label{evaleval55}
\mathcal R^*(W,U) =   \log \left( 1 -  \frac{K}{n_c^L}  +  \frac{K}{n_c^L} \Big( 1+ (n_c-1) e^{-   \eta^* \sqrt{KL}  \,c} \Big)^L \right) \; + \;   \frac{\lambda}{2}  \left( \sqrt{KL} \, c \right)^2
\end{equation}
where  $\eta^* =   \frac{n_c}{n_c-1} \sqrt{\frac{d}{KL}} $.
\end{lemma}
Combining lemma  \ref{lemma:value-LN} with proposition \ref{proposition:collapseII} concludes the proof of theorem  \ref{theorem:thm3}.

\begin{proof}[Proof of lemma \ref{lemma:value-LN}]
We recall that
\[
 \mathcal R_0(W,U)  =  \frac{1}{K}  \sum_{k=1}^{K}   \sum _{\bx \in \data_k}  \log\left( 1 +  \sum_{j \neq k}   e^{ - \margin_{W,U}(\bx , j)  } \right)  \;   \mathcal D_{\bz_k}(\bx)
\]
and 
\begin{multline} \label{setOmegaII}
\Omega^{II}_{c} = \Big\{ (W, U) :  \text{There exist a mean-zero equiangular matrix $\mathfrak F$ such that } \\   \vphi(W) = \sqrt{d}\; \mathfrak F \; P \quad   \text{ and } \quad    \hat U = c \;  \mathfrak F  \; Z    \Big\}
\end{multline}
Consider two latent variables
$$
\bz_k = [\alpha_1, \ldots, \alpha_L] \quad \text{ and } \bz_j = [\alpha'_1, \ldots, \alpha'_L] 
$$
and assume  $\bx \in \data_k$. Using the identity 
$
P \hot(\bx) =  Z_k
$
we then  obtain  
\begin{align*}
\margin_{\vphi(W),U}(\bx , j) &= \dotprodbis{\hat U_{k} - \hat U_{j}, \vphi(W) \hot(\bx)} \\
& =  c \sqrt{d} \;  \dotprodbis{\mathfrak F\; Z_{k} -\mathfrak F\; Z_{j}, \mathfrak F\; P \hot(\bx)} \\
%& =   \sqrt{\frac{n_c}{RL}}  \dotprodbis{CZ_{k} -CZ_s, CQ^T X_i} \\
& =    c \sqrt{d} \;   \dotprodbis{\mathfrak F\; Z_{k} -\mathfrak F\; Z_{j}, \mathfrak F\; Z_k}  \\
& =   c \sqrt{d}   \; \sum_{\ell =1}^L \dotprodbis{ \; \cpt_{\alpha_{\ell } } -\cpt_{\alpha'_{\ell } }  \;,\;  \cpt_{\alpha_{\ell} }  \; } \\
& =  c \sqrt{d} \;\; \left( L -  \sum_{\ell =1}^L \dotprodbis{ \; \cpt_{\alpha'_{\ell } }  \;,\;  \cpt_{\alpha_{\ell} }  \; }     \right) \\
& =  c \sqrt{d}  \frac{n_c}{n_c-1} \dist(\bz_j,\bz_{k}) 
\end{align*}
Letting $\omega^* =  \sqrt{d}  \frac{n_c}{n_c-1} $ we therefore obtain
\begin{align}
 \mathcal R_0(W,U)  &=  \frac{1}{K}  \sum_{k=1}^{K}   \sum _{\bx \in \data_k}  \log\left( 1 +  \sum_{j \neq k}   e^{ -  c \,\omega^* \,  \dist(\bz_j,\bz_k) } \right)  \;   \mathcal D_{\bz_k}(\bx) \nonumber \\
 & = \frac{1}{K}  \sum_{k=1}^{K}   \log\left( 1 +  \sum_{j \neq k}   e^{ -  c\, \omega^* \,  \dist(\bz_j,\bz_k) } \right)  \label{vavava55}
\end{align}
where we have used the quantity inside the $\log$ does not depends on $\bx$.
Using the identity 
$
|S_r| = \frac{K}{n_c^L}  {L \choose r}  (n_c-1)^r
$
 we then obtain  obtain
 \begin{align*}
  \sum_{j \neq k}   e^{ -  c\,\omega \,  \dist(\bz_j,\bz_k) } &  = \sum_{r=1}^L  \left| \left\{ j: \dist(\bz_j, \bz_k) = r \right\} \right|  \;  e^{-  c\,\omega^* \,   r  }  \\
  & =    \frac{K}{n_c^L} \sum_{r=1}^L  { L \choose r} (n_c-1)^r \; e^{-  c\,\omega^* \,   r  }  \\
  & =   -  \frac{K}{n_c^L}  +  \frac{K}{n_c^L} \sum_{r=0}^L  { L \choose r} (n_c-1)^r \;  e^{-  c\,\omega^* \,   r  }  \\
   & =  -  \frac{K}{n_c^L}  +  \frac{K}{n_c^L} \Big( 1+ (n_c-1) e^{- c\,\omega^* \, } \Big)^L 
 \end{align*}
where we have used the binomial theorem to obtain the last equality. The above quantity does not depends on $k$, therefore \eqref{vavava55} can be expressed as
$$
 \mathcal R_0(W,U) = \log \left( 1 -  \frac{K}{n_c^L}  +  \frac{K}{n_c^L} \Big( 1+ (n_c-1) e^{-  c\,\omega^*} \Big)^L \right) 
$$
We then remark that the matrix  $\mathfrak F  \; Z $ has $KL$ columns, and that each of these columns has norm $1$. We therefore have
$$
\|\hat U\|^2_{F} = \| c\, \mathfrak F \, Z \|_F^2 = c^2  KL \qquad  \text{for all } (W,U) \in \Omega^{II}_{c}
$$
To conclude the proof we simply note that $\omega^* = \sqrt{KL} \, \eta^* $.
\end{proof}

\section{Proof of theorem \ref{theorem:2}} \label{section:thm2}

This section is devoted to the proof of theorem \ref{theorem:2} from the main paper, which we recall below for convenience.
\begin{theorem2}[Directional Collapse of $h$]
Assume $K=n_c^L$ and $\{\bz_1, \ldots, \bz_K\} = \mathcal Z$.
 Assume also that  the regularization parameter $\lambda$ satisfies 
 \begin{equation}
 \lambda^2<  \frac{L}{n_c^{L+1}} \sum_{\beta=1}^{s_c} \mu_\beta^2 \label{lambdabound}
\end{equation}
Finally, assume that   $(W,\,U)$ is in a type-III collapse configuration for some constants  $c,r_1, \ldots, r_{s_c} \ge 0$. Then $(W,U)$ is a critical point of  $\risk$ if and only if $(c,r_1, \ldots, r_{s_c})$ solve the system 
\begin{align}
 & \frac{\lambda}{L} \; \frac{r_\beta}{c}  \left(  n_c-1 +  \exp\left(  \frac{n_c}{n_c-1} c\, r_\beta   \right)\right)   =  {\mu_\beta}   \qquad \text{ for all } 1 \le \beta \le s_c \label{sys1} \\
 &  \sum_{\beta = 1}^{s_c} \left(\frac{r_\beta}{c} \right)^2 = L n_c^{L-1} \label{sys2}.
\end{align}
\end{theorem2}
At the end of this section, we also show that if \eqref{lambdabound} holds, then the system \eqref{sys1} -- \eqref{sys2} has a unique solution (see proposition \ref{proposition:unique} in subsection \ref{section:F2}).

The strategy to prove theorem \ref{theorem:2} is straightforward: we simply need to evaluate the gradient of the risk on weights $(W,U)$ which are in a type-III collapse configuration. Setting this gradient to zero will then lead to a system for the constants $c,r_1, \ldots, r_{s_c}$ defining the configuration. While conceptually simple, the gradient computation is quite lengthy.

We start by deriving  formulas for the partial derivatives of $\mathcal R_0$ with respect to the linear weights $\bu_{k,\ell}$ and the word embeddings $\bw_{(\alpha,\beta)}$.  As we will see, 
 $\partial \mathcal R_0 /\partial \bu_{k,\ell}$ and  ${\partial \mathcal R_0}/{ \partial \bw_{(\alpha,\beta)}}$ plays symmetric roles. In order to observe this symmetry, the following notation will prove convenient:
 \begin{align} \label{Phi}
\Phi_{(\alpha,\beta) , (k,\ell)} (W,U) 
& := \frac{1}{K} \sum_{j=1}^K \sum_{\bx \in \data_j} \ones_{\{ x_\ell = (\alpha,\beta) \}} \Big( \ones_{\{j=k\}} - q_{k,W,U}(\bx) \Big) \, \mathcal D_{\bz_j}(\bx)
\end{align}
where 
$$
q_{k,W,U}(\bx) := \frac{e^{\dotprodbis{\hat U_k, W \hot(\bx)}}}{\sum_{k'=1}^K e^{\dotprodbis{U_{k'}, W \hot(\bx)}} }
$$
We may now state the first  lemma of this section:
\begin{lemma} \label{lemma:grad_simple} The partial derivatives of $\mathcal R_0$ with respect to $\bu_{k,\ell}$ and $\bw_{(\alpha,\beta)}$ are given by
\begin{gather*}
 -\frac{\partial \mathcal R_0}{ \partial \bu_{k,\ell}} (W ,U)  =  \sum_{\alpha=1}^{n_c} \sum_{\beta = 1}^{s_c}  \Phi_{(\alpha, \beta), (k,\ell)} (W,U) \; \bw_{(\alpha,\beta)} \\
-\frac{\partial \mathcal R_0}{ \partial \bw_{(\alpha,\beta)}} (W ,U) = \sum_{k=1}^{K} \sum_{\ell=1}^L  \Phi_{(\alpha,\beta), (k,\ell)} (W,U) \; \bu_{k,\ell} 
  \end{gather*}
\end{lemma}
\begin{proof} Given $K$ matrices  $V_1, \ldots, V_K \in \real^{n_w \times KL}$, we define 
$$
  f(V_1, \ldots, V_K) :=   \frac{1}{K}  \sum_{k=1}^{K}  \sum _{\bx \in \data_k}   \ell \Big( \dotprod{V_1 , \hot(\bx)}_F , \ldots, \dotprod{V_K , \hot(\bx)}_F  \; ;  k\Big)   \; \;   \mathcal D_{\bz_k}(\bx) 
$$
where $\ell(y_1,\ldots, y_K ; k)$ is the cross entropy loss
 \[
\ell(y_1, \ldots, y_K ;  k) = -   \log \left(  \frac{\exp\left( y_k\right)}{\sum_{k'=1}^K \exp\left( y_{k'}\right)}\right) 
\]
The partial derivative of $f$ with respect to the matrix $V_j$ can easily be found to be
\begin{align} \label{dfdV}
-\frac{\partial  f}{ \partial V_j}(V_1, \ldots, V_K)   = \frac{1}{K}  \sum_{k=1}^{K}  \sum _{\bx \in \data_k}   \left(   \ones_{\{j=k\}} -   \frac{e^{\dotprod{V_j , \hot(\bx)}_F }}{\sum_{k'=1}^K e^{\dotprod{V_{k'} , \hot(\bx)}_F }}  \right)\;\;  \hot(\bx)  \; \;   \mathcal D_{\bz_k}(\bx) 
\end{align}
We then recall from \eqref{def:net} that the $k^{th}$ entry of the vector $\by = h_{W,U}(\bx)$ is 
$$
 y_k = \dotprodbig{ \; \hat U_k \; , \;  W \, \hot(\bx)} =  \dotprodbig{ \;  W^T\hat U_k \; , \;   \hot(\bx)} 
$$
and so the unregularized risk can be expressed in term of the function $f$:
\begin{align*}
\mathcal R_0(W,U) &=  \frac{1}{K}  \sum_{k=1}^{K}  \sum _{\bx \in \data_k}   \ell \Big( \dotprod{W^T \hat U_1 , \hot(\bx)}_F , \ldots, \dotprod{W^T\hat U_K , \hot(\bx)}_F  \; ;  k\Big)   \; \;   \mathcal D_{\bz_k}(\bx) \\
&=    f(W^T \hat U_1, \ldots, W^T\hat U_K  ) 
\end{align*}
The chain rule then gives
 \begin{align}
\frac{\partial \risk_0}{\partial W} (W,U) &=  \sum_{j=1}^K \hat U_j \;  \left[ \frac{\partial f}{\partial V_j} (W^T\hat U_1 , \ldots , W^T\hat U_K ) \right]^T \label{dudu1}  \\
\frac{\partial \risk_0}{\partial \hat U_j} (W,U) &= W \;  \left[ \frac{\partial f}{\partial V_j} (W^T\hat U_1 , \ldots , W^T\hat U_K )  \right] \label{dudu2}
\end{align}
Using formula \eqref{dfdV} for $\partial f/ \partial V_j$ and the notation
$$
q_{j,W,U}(\bx) := \frac{e^{\dotprodbis{W^T\hat U_j,  \hot(\bx)}}}{\sum_{k'=1}^K e^{\dotprodbis{W^TU_{k'},  \hot(\bx)}} }
$$
we can express \eqref{dudu1} and \eqref{dudu2} as follow
  \begin{align*}
-\frac{\partial \risk_0}{\partial W} (W,U) &=  \sum_{j=1}^K  \hat U_j   \left[  \frac{1}{K}  \sum_{k=1}^{K}  \sum _{\bx \in \data_k}   \Big(   \ones_{\{j=k\}} -   q_{j,W,U}(\bx)  \Big)\;\;   \hot(\bx)  \; \;   \mathcal D_{\bz_k}(\bx)  \right]^T  \\
-\frac{\partial \risk_0}{\partial \hat U_j} (W,U) &=   W \left[ \frac{1}{K} \sum_{k=1}^{K}  \sum _{\bx \in \data_k}   \Big(   \ones_{\{j=k\}} -   q_{j,W,U}(\bx)  \Big)\; \hot(\bx)  \; \;   \mathcal D_{\bz_k}(\bx)   \right] 
\end{align*}
We now compute the partial derivative of $\mathcal R_0$ with respect to $\bu_{j,\ell}$. Let $\be_\ell \in \real^L$ be the $\ell^{th}$ basis vector. We then have 
\begin{align*}
-\frac{\partial \risk_0}{\partial \bu_{j,\ell}} (W,U) &=
-\left[\frac{\partial \risk_0}{\partial \hat U_j} (W,U) \right] \be_\ell\\
&=   \frac{1}{K}\ \sum_{k=1}^{K}  \sum _{\bx \in \data_k}   \left(   \ones_{\{j=k\}} -   q_{j,W,U}(\bx)  \right)\;\; \left( W  \hot(\bx) \; \be_\ell \right) \; \;   \mathcal D_{\bz_k}(\bx)
\end{align*}
Recall from \eqref{zebra} that $W \hot(\bx)$ is the matrix that contains the $d$-dimensional embeddings of the words that constitute the sentence $\bx \in \data$.
So   $W  \hot(\bx)\,  \be_\ell  $ is simply the embedding of the $\ell^{th}$ word of the sentence $\bx$, and  we can write it as
$$
 W  \hot(\bx) \, \be_\ell = \sum_{\alpha=1}^{n_c} \sum_{\beta=1}^{s_c} \ones_{\{ x_\ell = (\alpha,\beta)\}} \bw_{(\alpha,\beta)}
$$
We therefore have 
\begin{align*}
-\frac{\partial \risk_0}{\partial \bu_{j,\ell}} (W,U) &=  
\frac{1}{K}\ \sum_{k=1}^{K}  \sum _{\bx \in \data_k}   \left(   \ones_{\{j=k\}} -   q_{j,W,U}(\bx)  \right)\; \left( \sum_{\alpha=1}^{n_c} \sum_{\beta=1}^{s_c} \ones_{\{ x_\ell = (\alpha,\beta)\}} \bw_{(\alpha,\beta)} \right) \;   \mathcal D_{\bz_k}(\bx) \\
& =  \sum_{\alpha=1}^{n_c} \sum_{\beta = 1}^{s_c}  \left(
\frac{1}{K}\ \sum_{k=1}^{K}  \sum _{\bx \in \data_k}   \left(   \ones_{\{j=k\}} -   q_{j,W,U}(\bx)  \right)\;\ones_{\{ x_\ell = (\alpha,\beta)\}}  \mathcal D_{\bz_k}(\bx) \right)\; \bw_{(\alpha,\beta)}  \\
&= \sum_{\alpha=1}^{n_c} \sum_{\beta = 1}^{s_c}  \Phi_{(\alpha, \beta), (j,\ell)} (W,U) \; \bw_{(\alpha,\beta)} 
\end{align*}
which is the desired formula.

We now compute the gradient with respect $\bw_{(\alpha,\beta)}$. Recalling that $\hot(\alpha,\beta)$ is the one hot vector associate with word $(\alpha,\beta)$, we have 
  \begin{align*}
  -\frac{\partial \risk_0}{\partial \bw_{(\alpha,\beta)}} (W,U) & = 
-  \left[ \frac{\partial \risk_0}{\partial W} (W,U)  \right] \;\; \hot(\alpha,\beta)
\\&=   \frac{1}{K} \sum_{j=1}^K \sum_{k=1}^{K}  \sum _{\bx \in \data_k}   \left(   \ones_{\{j=k\}} -   q_{j,W,U}(\bx)  \right)\;\;  \left( \hat U_j  \;  \hot(\bx)^T     \hot(\alpha,\beta) \right)\;   \mathcal D_{\bz_k}(\bx)  
\end{align*}
Recall that the  $\ell^{th}$ column of $\hot(\bx)$ is the one-hot encoding of the $\ell^{th}$ word in the sentence $\bx$. Therefore, the $\ell^{th}$ entry of the vector $\hot(\bx)^T    \hot(\alpha,\beta) \in \real^L$ is given by the formula 
$$
 \left[ \hot(\bx)^T    \hot(\alpha,\beta)\right]_\ell  = \begin{cases}
 1 & \text{if } x_\ell = (\alpha,\beta) \\
 0& \text{otherwise}
 \end{cases}
$$
As a consequence 
$$
\hat U_j  \;  \hot(\bx)^T    \hot(\alpha,\beta)  = \sum_{\ell=1}^L  \ones_{ \{x_\ell = (\alpha,\beta)\}} \bu_{j,\ell}
$$
which leads to 
\begin{align*}
  -\frac{\partial \risk_0}{\partial \bw_{(\alpha,\beta)}} (W,U) 
  &=   \frac{1}{K} \sum_{j=1}^K \sum_{k=1}^{K}  \sum _{\bx \in \data_k}   \left(   \ones_{\{j=k\}} -   q_{j,W,U}(\bx)  \right)\;\;  \left(  \sum_{\ell=1}^L  \ones_{ \{x_\ell = (\alpha,\beta)\}} \bu_{j,\ell} \right)\;   \mathcal D_{\bz_k}(\bx) \\
  &=   \sum_{\ell=1}^L   \sum_{j=1}^K  \left(\frac{1}{K} \sum_{k=1}^{K}  \sum _{\bx \in \data_k}   \left(   \ones_{\{j=k\}} -   q_{j,W,U}(\bx)  \right)\;\;    \ones_{ \{x_\ell = (\alpha,\beta)\}}  \right)\;   \mathcal D_{\bz_k}(\bx) \;   \bu_{j,\ell} \\
  & = \sum_{\ell=1}^L \sum_{j=1}^{K}   \Phi_{(\alpha,\beta), (j,\ell)} (W,U) \; \bu_{j,\ell} 
\end{align*}
which is the desired formula.
\end{proof}

\subsection{Gradient of the risk for weights in type-III collapse configuration}
In lemma \ref{lemma:grad_simple} we computed the gradient of the risk for any possible weights $(W,U)$ and for any possible latent variables $\bz_1, \ldots, \bz_K$. In this section we will derive a formula for the gradient when the weights are in type-III collapse configuration and when the latent variables satisfy $\{\bz_1, \ldots, \bz_K\} = \mathcal Z$.
We start by recalling the definition of a type-III collapse configuration.
\begin{definition}[Type-III Collapse]  \label{def:3}  The weights $(W,U)$ of the network $h_{W,U}$ form a type-III collapse configuration if and only if
\begin{enumerate}[topsep=0pt]
\item[\rm i)] There exists positive scalars $r_{\beta} \ge 0$ so that $\bw_{(\alpha,\,\beta)} =  r_\beta  \,\cpt_\alpha$ \;  for all $(\alpha,\beta) \in \voc$.
\item[\rm ii)] There exists $c \ge 0$ so that $\bu_{k, \ell} =  c\,  \cpt_\alpha$ \; for all  $(k,\ell)$ satisfying $z_{k, \ell}=\alpha$ and all $\alpha \in \mathcal C$.
\end{enumerate}
hold for some collection $\cpt_1, \ldots, \cpt_{n_c} \in \real^d$ of equiangular vectors.
\end{definition}
We also define the constant $\gamma \in \real $ and the sigmoid $\sigma: \real \to \real$ as follow:
\begin{equation}
\gamma:= \frac{1}{n_c-1} \qquad \text{and } \qquad \sigma(x) :=  \frac{  1 }{1+ \gamma e^{ \left( 1+ \gamma \right) x  }} 
\end{equation}
The goal of this subsection is to prove the following proposition.
\begin{proposition} \label{proposition:grad} Suppose  $K=n_c^L$ and $\{\bz_1, \ldots, \bz_K\} = \mathcal Z$.
If  the  weights $(W,U)$ are in a type-III collapse configuration with constants $c, r_1, \ldots, r_{s_c} \ge 0$, then 
\begin{align*}
-\frac{\partial\, \risk_0}{ \partial \bu_{k,\ell}} (W ,U) 
& =  \frac{1}{c} \; \frac{1+ \gamma}{n_c^L}\left(  \sum_{\beta = 1}^{s_c} \mu_\beta  \; \sigma(c \, r_{\beta}) \; r_\beta \right) \;  \bu_{k,\ell} \\
-\frac{\partial \, \risk_0}{ \partial \bw_{(\alpha,\beta)}} (W,U)  %&  =  L \frac{\mu[\beta]}{n_c}   \frac{\sigma(c^2 \lambda_{\beta})}{\lambda_\beta}  \;    (1+\gamma)  \bw_{\alpha,\beta} \; \\
& =  c \; \frac{L (1+\gamma)}{n_c} \;  \frac{\mu_\beta \,  \sigma( c \;  r_{\beta})}{r_\beta} \;  \bw_{(\alpha,\beta)} 
\end{align*}
\end{proposition}
Importantly, note that the above proposition states that  $ {\partial\, \risk_0}/{ \partial \bu_{k,\ell}}$ and $\bu_{k,\ell} $ are aligned, and that $ {\partial\, \risk_0}/{ \partial \bw_{(\alpha,\beta)}}$ and $\bw_{(\alpha,\beta)} $ are aligned.

We start by introducing some notations which will make these gradient computations easier.  The latent variables $\bz_1, \ldots, \bz_K$ will be written as
$$
\bz_k = [\; z_{k,1} \; ,  \; z_{k,2} \; ,  \; \ldots \; ,  \;z_{k,L} \; ]  \in \mathcal Z
$$   
where  $1 \le z_{k,\ell}  \le n_c$.
We  remark that any sentence  $\bx$ generated by the latent variable $\bz_k$ must be of the form
$$
\bx = [(z_{k,1}, \beta_1) , \ldots, (z_{k,L}, \beta_L)] 
$$
for some  $(\beta_1, \ldots, \beta_L) \in [n_c]^L$, and that this sentence has a probability $\mu_{\beta_1} \mu_{\beta_2} \cdots \mu_{\beta_L}$ of being sampled. In light of this, we make the following definitions. For every    $\bbeta = (\beta_1, \ldots, \beta_L) \in [n_c^L] $ we let
\begin{align}
&\bx_{k,\bbeta} := [(z_{k,1}, \beta_1) , \ldots, (z_{k,L}, \beta_L)]  \; \in \data \\
& \mu[\bbeta] := \mu[{\beta_1}] \, \mu[{\beta_2}]\,  \cdots \, \mu[{\beta_L}]  \; \in [0,1]
\end{align}
where we have used $\mu[\beta_\ell]$ instead of $\mu_{\beta_\ell}$ in order to avoid the double subscript. 
With these definitions at hand we have that 
$$
\mathcal D_{\bz_j}(\bx_{k,\bbeta}) =  \begin{cases}  \mu[\bbeta] & \text{if } k=j \\
0 & \text{otherwise}
\end{cases}
$$
 We are now ready to prove proposition \ref{proposition:grad}. We break the computation  into four lemmas.
 The first one simply uses the notations that we just introduced in order to express $\Phi_{(\alpha,\beta) , (k,\ell)}$ in a more convenient format. 
\begin{lemma} \label{lemma:phi0}
The quantity $\Phi_{(\alpha^\star,\beta^\star) , (k,\ell)} (W,U)$ can be expressed as
\[
\Phi_{(\alpha^\star,\beta^\star) , (k,\ell)} (W,U) = \frac{1}{K}  \sum_{\bbeta \in [n_c^L]} \ones_{\{  \beta_\ell= \beta^\star \}}   \left( \ones_{\{z_{k,\ell}=\alpha^\star \} } -\sum_{j=1}^K \ones_{\{  z_{j,\ell}= \alpha^\star\}} \;  q_{k,W,U}(\bx_{j,\bbeta})  \right)  \, \mathcal \mu[\bbeta].
\]
\end{lemma}
\begin{proof} Using the above notations, we rewrite $\Phi_{(\alpha,\beta) , (k,\ell)} (W,U)$ as follow: 
 \begin{align*} \label{Phi}
\Phi_{(\alpha^\star,\beta^\star) , (k,\ell)} (W,U) 
& = \frac{1}{K} \sum_{j=1}^K \sum_{\bx \in \data_j} \ones_{\{ x_\ell = (\alpha^\star,\beta^\star) \}} \Big( \ones_{\{j=k\}} - q_{k,W,U}(\bx) \Big) \, \mathcal D_{\bz_j}(\bx)\\
&  = \frac{1}{K} \sum_{j=1}^K \sum_{\bbeta \in [n_c^L]} \ones_{\{  (z_{j,\ell}, \beta_\ell) = (\alpha^\star,\beta^\star) \}} \Big( \ones_{\{j = k\}} - q_{k,W,U}(\bx_{j,\bbeta}) \Big) \, \mathcal D_{\bz_j}(\bx_{j,\bbeta}) \\
&  = \frac{1}{K} \sum_{j=1}^K \sum_{\bbeta \in [n_c^L]} \ones_{\{  z_{j,\ell}= \alpha^\star\}}   \ones_{\{  \beta_\ell= \beta^\star \}}  \Big( \ones_{\{j=k\}} - q_{k,W,U}(\bx_{j,\bbeta}) \Big) \,  \mu[\bbeta] \\
&  = \frac{1}{K}  \sum_{\bbeta \in [n_c^L]} \ones_{\{  \beta_\ell= \beta^\star \}}   \left(\sum_{j=1}^K \ones_{\{  z_{j,\ell}= \alpha^\star\}}    \Big( \ones_{\{j=k\}} - q_{k,W,U}(\bx_{j,\bbeta}) \Big) \right)  \, \mu[\bbeta]  
\end{align*}
To conclude the proof we simply remark that
$
\sum_{j} \ones_{\{  z_{j,\ell}= \alpha^\star\}}    \ones_{\{j=k\}} = \ones_{\{  z_{k,\ell}= \alpha^\star \}} 
$.
\end{proof}
The following notation will be needed in our next lemma:
\begin{equation} \label{delta}
\delta(\alpha,\alpha') = \begin{cases}
1 & \text{if }\alpha = \alpha' \\
- \gamma & \text{ if } \alpha \neq \alpha'
\end{cases} \qquad \text{ for all } \alpha,\alpha' \in [n_c]
\end{equation}
where we recall that $\gamma = 1/(n_c-1)$.
We think  of  $\delta(\alpha,\alpha')$  as a
 `biased Kroecker delta' on the concepts. Importantly, note that if $\mathfrak f_1 , \ldots, \mathfrak f_{n_c}$ are equiangular, then
 $$
 \dotprod{\mathfrak f_\alpha , \mathfrak f_{\alpha'} } = \delta(\alpha, \alpha')
 $$
 which is the motivation behind this definition.  We may now state our second lemma.

\begin{lemma} Assume  $K=n_c^L$ and $\{\bz_1, \ldots, \bz_K\} = \mathcal Z$. Assume also   that  the   weights $(W,U)$ are in a type-III collapse configuration with constants $c, r_1, \ldots, r_{s_c} \ge 0$.
Then 
$$
q_{k,W,U}(\bx_{j,\bbeta}) = \frac{ \prod_{\ell = 1}^L  \exp\Big(c \; r_{\beta_\ell} \;  \delta(z_{j,\ell}, z_{k,\ell} ) \Big)}{ \prod_{\ell=1}^L\psi(c\; r_{\beta_\ell}) }  \qquad \text{where} \quad  \psi(x) = e^x +  \frac{1}{ \gamma} e^{-\gamma x}.
$$
for all $j,k \in [K]$ and all $\bbeta = (\beta_1, \ldots, \beta_L) \in [n_c]^L$.
\end{lemma}
\begin{proof} Recalling that $\bx_{j,\bbeta} := [(z_{j,1}, \beta_1) , \ldots, (z_{j,L}, \beta_L)] $, we obtain
 \begin{align*}
 \dotprodbis{\hat U_k, W \hot(\bx_{j,\bbeta})}
 = \sum_{\ell = 1}^L \dotprod{\bu_{k,\ell} , \bw_{(z_{j,\ell} , \beta_\ell) }  }
 =  \sum_{\ell = 1}^L  \dotprod{ \;c\;\cpt_{ z_{k,\ell}} \; , \;   r_{\beta_\ell} \;\cpt_{z_{j,\ell} } \;   } 
%& = c^2 \sum_{\ell = 1}^L  \lambda_{b_\ell} \dotprodbis{\cpt_{\bar a_\ell} ,   \;\cpt_{a_\ell }   } \\
 =  c \sum_{\ell = 1}^L  r_{\beta_\ell}  \, \delta(  z_{k,\ell}\; , z_{j,\ell}   )   
\end{align*}
We then have 
\begin{multline*}
q_{k,W,U}(\bx_{j,\bbeta}) = \frac{e^{ \dotprodbis{\hat U_k, W \hot(\bx_{j,\bbeta})}}}{\sum_{k'=1}^K  e^{\dotprodbis{\hat U_{k'}, W \hot(\bx_{j,\bbeta})}}} 
= \frac{ \exp \left(   c \sum_{\ell = 1}^L  r_{\beta_\ell}  \, \delta(  z_{k,\ell}\; , z_{j,\ell}   )    \right)   }{\sum_{k'=1}^K   \exp \left(  c \sum_{\ell = 1}^L  r_{\beta_\ell}  \, \delta(  z_{k',\ell}\; , z_{j,\ell}   )  \right)   } \\
= \frac{  \prod_{\ell = 1}^L \exp \Big(   c \;   r_{\beta_\ell}   \; \delta(  z_{k,\ell}\; , z_{j,\ell}   )    \Big)   }{\sum_{k'=1}^K   \prod_{\ell = 1}^L \exp \Big(  c \; r_{\beta_\ell}  \; \delta(  z_{k',\ell}\; , z_{j,\ell}   )  \Big)   }
\end{multline*}
Since $\{\bz_1, \ldots, \bz_K\} = \mathcal Z$, the latent variables
$
\bz_{k'} = [ z_{k',1},   \ldots  ,  z_{k',L}  ]  
$
achieve all possible tuples $[\alpha_1', \cdots, \alpha'_L] \in [n_c]^L$. The bottom term can therefore be expressed as
\begin{align*}
\sum_{k'=1}^K  & \prod_{\ell = 1}^L \exp \Big(  c \; r_{\beta_\ell}  \; \delta(  z_{k',\ell}\; , z_{j,\ell}   )  \Big) \\
& = \sum_{\alpha'_1 =1}^{n_c} \sum_{\alpha'_2 =1}^{n_c} \cdots  \sum_{\alpha'_L =1}^{n_c}  \exp\Big( { c \,r_{\beta_1}   \delta(  \alpha'_{1} , z_{j,1}   )  }  \Big) \; \exp\Big({ c\,  r_{\beta_2}  \delta(  \alpha'_2 , z_{j,2}   )  } \Big)
\cdots  \; \exp\Big( { c\, r_{\beta_L}  \delta(  \alpha'_L , z_{j,L}   )  }  \Big) \\
& = \prod_{\ell=1}^L  \left(   \sum_{\alpha'_\ell =1}^{n_c}  \exp\Big({c \, r_{\beta_\ell}  \delta(\alpha'_\ell, z_{k,\ell} ) } \Big) \right) 
%& =   \prod_{\ell=1}^L \phi( \lambda_{b_\ell})
\end{align*}
Recalling the definition of $\delta(\alpha,\alpha')$, we find  that
\begin{multline} 
 \sum_{\alpha'_\ell =1}^{n_c}  \exp\Big({ c \, r_{\beta_\ell} \delta(\alpha'_\ell, z_{k,\ell}) }\Big)  = \exp(c \,r_{\beta_\ell}) + \sum_{\alpha'_\ell \neq  z_{k,\ell}} \exp\left(- \frac{c\, r_{\beta_\ell}}{n_c-1}\right) \\
 = \exp(c \,r_{\beta_\ell}) +(n_c-1)\, \exp\left(- \frac{c\, r_{\beta_\ell}}{n_c-1}\right)  
 = \psi(c \,r_{\beta_\ell}) \label{repeat66}
\end{multline}
\end{proof}

We now find a convenient expression for the term appearing between parenthesis in the statement of   lemma \ref{lemma:phi0}.
\begin{lemma} \label{lemma:phi1} Assume  $K=n_c^L$ and $\{\bz_1, \ldots, \bz_K\} = \mathcal Z$. Assume also   that  the   weights $(W,U)$ are in a type-III collapse configuration with constants $c, r_1, \ldots, r_{s_c} \ge 0$.
Then 
\begin{align} \label{term_par}
\ones_{\{z_{k,\ell}=\alpha^\star \} } - \sum_{ j=1 }^K   \ones_{\{ z_{j,\ell} = \alpha^\star\}} \;  q_{k,W,U}(\bx_{j,\bbeta}) 
&  =    \delta( z_{k,\ell}, \alpha^\star) \;  \sigma(c \, r_{\beta_\ell}) 
\end{align}
for all $k \in [K], \ell \in[L]$, $\alpha^\star \in [n_c]$ and all $\bbeta = (\beta_1, \ldots, \beta_L) \in [n_c]^L$.
\end{lemma}
\begin{proof}
For simplicity we are going to prove equation \eqref{term_par} in the case  $\ell=1$. Using the previous lemma we obtain
\begin{align*}
\sum_{ j=1 }^K   \ones_{\{ z_{j,1} = \alpha^\star \}}  q_{k,W,U}(\bx_{j,\bbeta}) 
 &= \sum_{ j=1 }^K   \ones_{\{ z_{j,1} = \alpha^\star\}}  \frac{ \prod_{\ell = 1}^L  \exp\Big(c \, r_{\beta_\ell}  \delta(z_{j,\ell}, z_{k,\ell} ) \Big)}{ \prod_{\ell=1}^L\psi(c\, r_{\beta_\ell}) } 
\end{align*}
Since 
the latent variables
$
\bz_{j} = [ z_{j,1},   \ldots  ,  z_{j,L}  ]  
$
achieve all possible tuples $[\alpha_1, \cdots, \alpha_L] \in [n_c]^L$, we can rewrite the above as
\begin{align}
%\sum_{ j=1 }^K   \ones_{\{ z_{j,1} = \alpha\}}  q_{k}(\bx_{j,\bbb}) 
   &\sum_{\alpha_1 =1}^{n_c} \sum_{\alpha_2 =1}^{n_c} \cdots  \sum_{\alpha_L =1}^{n_c}    \ones_{\{ \alpha_1 = \alpha^\star\}}  \frac{ \prod_{\ell = 1}^L  \exp\Big(c \, r_{\beta_\ell}  \delta(\alpha_\ell, z_{k,\ell} ) \Big)}{ \prod_{\ell=1}^L\psi(c\, r_{\beta_\ell}) } \nonumber \\
  & \hspace{1cm}=   \sum_{\alpha_2 =1}^{n_c} \cdots  \sum_{\alpha_L =1}^{n_c}    \frac{   \exp\Big(c \, r_{\beta_1}  \delta(\alpha^\star, z_{k,1} ) \Big) \;\;  \prod_{\ell = 2}^L  \exp\Big(c \, r_{\beta_\ell}  \delta(\alpha_\ell, z_{k,\ell} ) \Big)}{ \prod_{\ell=1}^L\psi(c\, r_{\beta_\ell}) } \nonumber  \\
  &  \hspace{1cm}=     \frac{   \exp\Big(c \, r_{\beta_1}  \delta(\alpha^\star, z_{k,1} ) \Big)}{ \prod_{\ell=1}^L\psi(c\, r_{\beta_\ell}) } 
    \sum_{\alpha_2 =1}^{n_c} \cdots  \sum_{\alpha_L =1}^{n_c}   \; \prod_{\ell = 2}^L  \exp\Big(c \, r_{\beta_\ell}  \delta(\alpha_\ell, z_{k,\ell} ) \Big)  \label{kakako}
\end{align}
We then note that 
$$
 \sum_{\alpha_2 =1}^{n_c} \cdots  \sum_{\alpha_L =1}^{n_c}   \; \prod_{\ell = 2}^L  \exp\Big(c \, r_{\beta_\ell}  \delta(\alpha_\ell, z_{k,\ell} ) \Big)  =
  \prod_{\ell=2}^L  \left(   \sum_{\alpha'_\ell =1}^{n_c}  \exp\Big({c \, r_{\beta_\ell}  \delta( \alpha_\ell, z_{k,\ell}} ) \Big) \right) 
$$
and, repeating computation \eqref{repeat66}, we find that 
$$
\sum_{\alpha_\ell =1}^{n_c}  \exp\left({c \, r_{\beta_\ell}  \delta( \alpha_\ell, z_{k,\ell}} ) \right)  = \psi(c \,r_{\beta_\ell})
$$
Going back to \eqref{kakako} we therefore have 
\begin{align*}
\sum_{ j=1 }^K   \ones_{\{ z_{j,1} = \alpha^\star \}}  q_{k,W,U}(\bx_{j,\bbeta}) 
& =      \frac{   \exp\Big(c \, r_{\beta_1}  \delta(\alpha^\star, z_{k,1} ) \Big)}{ \prod_{\ell=1}^L\psi(c\, r_{\beta_\ell}) }   { \prod_{\ell = 2}^L  \psi(c\,  r_{\beta_\ell}) }  
  =   \frac{   \exp\Big(c \, r_{\beta_1}  \delta(\alpha^\star , z_{k,1} ) \Big)}{\psi(c\, r_{\beta_1}) }  
\end{align*}
and so 
\begin{align}
\ones_{\{z_{k,1}=\alpha^\star \} } - \sum_{ j=1 }^K   \ones_{\{ z_{j,1} = \alpha^\star\}} \;  q_{k,W,U}(\bx_{j,\bbeta}) 
&  =  \begin{cases} 
    1 -  \frac{\exp\left(c \, r_{\beta_1}   \right) }{\psi( c \, r_{\beta_1})}   & \text{if }  z_{k,1} = \alpha^\star \\
     -  \frac{\exp\left(- \gamma \,  c\,  r_{\beta_1}  \right) }{\psi( c \, r_{\beta_1})} &  \text{if }  z_{k,1} \neq  \alpha^\star 
    \end{cases}
\end{align}
We  now manipulate the above formula. Recalling that $\gamma = 1 / (n_c-1)$, and recalling the definition of $\psi(x)$, we get
\begin{align}
1 -  \frac{e^x }{\psi(x)} =  1 - \frac{e^x}{e^x + \frac{1}{\gamma} e^{-\gamma x} } = \frac{1}{1 + \gamma e^{(1+\gamma)x}} = \sigma(x)
\end{align}
and 
\begin{align*}
-  \frac{e^{-\gamma x} }{\psi(x)} = - \frac{e^{-\gamma x}}{e^x + \frac{1}{\gamma} e^{-\gamma x} } = \; -  \gamma \left( \frac{1}{1 + \gamma e^{(1+\gamma)x}}\right)
= - \gamma \sigma(x)
\end{align*}
which concludes the proof.
\end{proof}

Our last lemma provides a formula for the quantity $\Phi_{(\alpha^\star,\beta^\star) , (k,\ell)} (W,U) $ when the weights are in a type-III collapse configuration.
\begin{lemma} \label{lemma:phi77}Assume  $K=n_c^L$ and $\{\bz_1, \ldots, \bz_K\} = \mathcal Z$. Assume also   that  the   weights $(W,U)$ are in a type-III collapse configuration with constants $c, r_1, \ldots, r_{s_c} \ge 0$.
Then 
\begin{align} \label{phi77}
\Phi_{(\alpha,\beta) , (k,\ell)} (W,U) =  \frac{\mu_\beta}{n^L_c}  \;   \sigma(c \, r_{\beta})  \;  \delta( z_{k,\ell}, \alpha) 
\end{align}
for all $k \in [K], \ell \in[L]$, $\alpha \in [n_c]$ and  $\beta \in [s_c]$.
\end{lemma}
\begin{proof}
Combining lemmas \ref{lemma:phi0} and \ref{lemma:phi1}, and recalling that $K = n_c^L$, we obtain 
\begin{align*}
\Phi_{(\alpha^\star,\beta^\star) , (k,\ell)} (W,U)  & = \frac{1}{n_c^L}  \sum_{\bbeta \in [n_c^L]} \ones_{\{  \beta_\ell= \beta^\star \}}   \left( \ones_{\{z_{k,\ell}=\alpha^\star \} } -\sum_{j=1}^K \ones_{\{  z_{j,\ell}= \alpha^\star\}} \;  q_{k,W,U}(\bx_{j,\bbeta})  \right)  \, \mathcal \mu[\bbeta]\\
& =  \frac{1}{n_c^L}  \sum_{\bbeta \in [n_c^L]} \ones_{\{  \beta_\ell= \beta^\star \}}   \Big( 
 \delta( z_{k,\ell}, \alpha^\star) \;  \sigma(c \, r_{\beta_\ell})  \Big)  \, \mathcal \mu[\bbeta] \\
 & =  \frac{ \delta( z_{k,\ell}, \alpha^\star)}{n_c^L}  \sum_{\bbeta \in [n_c^L]} \ones_{\{  \beta_\ell= \beta^\star \}}    
 \;  \sigma(c \, r_{\beta_\ell}) 
   \, \mathcal \mu[\bbeta]
\end{align*}
    Choosing $\ell=1$ for simplicity we get
    \begin{align*}
    \sum_{\bbeta \in [n_c^L]} \ones_{\{  \beta_1= \beta^\star \}}   \;  \sigma(c \, r_{\beta_1})  \, \mathcal \mu[\bbeta]  
    &= \sum_{\beta_1 = 1}^{s_c}   \sum_{\beta_2 = 1}^{s_c} \cdots \sum_{\beta_L=1}^{s_c} \;  \ones_{\{\beta_1= \beta^\star \}}  \;    \;   \sigma(c \,r_{\beta_1})  \; \mu[{\beta_1}] \mu[{\beta_2}] \cdots \mu[{\beta_L}] \\
    &=  \sum_{\beta_2 = 1}^{s_c} \cdots \sum_{\beta_L=1}^{s_c}    \;   \sigma(c \,r_{\beta^\star})  \; \mu[{\beta^\star}] \mu[{\beta_2}] \cdots \mu[{\beta_L}] \\
    & =  \mu[\beta^\star] \;  \sigma(c\,  r_{\beta^\star}) 
    \end{align*}
    which concludes the proof.
    \end{proof}
    
    We now prove the proposition.
    \begin{proof}[Proof of proposition \ref{proposition:grad}]
     Combining lemmas \ref{lemma:grad_simple} and \ref{lemma:phi77}, and using the fact that $\bw_{(\alpha,\beta)}  = r_\beta \mathfrak f_\alpha$, we obtain
\begin{align*}
 -\frac{\partial \mathcal R_0}{ \partial \bu_{k,\ell}} (W ,U)  &=  \sum_{\alpha=1}^{n_c} \sum_{\beta = 1}^{s_c}  \Phi_{(\alpha, \beta), (k,\ell)} (W,U) \; \bw_{(\alpha,\beta)} \\ 
 & =     \;    \sum_{\alpha=1}^{n_c} \sum_{\beta = 1}^{s_c} \left( \frac{\mu_\beta}{n^L_c} \;   \sigma(c \, r_{\beta})  \;  \delta( z_{k,\ell}, \alpha)  \right) \; r_\beta \mathfrak f_\alpha\\
 & = \frac{1}{n_c^L}\left(  \sum_{\beta = 1}^{s_c} \mu_\beta \; \sigma(c \,r_{\beta}) \; r_\beta \right) \left(  \sum_{\alpha=1}^{n_c}  \delta(  z_{k,\ell}, \alpha)  \;  \cpt_\alpha  \right)
  \end{align*}
  Using the fact that $\sum_{\alpha=1}^{n_c} \cpt_\alpha = 0$ we get 
\begin{equation} \label{repeat99}
 \sum_{\alpha=1}^{n_c}  \delta(  z_{k,\ell}, \alpha)  \;  \cpt_\alpha  =  \cpt_{z_{k,\ell}} -  \gamma  \sum_{\alpha \neq z_{k,\ell}} \cpt_\alpha =  \cpt_{z_{k,\ell}}  + \gamma \,  \cpt_{z_{k,\ell}}  -  \gamma  \sum_{\alpha=1}^{n_c} \cpt_\alpha = 
 (1+ \gamma) \;  \cpt_{z_{k,\ell}} 
\end{equation}
Using the fact that $\bu_{k,\ell} = c \, \cpt_{z_{k,\ell}}$ we then get
\begin{align*}
 -\frac{\partial \mathcal R_0}{ \partial \bu_{k,\ell}} (W ,U)  &= \frac{1}{n_c^L}\left(  \sum_{\beta = 1}^{s_c} \mu_\beta \; \sigma(c \,r_{\beta}) \; r_\beta \right) (1+ \gamma) \;  \cpt_{z_{k,\ell}} \\
& =  \frac{1+ \gamma}{n_c^L}\left(  \sum_{\beta = 1}^{s_c} \mu_\beta \; \sigma(c \; r_{\beta}) \; r_\beta \right) \;  \frac{\bu_{k,\ell}}{c}
\end{align*}
which is the desired formula.

Moving to the other gradient we get
\begin{align*}
-\frac{\partial \mathcal R_0}{ \partial \bw_{(\alpha,\beta)}} (W ,U) &= \sum_{k=1}^{K} \sum_{\ell=1}^L  \Phi_{(\alpha,\beta), (k,\ell)} (W,U) \; \bu_{k,\ell} \\
& = \sum_{k=1}^K \sum_{\ell=1}^L  \left( \frac{\mu_\beta}{n^L_c}  \;   \sigma(c \; r_{\beta})  \;  \delta(  z_{k,\ell}, \alpha )  \right)\;  c  \; \cpt_{z_{k,\ell} } \\
%& =   \frac{\mu[\beta]}{n^L_c}   \sigma(c^2 \lambda_{\beta})  \; c   \sum_{\ba \in [n_c]^L} \sum_{\ell=1}^L   \delta_{  a_\ell, \alpha}    \; \cpt_{a_\ell}  \\
&  =   \frac{\mu_\beta}{n^L_c}   \sigma(c \;r_{\beta})  \; c  \sum_{\ell=1}^L   \left(  \sum_{k=1}^K  \delta( z_{k,\ell}, \alpha)    \; \cpt_{z_{k,\ell}} \right)
\end{align*}
Since 
the latent variables
$
\bz_{k} = [ z_{k,1},   \ldots  ,  z_{k,L}  ]  
$
achieve all possible tuples $[\alpha'_1, \cdots, \alpha'_L] \in [n_c]^L$, we have, fixing $\ell=1$ for simplicity, 
\begin{align}
 \sum_{k=1}^K  \delta( z_{k,1}, \alpha)    \; \cpt_{z_{k,1}}  =
\sum_{\alpha'_1 =1 }^{n_c} \sum_{\alpha'_2=1}^{n_c} \cdots \sum_{\alpha'_L=1}^L  \delta(  \alpha'_1, \alpha)    \; \cpt_{\alpha'_1}   =  n_c^{L-1}  \sum_{\alpha'_1=1}^L  \delta(\alpha_1', \alpha)    \; \cpt_{\alpha'_1}\end{align}
Repeating computation \eqref{repeat99} shows that the above is equal to $n_c^{L-1} (1+\gamma)  \; \cpt_{\alpha} $. We then use the fact that $\bw_{(\alpha,\beta)} = r_\beta \cpt_\alpha$ to obtain
\begin{align*}
-\frac{\partial \mathcal R_0}{ \partial \bw_{(\alpha,\beta)}} (W ,U) 
&  =   \frac{\mu_\beta}{n^L_c}   \sigma(c \;r_{\beta})  \; c   L\Big( n_c^{L-1} (1+\gamma)  \; \cpt_{\alpha} \Big)\\
&  =   \frac{\mu_\beta}{n^L_c}   \sigma(c \;r_{\beta})  \; c   L\Big( n_c^{L-1} (1+\gamma)  \; \frac{\bw_{(\alpha,\beta)} }{r_\beta}\Big)\\
&  =   \frac{\mu_\beta}{n_c}   \sigma(c \;r_{\beta})  \; c   L\Big(  (1+\gamma)  \; \frac{\bw_{(\alpha,\beta)} }{r_\beta}\Big)\\
%& = \frac{L (1+\gamma)}{n_c}  \mu_\beta  \frac{\sigma(c \,r_{\beta})}{r_\beta} \bw_{(\alpha,\beta)} 
\end{align*}
which is the desired formula.
\end{proof}

\subsection{Proof of the theorem and study of the non-linear system} \label{section:F2}
In this subsection we start by  proving theorem \ref{theorem:2}, and  then we show that the  system \eqref{sys1} -- \eqref{sys2} has a unique solution if the regularization parameter $\lambda$ is small enough.

\begin{proof}[Proof of theorem \ref{theorem:2}]
Recall that the regularized risk associated with the network $h_{W,U}$ is defined by
\begin{align} \label{flute33}
\mathcal R(W,U) &=  \mathcal R_0(W,U) +  \frac{\lambda}{2} \left( \|W\|_F^2 +  \|U\|_F^2  \right) \\
&=  \mathcal R_0(W,U) +  \frac{\lambda}{2} \left(  \sum_{\alpha=1}^{n_c} \sum_{\beta=1}^{s_c}\|\bw_{(\alpha,\beta)}\|^2 +  \sum_{k=1}^{K} \sum_{\ell=1}^{L}\|\bu_{k,\ell}\|^2   \right) \
\end{align}
and therefore $(W,U)$ is a critical points if and only if
 $$
 -\frac{\partial \risk_0}{ \partial \bu_{k,\ell}} (W ,U) =  \lambda \; \bu_{k,\ell} \qquad \text{and} \qquad  -\frac{\partial \risk_0}{ \partial \bw_{(\alpha,\beta)}} (W ,U) =  \lambda \; \bw_{(\alpha,\beta)} 
 $$
 According to proposition \ref{proposition:grad}, 
if (W,U) is in a type-III collapse configuration, then the above equations becomes
$$
\frac{1}{c} \; \frac{1+ \gamma}{n_c^L}\left(  \sum_{\beta = 1}^{s_c} \mu_\beta  \; \sigma(c \, r_{\beta}) \; r_\beta \right) \;  \bu_{k,\ell}  =  \lambda \; \bu_{k,\ell} \quad \text{and} \quad  
 c \; \frac{L (1+\gamma)}{n_c} \;  \frac{\mu_\beta \,  \sigma( c \;  r_{\beta})}{r_\beta} \bw_{(\alpha,\beta)}  =  \lambda \; \bw_{(\alpha,\beta)} 
 $$
So $(W,U)$ is critical if and only if the constants
$r_1, \ldots, r_{s_c}$ and $c$  satisfy the $s_c + 1$ equations
\begin{align}
& \frac{1}{c} \; \frac{1+ \gamma}{n_c^L} \sum_{\beta = 1}^{s_c} \mu_\beta \;  \sigma(c \, r_{\beta})  r_\beta  = \lambda \\
 &c \;  \frac{L (1+\gamma)}{n_c}  \frac{ \mu_\beta  \; \sigma(c \, r_{\beta})}{r_\beta} = \lambda \qquad \text{ for all } \beta \in [s_c] 
\end{align}
From the second equation we have that
$$
  (1+\gamma) \; \mu_\beta  \; \sigma(c \, r_{\beta}) \;  r_\beta = \;   \frac{n_c \; \lambda\; r^2_\beta }{L\, c}
$$
Using this we can rewrite the first equation as
$$
\frac{1}{c} \; \frac{1}{n_c^L} \sum_{\beta = 1}^{s_c}  \frac{n_c \; \lambda\; r^2_\beta }{L\, c}   = \lambda \qquad 
\text{which simplifies to } \qquad  \sum_{\beta = 1}^{s_c}  \left( \frac{ r_\beta }{c} \right)^2   = L n_c^{L-1}.
$$
which is the desired equation (see \eqref{sys2}).

We now rewrite the second equation as 
$$
\frac{\lambda}{L} \frac{r_\beta}{c}   \frac{ n_c  }{ (1+\gamma) \, \sigma(c \, r_{\beta}) } =   \mu_\beta
$$
We then recall that 
$ \sigma(x) :=  \frac{  1 }{1+ \gamma e^{ \left( 1+ \gamma \right) x  }} $ and therefore 
$$
\frac{n_c}{(1+ \gamma)  \; \sigma( cr_\beta)}=   \frac{n_c}{1+\gamma} (  1+ \gamma e^{ \left( 1+ \gamma \right) c r_\beta  } ) = 
 n_c-1 + \exp\left(  \frac{n_c}{n_c-1} c r_\beta  \right)
$$
and therefore the second  equation can be written as
$$
\frac{\lambda}{L} \; \frac{r_\beta}{c}  \left(  n_c-1 +  \exp\left(  \frac{n_c}{n_c-1} c\, r_\beta   \right)\right)   =  {\mu_\beta}.  
$$
\end{proof}

We now prove that if the regularization parameter $\lambda$ is small enough then the system has a unique solution.
\begin{proposition}\label{proposition:unique}
Assume  $\mu_1 \ge \mu_2 \ge \ldots \ge \mu_{s_c} > 0$ and 
\begin{equation}
 \lambda^2<  \frac{L}{n_c^{L+1}} \sum_{\beta=1}^{s_c} \mu_\beta^2 \label{lambdabound}.
\end{equation}
Then the system $s_c+1$ equations
\begin{align}
 & \frac{\lambda}{L} \; \frac{r_\beta}{c}  \left(  n_c-1 +  \exp\left(  \frac{n_c}{n_c-1} c\, r_\beta   \right)\right)   =  {\mu_\beta}   \qquad \text{ for all } 1 \le \beta \le s_c \label{sys1} \\
 &  \sum_{\beta = 1}^{s_c} \left(\frac{r_\beta}{c} \right)^2 = L n_c^{L-1} \label{sys2}
\end{align}
has a unique solution $(c,r_1, \ldots, r_{s_c}) \in \real^{s_c+1}_{+}$. Moreover this solution satisfies
$
r_1 \ge r_2 \ge \ldots \ge r_{s_c} > 0.
$

\end{proposition}
\begin{proof}
Letting $\rho_{\beta}:= r_\beta /c$, the system is equivalent to
\begin{align}
 & g(c , \rho_\beta) =  \frac{L}{\lambda n_c} {\mu_\beta} \qquad \text{ for all } \beta \in [s_c] \label{system1} \\
& \sum_{\beta = 1}^{s_c}  \rho^2_\beta  = L n_c^{L-1} \label{system2} 
\end{align}
for the  unknowns $(c, \rho_1, \rho_2, \ldots, \rho_{s_c})$ where 
 $$
g(c,x) =     x  \left(  1+ \gamma e^{ \left( 1+ \gamma \right)  c^2  x   }\right) / (1+ \gamma) \qquad \text{ and } \qquad \gamma = 1 / (n_c-1)
$$
Note that
$$
\frac{\partial g}{\partial x}(c,x) \ge  \left(  1+ \gamma e^{ \left( 1+ \gamma \right) c^2 x   }\right) / (1+\gamma) \ge 1 \qquad \forall (c,x) \in \real \times [0,+\infty)
$$
and therefore $x \mapsto g(c,x)$ is strictly increasing  on $[0,+\infty)$.  Also note that  we have
$$
g(c,0)=0, \qquad \lim_{x \to +\infty} g(c,x) = +\infty 
$$
So $x \mapsto g(c,x)$ is a bijection from $[0,+\infty)$ to $[0,+\infty)$ as well as a bijection from  $(0,+\infty)$ to $(0,+\infty)$.  Recall that  $\mu_\beta \in (0,+\infty)$ for all $\beta \in [s_c]$.
 Therefore
 given $c \in \real$ and $\beta \in [s_c]$, the equation
$$  g(c , x) =  \frac{L}{\lambda n_c} {\mu_\beta} $$
has a unique solution in $(0,+\infty)$ that we denote by $\phi_\beta(c)$.  In other words, the function $\phi_\beta(c)$ is implicitly defined by
\begin{equation}
g(c , \phi_{\beta}(c)) =  \frac{L}{\lambda n_c} {\mu_\beta}. \label{bozzo54}
\end{equation}
Also, since $g(0,x) = x$, we have
$$
\phi_\beta(0) =  \frac{L}{\lambda n_c} {\mu_\beta}
$$

\begin{claim} The function $\phi_\beta: [0,+\infty) \to (0,+\infty) $ is continuous, strictly decreasing, and satisfies $\lim_{c \to +\infty} \phi_\beta(c)=0$.
\end{claim}

\begin{proof} 
 We first show that  $c \mapsto \phi_\beta(c)$ is continuous.  Since  $\frac{\partial g}{\partial x}(c,x) \ge 1$ for all $x\ge 0 $, we have 
 $$
g(c, x_2) - g(c, x_1) = \int_{x_1}^{x_2} \frac{\partial g}{\partial x}(c,x) dx \ge  \int_{x_1}^{x_2} 1 dx = x_2 - x_1 \qquad  \text{ for all  $c$ and all } x_2 \ge x_1 \ge 0.
$$
 As a consequence, for all $c_1,c_2$, we have
\begin{equation}
|\phi_\beta(c_2) - \phi_\beta(c_1) | \le  | g(c_1, \phi_\beta(c_2) ) - g(c_1, \phi_\beta(c_1) )| =  | g(c_1, \phi_\beta(c_2) ) - g(c_2, \phi_\beta(c_2) )|
\label{bit36}
\end{equation}
where we have used the fact that $g(c_1, \phi_\beta(c_1) ) =  \frac{L}{\lambda n_c} {\mu_\beta} = g(c_2, \phi_\beta(c_2) ) $.
From \eqref{bit36} it is clear that the continuity of  $c \mapsto g(c, x )$ implies the continuity of $c \mapsto \phi_\beta(c)$.

We now prove that $\phi_\beta$ is strictly decreasing on $[0,+\infty)$. Let $0 \le c_1 < c_2$. Note that for any $x>0$, the function $c \mapsto g(c,x)$ is strictly increasing on $[0,+\infty)$. Since $\phi_\beta(c)>0$ we therefore have
$$
g(c_2, \phi_\beta(c_2)) =  \frac{L}{\lambda n_c} {\mu_\beta}  =  g(c_1, \phi_\beta(c_1)) < g(c_2, \phi_\beta(c_1)) 
$$
Since $x \mapsto g(c,x)$ is strictly increasing for all $c$, the above implies that  $\phi_\beta(c_2) <  \phi_\beta(c_1)$.

Finally we show that $\lim_{c \to +\infty} \phi_\beta(c)=0$.  Since $\phi_\beta$ is decreasing and non-negative on $[0,+\infty)$, the $\lim_{c\to +\infty} \phi_\beta(c)=A$ is well defined. 
We obviously have $\phi_\beta(c) \ge A$ for all $c \ge 0$. Since $x \mapsto g(c,x)$ is  increasing we have
$$
 \frac{L}{\lambda n_c} {\mu_\beta}   = g(c, \phi_\beta(c) ) \ge  g(c, A )
$$
But the function $g(c, A )$ is unbounded for all  $A>0$. Therefore we must have $A=0$.
\end{proof}

System \eqref{system1}--\eqref{system2} is equivalent to
\begin{align}
 &  \rho_\beta = \phi_\beta(c) \qquad \text{ for all } \beta \in [s_c] \label{system10} \\
& \sum_{\beta = 1}^{s_c} \left(  \phi_\beta(c) \right)^2  = L n_c^{L-1}  \label{system20} 
\end{align}
Define the function
$$
\Phi(c) := \sum_{\beta = 1}^{s_c} \left(  \phi_\beta(c) \right)^2
$$
Then $\Phi$ clearly inherits the properties of the $\phi_\beta$'s: it is continuous, strictly decreasing, and satisfies
$$
\Phi(0)=  \sum_{\beta=1}^{s_c}  \left( \frac{L}{\lambda n_c} {\mu_\beta} \right)^2 \qquad \text{ and } \qquad \lim_{c\to +\infty} \Phi(c)=0
$$
Therefore, if 
$$
L n_c^{L-1} \le \sum_{\beta=1}^L  \left( \frac{L}{\lambda n_c} {\mu_\beta} \right)^2
$$
then there is a unique $c \ge 0$ satisfying \eqref{system20}. Since $x \mapsto g(c,x)$ is increasing, equation \eqref{system1}  implies that the corresponding $\rho_\beta$'s satisfy
$
\rho_1 \ge \rho_2 \ge \ldots \ge \rho_{s_c} > 0.
$

\end{proof}

\section{No spurious local minimizer for $\mathcal R(W,U)$.} \label{section:spurious}
In this section we prove that if $d> \min(n_w, KL)$,  then $\mathcal R(W,U)$ does not have spurious local minimizers; all local minimizers are global. 
To do this, we introduce the function
$$
f: \real^{d \times KL} \to \real
$$
define as follow.  Any matrix $V \in \real^{d \times KL}$ can be partition into $K$ submatrices $V_k   \in \real^{d \times L} $ according
\begin{equation} \label{A}
 V = \begin{bmatrix}
 V_1 & V_2 & \cdots & V_K
 \end{bmatrix} \qquad \text{where }  V_k \in \real^{d \times L}
 \end{equation} 
 The function $f$ is then defined by the formula 
  $$
  f(V) :=   \frac{1}{K}  \sum_{k=1}^{K}  \sum _{\bx \in \data_k}   \ell \Big( \dotprodbig{V_1 , \hot(\bx)} , \ldots, \dotprodbig{V_K , \hot(\bx)}  \; ; \;  k \;\Big)   \; \;   \mathcal D_{\bz_k}(\bx) 
$$
where $\ell(y_1 , \ldots, y_K ; k)$ denotes the cross entropy loss
 \[
\ell(y_1, \ldots, y_K ;  k) = -   \log \left(  \frac{\exp\left( y_k\right)}{\sum_{k'=1}^K \exp\left( y_{k'}\right)}\right) 
\]
We remark that  $f$ is clearly convex and differentiable.
We then recall from \eqref{def:net} that the $k^{th}$ entry of the vector $\by = h_{W,U}(\bx)$ is 
$$
 y_k = \dotprodbig{ \; \hat U_k \; , \;  W \, \hot(\bx)} =  \dotprodbig{ \;  W^T\hat U_k \; , \;   \hot(\bx)} 
$$
Recalling that
$ 
 \hat U = \begin{bmatrix}
 \hat U_1 &  \cdots & \hat U_K
 \end{bmatrix} 
$, 
we then see that the  risk  can be expressed as
\begin{equation}
\mathcal R(W,U) = f(W^T \hat U) +\frac{\lambda}{2} \left(  \|W\|^2_{F} + \|U\|_F^2 \right) \label{RRR}
\end{equation}
The fact that $\mathcal R(W,U)$ does not have spurious local minimizers come from the following general theorem.

\begin{theorem}\label{theorem:nospurious} Let $g : \real^{m\times n} \to \real$ be  a convex and differentiable function. Define 
$$
\vphi(A,B) := g(A^TB) +\frac{\lambda}{2} \left(  \|A\|^2_{F} + \|B\|_F^2 \right) \qquad \text{where } \;\;   A \in \real^{d\times m} \text{ and } B \in \real^{d \times n}
$$
and assume $\lambda>0$ and  $d> \min(m,n)$.  Then any local minimizer $(A,B)$ of the function $\vphi: \real^{d \times m} \times \real^{d \times n} \to \real$ is also a global minimizer.
\end{theorem}
The above theorem states  that any functions of the form $\vphi(A,B) = g(A^TB)$, with $g$  convex and differentiable,  does not have spurious local minimizer if $d>\min(m,n)$. This theorem directly apply to \eqref{RRR} and shows that the risk $\mathcal R(W,U)$ does not have spurious local minimizers when $d>\min(n_w, KL)$. 

The remainder of the section is devoted to the proof of theorem \ref{theorem:nospurious}. 
We will follow the exact  same steps as in
\cite{zhu2021geometric}, and provide the proof mostly for completeness (and also to show how the techniques from \cite{zhu2021geometric} apply to our case).
 Finally,  we refer to \cite{laurent2018deep} for a proof of  theorem    \ref{theorem:nospurious} in the case $\lambda = 0$.

\begin{proof}[Proof of theorem \ref{theorem:nospurious}]
To prove the  theorem  it suffices to assume that $d> m$ without loss of generality. To see this, note that the function $\tilde g(D) = g(D^T)$ is also convex and differentiable  and note 
    that $(A,B)$ is a local minimum of 
 $$g(A^TB) +\frac{\lambda}{2} \left(  \|W\|^2_{F} + \|U\|_F^2 \right)$$ 
 if and only if it is a local minimum of 
 $$\tilde g(B^TA) +\frac{\lambda}{2} \left(  \|W\|^2_{F} + \|U\|_F^2 \right)$$ 
So the theorem for the case $d>n$ follows by appealing to the case $d>m$ with the function $\tilde g$.

So we may assume $d>m$. Following \cite{zhu2021geometric}, we define  the function $\psi: \real^{m\times n} \to \real$ by
$$
\psi(D) := g(D) +  \|D\|_* 
$$
where  $\|D\|_*$ denote the nuclear norm of $D$.
We then have:
\begin{claim} For all $A \in \real^{d \times m}$ and $B \in \real^{d \times n}$, we have that 
$
\psi(A^TB)  \le \vphi(A,B)
$.
\end{claim}
\begin{proof}
This is a direct consequence of the inequality
$$
\|A^TB\|_* \le \frac{1}{2} \left(\|A\|_F^2 + \|B\|_F^2\right) 
$$
that we reprove here for completeness. Let
$
A^TB = U \Sigma V^T
$ 
be the compact SVD of $A^TB$. That is $\Sigma \in \real^{r\times r}$, $U \in \real^{m \times r}$, $V \in \real^{n \times r}$, and $r$ is the rank of $A^TB$. We then have
\begin{multline*}
\|A^TB\|_* =  \text{Tr} (\Sigma) = \text{Tr} (U^TA^TB V) =  \dotprodbis{A U , B V}   \le  \frac{1}{2} \left(\|A U\|_F^2 + \|B  V\|_F^2  \right)  \le \frac{1}{2} \left(\|U\|_F^2 + \| V\|_F^2  \right)
\end{multline*}
\end{proof}

Computing the derivatives of $\vphi$ gives
\begin{align}
\frac{\partial \vphi}{\partial A} (A,B) = B \;  \left[ \nabla g(A^TB) \right]^T + \lambda A  \qquad \text{and} \qquad
\frac{\partial \vphi}{\partial B} (A,B) = A \;   \nabla g (A^TB) + \lambda B
\end{align}
So $(A,B)$ is a critical point of $\vphi$ if and only if
\begin{align}
\lambda A &= -B \; \left[ \nabla g(A^TB) \right]^T \\
\lambda B &= -A \; \nabla g(A^TB)  \label{grad2}
\end{align}
Importantly, from the above we get
\begin{equation}
AA^T = BB^T \in \real^{d\times d}
\end{equation}
which implies that $A$ and $B$ have same singular values and same left singular vectors. To see this, let $U \in \real^{d \times d}$ be the orthonormal matrix containing the eigenvectors of $AA^T=BB^T$. From this matrix we can construct an  SVD for both $A$ and $B$:
$$
A = U \Sigma_A V_A^T \qquad \text{ and } \qquad B=U \Sigma_B V_B^T
$$ 
where $\Sigma_A \in \real^{d \times m}$ and  $\Sigma_B \in \real^{d \times n}$ have the same singular values. From this we get the SVD of $A^TB$,
\begin{align}
A^TB = V_A \Sigma_A^T \Sigma_B V^T_B
\end{align}
and it is transparent that, 
\begin{align} \label{same}
\|A^TB\|_* = \|A\|_F^2 = \|B\|_F^2.
\end{align}
In particular this implies that if $(A,B)$ is a critical point of $\vphi $, then we must have $\vphi (A,B)= \psi (A^TB)$.
This also implies that
\begin{equation} \label{lala}
\dotprodbis{\nabla g(A^TB), A^TB}  = \dotprodbis{A\nabla g(A^TB), B} = -\lambda\|B\|_F^2 = -\lambda\|A^TB\|_*
\end{equation}
Using this together with the fact that the nuclear norm is the dual of the operator norm, that is
$
\|C\|_* = \sup_{\|G\|_{op} \le 1 } \dotprodbis{G, C}
$, we easily obtain:
\begin{claim} Suppose  $(A,B)$ is a critical point of $\vphi$ which satisfies
 $\left\|  \nabla g(A^TB)\right\|_{op} \le \lambda$, then $D=A^TB$ is   a global minimizer of $\psi$.
\end{claim}
\begin{proof} 
For any matrix $C \in \real^{m \times n}$ we have 
\begin{align*}
\|A^TB\|_* + \dotprodbis{-\frac{1}{\lambda} \nabla g(A^TB),C-A^TB} = \dotprodbis{-\frac{1}{\lambda} \nabla g(A^TB),C}  \le \sup_{\|G\|_{op} \le 1 } \dotprodbis{ G , C} = \|C\|_*
\end{align*} 
and therefore $-\frac{1}{\lambda} \nabla g(A^TB) \in \partial \| A^TB \|_*$. This implies that $A^TB$ is a global min of $\psi$. \end{proof}

We then  make the following claim:
\begin{claim} Suppose  $(A,B)$ is a critical point of $\vphi$ which satisfies
\begin{itemize}
\item[(i)] $\text{ker}(A^T) \neq \emptyset$
\item[(ii)]  $\left\|  \nabla g(A^TB)\right\|_{op} > \lambda$
\end{itemize}
Then $(A,B)$ is not local min. 
\end{claim}
\begin{proof} 
We follow the computation from  \cite{zhu2021geometric}.
Let  $(A,B)$ be a critical point of $\vphi$. Since $AA^T=BB^T$, we must have that $\text{ker}(A^T)= \text{ker}(AA^T) = \text{ker}(BB^T)=\text{ker}(B^T)$. 
According to (ii)  these  kernels are non trivial and we may choose a unit vector $\bz \in \real^d$ that belongs to them. We then consider the perturbations
$$
dA =  \bz \ba^T  \quad dB =   \bz \bbb^T
$$
where $\ba \in \real^{m}$ and $\bbb\in \real^{n}$ are unit vectors to be chosen later. Note that since $\bz,\ba$ and $\bbb$ are unit vectors we have $\|dA\|^2_F= \|dB\|^2_F=1$. 
Moreover, the columns of $dA$ and $dB$ are clearly in the kernel of $A^T$ and $B^T$, therefore
$A^T dA = A^T dB = B^T dA = B^T dB = 0$. This implies that all the 'cross term' disappear when expanding the expression:
$$
(A+\varepsilon dA)^T (B+\varepsilon dB) = A^TB + \varepsilon^2 dA^TdB = A^TB + \varepsilon^2 \ba \bbb^T
$$
We also have 
$$
\|A+\varepsilon dA\|_F^2= \|A\|_F^2 +  \|\varepsilon dA\|_F^2 =  \|A\|_F^2 + \varepsilon^2
$$
and similarly, $\|B+\varepsilon dB\|_F^2= \|B\|^2 + \varepsilon^2$. We then get
\begin{align*}
\vphi(A+  \varepsilon dA,  B + \varepsilon dB) &=   g\Big((A+\varepsilon dA)^T (B+\varepsilon dB)\Big) + \frac{\lambda}{2} \left(\|A+  \varepsilon dA\|_F^2 + \|B+  \varepsilon dB\|_F^2\right) \\
& = g(A^TB + \varepsilon^2   \ba  \bbb^T) + \frac{\lambda}{2} \left(\| A\|_F^2 + \| B\|_F^2 \right) +  \lambda  \varepsilon^2  \\
& = \Big[  g(A^TB) + \dotprodbis{  \nabla f(A^TB) ,  \varepsilon^2   \ba  \bbb^T } + O(\varepsilon^4) \Big] + \frac{\lambda}{2} \left(\| A\|_F^2 + \| B\|_F^2 \right)  +  \lambda  \varepsilon^2  \\
& = \vphi(A,  B) + \varepsilon^2 \Big(  \dotprodbis{  \nabla g(A^TB) ,\ba  \bbb^T } +  \lambda  \Big) + O(\varepsilon^4) 
\end{align*}
Let  $G = \nabla f(A^TB) \in \real^{m \times n}$.
We want to choose the unit vectors $\ba$ and  $\bbb$ that makes
$
 \dotprodbis{  G ,\ba  \bbb^T }
$ as negative as possible.
The best choice is to choose $-\ba$ and $\bbb$ to be the first left and right singular vectors of $G$ since this give the negative of the best rank--$1$ approximation of $G$. So we choose $\ba \in \real^{m}$ and $\bbb \in \real^{n}$ such that
$
G  \bbb = - \sigma_1 \ba 
$, 
and therefore
$$
 \dotprodbis{  \ba  \bbb^T , G } = \text{Tr}( \bbb \ba^T G)=  \text{Tr}( \ba^T G \bbb) = - \sigma_1
$$ 
which gives
\begin{align*}
\vphi(A+  \varepsilon dA,  B + \varepsilon dB) \
 = \vphi(A,  B) + \varepsilon^2 \Big(  - \left\|  \nabla g(A^TB)\right\|_{op}  + {\lambda}     \Big) + O(\varepsilon^4) 
\end{align*}
and (ii) implies that $(A,B)$ is not a local min.
\end{proof}

Combining the three claims we can easily prove the theorem. Indeed, if $d>m$, then the kernel of $A^T$ is nontrivial and (i) is always satisfied. As a consequence, if $(A,B)$ is a local min of $\vphi$, then $A^TB$ must be a global min of  $\psi$. Since $\vphi(A,B)= \psi(A^TB)$ at critical points and $\vphi(A,B) \ge \psi(A^TB)$ otherwise, then $(A,B)$ must be a global min of $\vphi$.

\end{proof}

\end{document}